\documentclass{article}

\usepackage{microtype}
\usepackage{graphicx}
\usepackage{subfigure}
\usepackage{booktabs} 

\usepackage{hyperref}

\usepackage[accepted]{icml2024}

\usepackage{amsmath}
\usepackage{amssymb}
\usepackage{mathtools}
\usepackage{amsthm}

\usepackage[capitalize,noabbrev]{cleveref}

\theoremstyle{plain}
\newtheorem{theorem}{Theorem}[section]

\newtheorem{lemma}[theorem]{Lemma}
\newtheorem{corollary}[theorem]{Corollary}
\theoremstyle{definition}

\theoremstyle{remark}
\newtheorem{remark}[theorem]{Remark}

\usepackage[textsize=tiny]{todonotes}

\usepackage{amsmath,amsfonts,bm}

\def\eqref#1{equation~\ref{#1}}

\def\1{\bm{1}}

\def\vzero{{\bm{0}}}

\def\vmu{{\bm{\mu}}}

\def\va{{\bm{a}}}
\def\vb{{\bm{b}}}

\def\vm{{\bm{m}}}

\def\vv{{\bm{v}}}
\def\vw{{\bm{w}}}
\def\vx{{\bm{x}}}

\def\mA{{\bm{A}}}

\def\mI{{\bm{I}}}

\def\mW{{\bm{W}}}

\DeclareMathAlphabet{\mathsfit}{\encodingdefault}{\sfdefault}{m}{sl}
\SetMathAlphabet{\mathsfit}{bold}{\encodingdefault}{\sfdefault}{bx}{n}

\def\tT{{\tens{T}}}

\def\sF{{\mathbb{F}}}

\newcommand{\E}{\mathbb{E}}

\newcommand{\R}{\mathbb{R}}

\DeclareMathOperator*{\argmin}{arg\,min}

\allowdisplaybreaks[1]

\usepackage{url}

\definecolor{pinegreen}{rgb}{0.0, 0.47, 0.44}
\definecolor{cornellred}{rgb}{0.7, 0.11, 0.11}
\definecolor{cadmiumgreen}{rgb}{0.0, 0.42, 0.24}
\definecolor{royalblue}{rgb}{0.0, 0.14, 0.4}
\definecolor{spirodiscoball}{rgb}{0.06, 0.75, 0.99}
\definecolor{mylightblue}{rgb}{0.85, 0.90, 0.94}
\definecolor{kaistblue}{RGB}{20,135,200}
\definecolor{auburn}{RGB}{166,38,57}

\hypersetup{
  urlcolor = cadmiumgreen,
  linkcolor = cornellred,
  citecolor  = cadmiumgreen,
  colorlinks = true,
}

\usepackage{url}
\usepackage{nicefrac,xfrac}
\usepackage{nccmath}

\usepackage{xcolor}
\usepackage[abs]{overpic}
\usepackage{pifont}
\definecolor{mycolor}{HTML}{226CE0}
\usepackage{microtype}
\usepackage{graphicx}
\usepackage{subfigure}
\usepackage{booktabs} 
\usepackage{bm}
\usepackage{amsmath}
\usepackage{wrapfig}
\usepackage{amssymb} 
\usepackage{stmaryrd}
\usepackage{array}
\usepackage{amsthm}
\usepackage{natbib}
\usepackage{tabularx} 
\usepackage{multirow}

\counterwithin*{case}{theorem}
\usepackage{comment}
\usepackage{cleveref}
\usepackage{autonum}
\usepackage{enumitem}
\usepackage{xspace}
\usepackage{arydshln}
\usepackage{titletoc}
\usepackage{arydshln}
\usepackage{amssymb}
\usepackage{mathtools}
\usepackage{pifont}
\usepackage{tikz}
\usepackage{CJKutf8}

\usepackage{subcaption,graphicx}
\usepackage{mdframed}

\usepackage{wrapfig}

\newcommand{\ie}{\textit{i}.\textit{e}.}
\newcommand{\eg}{\textit{e}.\textit{g}.}

\def\vx{{\bm{x}}}

\def\vzero{{\bm{0}}}

\def\va{{\bm{a}}}
\def\vb{{\bm{b}}}
\def\vm{{\bm{\mu}}}
\def\vw{{\bm{w}}}
\def\vwt{{\vw^{*}}}

\def\vmu{{\bm{\mu}}}
\def\vmut{{\vmu^{*}}}
\def\vmt{{\vmu^{*}}}
\def\vep{\bm{\epsilon}}
\def\sf{f^*}

\def\mA{{\bm{A}}}

\def\mI{{\bm{I}}}

\def\mzero{{\bm{0}}}

\def\Dpr{\mathcal{D}^{\text{prior}}}
\def\Dpost{\mathcal{D}^{\text{post}}}
\def\Dpo{\mathcal{D}^{\text{post}}}

\def\ddist{\mathcal{D}_{\vx}}
\def\fdist{\mathcal{D}_{y|\vx}}

\def\tT{\widetilde{T}}
\def\tpi{\tilde{\pi}}
\def\tvmu{\tilde{\vmu}}
\def\tvm {\tilde{\vmu}}
\def\tvw {\tilde{\vw }}
\def\tmu{\tilde{\mu}}
\def\tsigma{\tilde{\sigma}}

\def\nm{\sigma_\mu}
\def\nw{\sigma_w}
\def\tnm{\tilde{\sigma}_\mu}
\def\tnw{\tilde{\sigma}_w}
\def\nx{\sigma_x}
\def\ny{\sigma_y}
\def\nxd{\tau_x}

\def\sm {\sigma_\mu}
\def\smu{\sigma_\mu}
\def\sw {\sigma_w }
\def\S{{\mathcal{S}}}
\def\F{{\mathcal{F}}}
\def\hF{\hat{\mathcal{F}}}
\def\sF{\mathcal{F}^*}

\def\SK{\S_K}
\def\Sk{\S_k}
\def\Skx{\S_k \oplus \vx_{k+1}}
\def\D{\mathcal{D}}

\def\dm{\delta_\mu}
\def\dw{\delta_w}
\def\mm{\bar{\vm}}
\def\mw{\bar{\vw}}
\def\mCm{\bar{\bm{\Sigma}}_\vm}
\def\mCw{\bar{\bm{\Sigma}}_\vw}
\def\Dmm{\Delta_\mu\bar{\vm}}
\def\Dmw{\Delta_w\bar{\vw}}
\def\DmCm{\Delta_\mu\bar{\bm{\Sigma}}_\vm}
\def\DmCw{\Delta_w\bar{\bm{\Sigma}}_\vw}

\def\advm{\Psi_\vm}
\def\advw{\Psi_\vw}

\def\dk{{\gamma}}

\def\L{\text{L}}
\def\U{\text{U}}

\def\DE{\displaystyle\mathop{\mathbb{E}}}

\usepackage{amsmath}
\DeclareMathOperator{\diag}{diag}

\newtheorem{asu}{Assumption}
\newcounter{subassumption}[asu]
\renewcommand{\thesubassumption}{(\textit{\roman{subassumption}})}
\makeatletter
\renewcommand{\p@subassumption}{\theasu}
\makeatother

\renewcommand{\thesubassumption}{(\alph{subassumption})}

\newcommand{\subasu}{%
  \refstepcounter{subassumption}%
  \thesubassumption~\ignorespaces}

\DeclareMathOperator*{\argminA}{argmin}

\def\CR{CR\xspace}
\def\CRfull{Component Re-weighting\xspace}
\def\CS{CS\xspace}
\def\CSfull{Component Shifting\xspace}
\def\BO{Bayes-Optimal\xspace}
\def\bo{Bayes-optimal\xspace}
\def\SP{Next-Token Predictor\xspace}
\def\sp{next-token predictor\xspace}
\def\BOSP{\BO \SP}
\def\bosp{\bo \sp}

\newcommand{\red}[1]{\textcolor{red}{#1}}
\newcommand{\blue}[1]{\textcolor{blue}{#1}}
\newcommand{\redb}[1]{\textcolor{black}{#1}}
\newcommand{\blueb}[1]{\textcolor{black}{#1}}

\newcommand{\GGG}[1]{\textcolor{orange}{#1}}

\def\mytitle{Dual Operating Modes of In-Context Learning}

\icmltitlerunning{\mytitle}

\begin{document}

\twocolumn[
\icmltitle{\mytitle}




\begin{icmlauthorlist}
\icmlauthor{Ziqian Lin}{yyy}
\icmlauthor{Kangwook Lee}{zzz}
\end{icmlauthorlist}

\icmlaffiliation{yyy}{Department of Computer Science, University of Wisconsin-Madison, Madison, Wisconsin, USA}
\icmlaffiliation{zzz}{Department of Electrical \& Computer Engineering, University of Wisconsin-Madison, Madison, Wisconsin, USA}

\icmlcorrespondingauthor{Kangwook Lee}{kangwook.lee@wisc.edu}

\icmlkeywords{Machine Learning, ICML}

\vskip 0.3in
]



\printAffiliationsAndNotice{}  

\begin{abstract}
In-context learning (ICL) exhibits dual operating modes: \emph{task learning}, \ie{} acquiring a new skill from in-context samples, and \emph{task retrieval}, \ie{}, locating and activating a relevant pretrained skill.
Recent theoretical work proposes various mathematical models to analyze ICL, but they cannot fully explain the duality.
In this work, we analyze a generalized probabilistic model for pretraining data, obtaining a quantitative understanding of the two operating modes of ICL.
Leveraging our analysis, we provide the first explanation of an unexplained phenomenon observed with real-world large language models (LLMs).
Under some settings, the ICL risk initially increases and then decreases with more in-context examples.
Our analysis offers a plausible explanation for this ``early ascent'' phenomenon: a limited number of in-context samples may lead to the retrieval of an incorrect skill, thereby increasing the risk, which will eventually diminish as task learning takes effect with more in-context samples.
We also analyze ICL with biased labels, \eg, zero-shot ICL, where in-context examples are assigned random labels, and predict the bounded efficacy of such approaches.
We corroborate our analysis and predictions with extensive experiments with Transformers and LLMs.
The code is available at:  
\url{https://github.com/UW-Madison-Lee-Lab/Dual_Operating_Modes_of_ICL}.
\end{abstract}

\section{Introduction}
Large language models (LLMs) exhibit a significant improvement in predictive performance when provided with in-context examples~\citep{BrownMRSKDNSSAA20}.
This emergent ability of LLMs, known as in-context learning (ICL), \textbf{operates in two distinct modes: task learning and task retrieval}~\citep{pan2023context}.
Large language models exemplify this duality.
They can learn unseen functions from in-context examples, demonstrating the learning mode~\citep{BrownMRSKDNSSAA20,RazeghiL0022,garg2022can}.
Concurrently, LLMs can also retrieve and utilize a \emph{pretrained} skill.
A clear evidence of the task retrieval mode is presented by \citet{MinLHALHZ22}, where the authors show ICL performance remains largely unaffected even when in-context examples are annotated with random labels.
This suggests that LLMs simply retrieve a pretrained skill rather than learn it from in-context examples.

The dual nature of ICL can be explained as follows. 
LLMs are a next-token predictor that is pretrained on a large pretraining set, consisting of diverse data from diverse domains/tasks.
To predict the next token optimally in such a scenario, the model must first learn the task prior from pretraining data and then implicitly perform Bayesian inference at the test time~\citep{xie2021explanation,raventos2023effects}. 
Optimal prediction on multitask pretraining data requires adherence to the learned prior (over the tasks present in the pretraining data) and making predictions based on the posterior. 
The ability to learn and apply this prior during test-time inference enables task retrieval--if in-context examples align closely with a task encountered during pretraining, the model can swiftly adjust its posterior and predict without learning a new skill. 
Simultaneously, the model can learn a novel or uncommon skill given sufficient in-context samples and a non-zero prior probability for that skill.

Although the link between pretraining and ICL's dual modes is conceptually straightforward, formally establishing this connection is an unresolved challenge. 
Motivated by this, our work seeks to address the following questions: \emph{How do we rigorously explain the dual operating modes of ICL? Can we define the conditions under which the retrieval mode is a dominant one and vice versa?}

\begin{figure*}[t!]
    \centering
    \includegraphics[width = 1.0\textwidth]{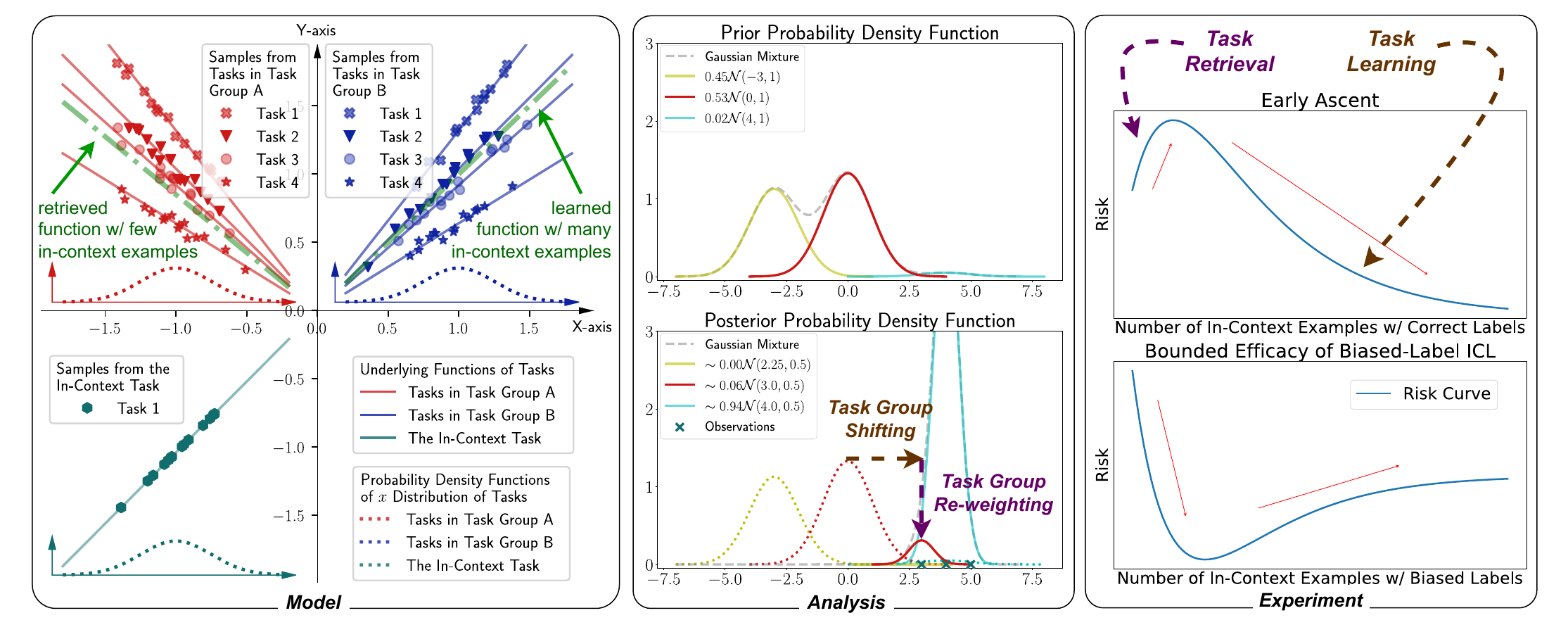}
    \caption{\textbf{A summary of our contributions.} We first propose a probabilistic model for pretraining data and in-context examples. By analyzing our model, we obtain a quantitative understanding of the dual operating modes of ICL, and explain two real-world phenomena observed with LLMs.}
    \label{fig:figure1}
\end{figure*}
\paragraph{A New Model for Pretraining Data} To find the answers to these questions, 
we first propose a new probabilistic model for pretraining data by assuming the pretraining data has a latent clustered structure.
In particular, we consider in-context learning of linear functions following the recent work~\citep{garg2022can,akyurek2022learning,Li23algorithm,von2022transformers,raventos2023effects,anonymous2023how}.
A next-token prediction model is prompted with (1) a sequence of $(\vx, y)$ pairs, which come from a common linear function, 
and (2) one test input $\vx_\text{test}$.
An ideal model capable of in-context learning linear models should internally fit a linear function (say $y = \widehat{\vw}^T \vx$) using the in-context examples and then generate the predicted label $y_\text{test} = \widehat{\vw}^T \vx_\text{test}$ as the next token.
The recent work~\citep{raventos2023effects,anonymous2023how} show that such in-context learning is feasible by training a next-token prediction model on a large pretraining dataset, consisting of sequences of labeled samples drawn from diverse linear functions.

We extend the existing model for pretraining data~\citep{raventos2023effects} by introducing multiple task groups and task-dependent input distributions.
When one generates pretraining data, one must specify a probability distribution of linear functions (equivalently, that of the linear coefficient $\vw$). 
While most of the prior work assumes that $\vw$ is drawn from a single Gaussian distribution, we will model it as drawn from a Gaussian \emph{mixture} model, where each Gaussian component models a \emph{task group}.
This model better reflects real-world data that exhibits a clustered structure~\citep{xie2021explanation}. 
Furthermore, we also allow each mixture component to have its own distribution for input $\vx$. 
Shown on the left-most panel in Fig.~\ref{fig:figure1} is a simple visualization of our model.
The blue task group is modeled as the distribution of linear functions with positive coefficients ($\vw \approx 1$) with the input distribution centered around $\E[\vx]=+1$.
The red lines represent the other task group -- linear functions with negative coefficients ($\vw \approx -1$) with the input distribution centered at $\E[\vx]=-1$.
See Sec.~\ref{sec:setting&connection} for more details.

\paragraph{Analysis} With our new model for pretraining data, we analyze the optimal pretrained model under the squared loss, \ie{}, the MMSE estimator of the label given input with in-context examples.
Here, the pretraining distribution (of linear functions) is the prior, and in-context examples are the observations.
Leveraging the fact that the Gaussian mixture is a conjugate prior to the Gaussian likelihood function, we obtain a closed-form expression of the posterior distribution.
By fully quantifying the posterior distribution of $\vw$ in the form of a Gaussian mixture, we characterize how in-context examples are used to update each component's posterior mean and posterior mixture probability. 
We will call updates of mixture probabilities as \emph{task group (component) re-weighting} and updates of component means as \emph{task group (component) shifting}. 
See the central panel in Fig.~\ref{fig:figure1} for visualization.
By analyzing these two effects, we obtain a quantitative understanding of how two different operating modes emerge.
In particular, we show that, under some mild assumptions, task group re-weighting is the dominant factor when provided with few in-context samples, rendering the task retrieval mode. 
With many in-context samples, task group shifting occurs, resulting in the task learning mode.

\paragraph{Explanation of Two Real-World Phenomena} To demonstrate the practical value of the new insights we have gained from our model, we will leverage our analysis to explain and predict two phenomena observed with LLMs in practice.
\begin{itemize}[leftmargin=0.2cm]
    \item \textbf{The \emph{early ascent} phenomenon} refers to the observation that, under certain conditions, the ICL risk initially increases and then decreases when more in-context examples are introduced~\citep{BrownMRSKDNSSAA20,xie2021explanation}.
    See the right-most panel of Fig.~\ref{fig:figure1} for visualization. 
Based on our analysis, we offer a plausible explanation for this early ascent phenomenon--a limited number of in-context samples may lead to the retrieval of an incorrect skill, thereby increasing the risk, which will eventually diminish as task learning takes effect with more in-context samples.
    \item 
    \textbf{Bounded efficacy of biased-label ICL} is predicted by our model. 
ICL performs well even with in-context examples that are annotated with biased labels~\citep{LyuMBZH23,MinLHALHZ22}.
Our model provides a rigorous justification of this approach: If in-context examples with biased labels carry sufficient information for retrieving a correct pretrained task, then this approach would work. 
At the same time, our analysis suggests that the operating mode of ICL will make a transition from task retrieval to task learning with more in-context examples.
When the learning mode starts taking place, the test risks of such methods will start increasing as the pretrained model will start fitting the biased labels. 
See the right-most panel of Fig.~\ref{fig:figure1} for visualization.
This bounded efficacy has not been reported in the literature~\citep{MinLHALHZ22,pan2023context}. 
We found that this was due to the small number of examples tested.
With more in-context samples, we observe the predicted bounded efficacy phenomenon with real-world LLMs such as Mistral 7B~\citep{jiang2023mistral}, Mixtral 8$\times$7B~\citep{jiang2024mixtral}, Llama 2~\citep{touvron2023llama}, and GPT-4~\citep{openai2023gpt4}.
\end{itemize}

\section{Related Work}
\paragraph{Dual Operating Modes of ICL.}
\citet{pan2023context} empirically disentangle the two operating modes of ICL: task recognition, which we refer to as task retrieval and task learning.
To illustrate, in the context of sentence sentiment classification using ICL, \citet{pan2023context} explore three labeling schemes for in-context examples: (i) correct semantic labels, (ii) correct but abstract labels (``0'' and ``1''), and (iii) random semantic labels (``positive'' or ``negative'').
\citet{pan2023context} claim that ICL is in the task recognition mode when the model is provided with randomly labeled in-context data, and observe that its efficacy does not correlate with model size or the quantity of demonstrations.
In fact, later, we will show that via our analysis, an increasing number of demonstrations will eventually decrease the ICL accuracy.
Conversely, ICL with correct but abstract labels, classified as task learning, shows improved performance in proportion to model size and in-context example count.
ICL with correct labels yields the highest accuracy since both task recognition and task learning benefit it.

\paragraph{Explaining ICL via Bayesian Inference.}
\citet{xie2021explanation} use a Hidden Markov Model (HMM)~\citep{ghahramani1995factorial,rabiner1989tutorial} to model the pretraining data. 
That is, each sequence in pretraining data is generated by an HMM, whose parameters are randomly drawn from a particular distribution.
During pretraining, a next-token prediction model is trained to predict tokens in pretraining sequences, which requires the inference of the latent HMM parameters. 
While this model accurately reflects real-world pretraining data characteristics, such as long-range dependencies, the absence of a closed-form solution for optimal prediction makes detailed analysis of ICL infeasible.
On the other hand, \citet{garg2022can,raventos2023effects} consider the setting where a next-token prediction model is pretrained on token sequences consisting of $(\vx, y)$ pairs in the form of $(\vx_1, y_1, \vx_2, y_2, \ldots)$. 
The pretraining objective is to predict only the tokens at odd positions, \ie, to predict $y$, but not $\vx$.
\citet{garg2022can} empirically evaluate the Transformer architecture~\citep{vaswani2017attention}, while the authors of \citet{raventos2023effects} proposed a probabilistic model to generate sequences according to noisy linear regression.
More specifically, $y_i = \langle \vx_i,\vwt \rangle + \epsilon_i$, where $\vwt$ is the coefficient shared within the same sequence and $\epsilon_i$ is noise. 
While this linear regression model facilitates a tractable analysis and elucidates certain aspects of the dual operating modes of ICL, it falls short in modeling the clustered characteristic of nature language. 
\citet{han23kernel} show that ICL asymptotically approaches kernel regression as the in-context samples increases.
\citet{jeon2024information} introduce information-theoretic tools to show that the ICL risk should decay in both the number and sequence lengths of in-context examples.
On the other hand, our proposed model allows for tractable analysis and captures the clustered characteristic of pretraining data. 

\paragraph{Explaining ICL via Gradient Descent.}
\citet{garg2022can} hint that the pretrained Transformer might implicitly execute gradient descent under ICL.
\citet{akyurek2022learning,von2022transformers,dai2023metaopt} expand this notion by theoretically showing that one attention layer can be exactly constructed to perform gradient descent, and empirically finding similarities between in-context inference and gradient descent algorithm. 
Further,~\citet{ahn23learnGD,mahankali2023step,zhang23learnGD} dive into the training process of Transformers. 
~\citet{ahn23learnGD,mahankali2023step} theoretically show that under certain conditions, Transformers with one or more attention layers trained on noisy linear regression task minimizing the pretraining loss will implement gradient descent algorithm. ~\citet{zhang23learnGD} show that a single linear self-attention layer trained by gradient flow with a suitable random initialization finds a global minimum of the objective function, where ICL of the Transformer achieves prediction error competitive with the
best linear predictor.

\paragraph{Others.}
\citet{anonymous2023how} studies the sample complexity required for pretraining a linear attention model and presents a statistical bound. 
In our work, we do not consider a particular model architecture nor the statistical aspects of pretraining -- we assume a pretrained model is optimally trained on infinitely large pretraining data, similar to the previous work~\citep{xie2021explanation,raventos2023effects,han23kernel}.
\citet{giannou2023looped} show a looped Transformer can emulate any algorithms, 
such as SGD.
\citet{yu23algorithm} show Transformers can perform in-context algorithm selection, \ie, adaptively selecting different ICL algorithms such as gradient descent, least square, or ridge regression.
\cite{Li23algorithm} study the generalization bounds for ICL with Transformers. 

\section{Pretraining and Data Generative Model}
\label{sec:setting&connection}
A next-token predictor is a sequential prediction model that predicts the next token given an initial token sequence.
Consider pretraining this model on sequences consisting of $(\vx, y)$\footnote{It is more rigorous to represent the vector $\vx$ as multiple tokens.
However, viewing it as a high-dimensional ``token'' simplifies our notation while not affecting our analysis. 
Thus, with a slight abuse of notation, we will treat both $\vx_i$ and $y_i$ as tokens for simplicity.} 
pairs in the form of $(\vx_1, y_1, \vx_2, y_2, \ldots)$, with the model trained to predict only the $y$ values, thereby skipping the prediction of $x$.
Here, we assume odd-numbered tokens represent $d$-dimension real-valued vectors, and even-numbered tokens represent scalars.
During inference, the model receives a sequence of $2k+1$ tokens.
The first $2k$ tokens are $k$ labeled samples $(\vx_i,y_i),i\in\{1,\ldots,k\}=:[K]$, and the last token is unlabeled $\vx_{k+1}$. 
Ideally, the model should predict the correct next token, $y_{k+1}$.

\subsection{Data Generative Model}
\label{subsec:setting}
In the pretraining phase, 
we assume the \sp is pretrained on diverse tasks, each representing a continuous joint distribution of $(\vx,y)$.
Before we move on to the exact pretraining data generative model proposed in this paper, we first provide a general setting for the data generation process.
A task is defined by a joint distribution $\mathcal{D}_{\vx,y}$, which specifies the likelihood of obtaining a sample $(\vx,y)$ from this task.
Each task is sampled from the task prior $\Dpr$, meaning $\Dpr$ represents a distribution over distributions.
The pretraining data comprises numerous sequences, each containing $K$ labeled samples i.i.d. drawn from a distribution $\mathcal{D}_{\vx,y}$.
We formally describe our pretraining data generative model in Assumption~\ref{asu:setting}.

\begin{asu}[Pretraining Data Generative Model]
\label{asu:setting} 
Given an integer $K>0$, a pretraining \textbf{task prior} $\Dpr$, we generate a sequence $\SK$ as follows:\\
\subasu \label{asu:setting1} Sample a task from the task prior: $\mathcal{D}_{\vx,y}\sim\Dpr$;\\
\subasu \label{asu:setting2} Sample $K$ labeled samples from the chosen task: $\forall i\in \{1,2,\ldots,K\}$, $(\vx_i, y_i)\sim \mathcal{D}_{\vx,y}$;\\
\subasu \label{asu:setting4} Define a sequence $\SK$: $\SK=[\vx_1,y_1,\ldots,\vx_K,y_K]$. 
\end{asu}
In the sequence, the first $2k$ elements of $\SK$ is denoted as $\Sk$, and the first $2k+1$ elements will be indicated by $\Skx$, \eg, $\S_0=[~]$, and $\S_1\oplus\vx_2=[\vx_1,y_1,\vx_2]$.

\subsection{\BOSP}
\label{subsec:connection}
Let
$    \mathcal{L} (\F) = 
    \mathop{\mathbb{E}}_{\SK} 
    \left[\frac{1}{K}\sum_{k=0}^{K-1} (\F(\Skx)-y_{k+1})^2 \right]
$ as the pretraining objective, where $\F$ is a next-token predictor and $\SK$ is generated from $\Dpr$ following Assumption~\ref{asu:setting}.
In other words, for each sequence, we pretrain $\mathcal{F}$ to predict each label $y$ based on preceding samples, measuring risk with the squared loss.
Due to the linearity of expectation, we have:
$    \mathcal{L} (\F) 
    =\frac{1}{K}\sum_{k=0}^{K-1}\DE_{S_K}  
    \left[ (\F(\Skx)-y_{k+1})^2\right].
$
A variable-input-length next-token predictor $\F$ can be viewed as $K$ fixed-input-length next-token predictors $\F_0,\ldots,\F_{K-1}$, where $\F_k$ takes a sequence of exactly $2k+1$ tokens as input. 
Thus, assuming the sufficient expressiveness of $\F$, the optimization problem $\sF = \argminA_{\F} \mathcal{L} (\F)$ can be decomposed into $K$ separate optimization problems for $k\in\{0,\ldots,K-1\}$:
\begin{align}
    \F_k^*=\argminA_{\F_k} \DE_{\SK} [(\F_k(\Skx)-y_{k+1})^2].
\end{align}
The solution denoted $\F_k^*$ is an MMSE estimator~\citep[page~63]{van2004detection} for each $k$.
Thus, the prediction $\sF(\Skx)= \F_k^*(\Skx)$ satisfies:
\begin{align}
&
    \sF(\Skx)
    =\DE_{\SK} \left[y_{k+1}\vert\Skx\right]
\\&\qquad
    = \displaystyle \mathop{\mathbb{E}}_{\D_{\vx,y}} \left[
    \displaystyle \mathop{\mathbb{E}}_{y_{k+1}} \left[y_{k+1}\vert \D_{\vx,y}, \Skx\right] \middle \vert  \Skx\right]
\\&\qquad
    \label{equation:connection}
    = \displaystyle \mathop{\mathbb{E}}_{\D_{\vx,y}} \left[
    \displaystyle \mathop{\mathbb{E}}_{y_{k+1}} \left[y_{k+1}\vert \D_{\vx,y}, \vx_{k+1}\right] \middle \vert  \Skx\right].
\end{align}
Thus, $\sF(\Skx)$ is the expectation (over task posterior) of $\displaystyle \mathop{\mathbb{E}}_{y_{k+1}} \left[y_{k+1}\vert \D_{\vx,y},\vx_{k+1}\right]$ regarding $\Skx$ as observation. We show that a pretrained Transformer can empirically approximate Bayesian inference in Appendix~\ref{sec:transformer}.

\begin{figure*}[th!]
\centering
    \subfigure[Pretraining data \citet{raventos2023effects}.]{
        \includegraphics[width=0.3\textwidth]{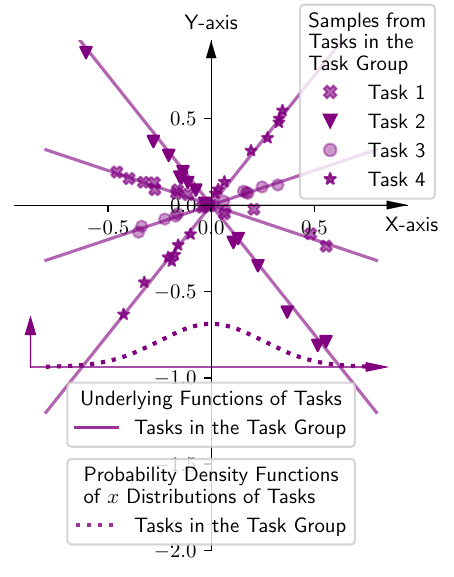}
        }
    \hspace{1em}
    \subfigure[Our pretaining data with 2 task groups.]{
        \includegraphics[width=0.525\textwidth]{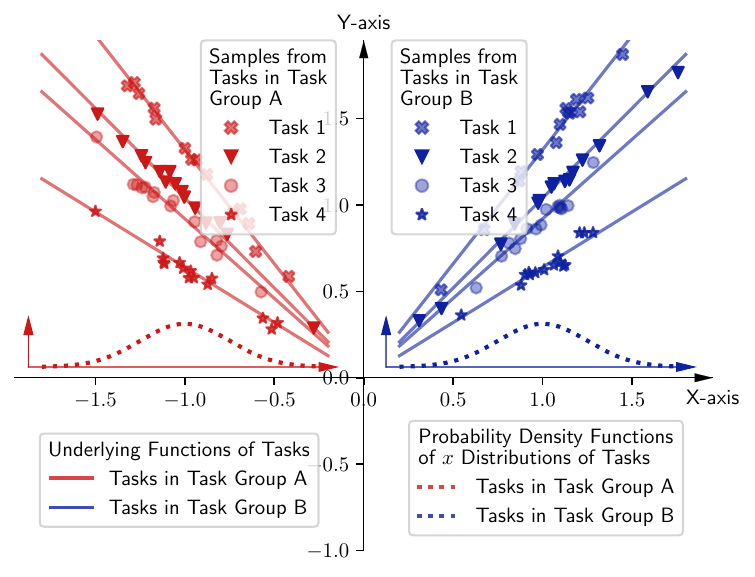}
        }
    \caption{Pretraining data model of \citet{raventos2023effects} and ours.}
    \label{fig:datacompare}
\end{figure*}

\subsection{Gaussian/Linear Assumptions on Pretraining Data Generative Model}
\label{sec:assumption}
Let us now elaborate further assumptions on $\Dpr$ and $\D_{\vx,y}$ in the Assumption~\ref{asu:setting} for a tractable posterior, extending beyond the scope of \citet{raventos2023effects}, who propose the data generative model that each task is a noisy linear regression task, the function 
$\vw$ for each task is drawn from the same Gaussian distribution, and different tasks share the same $\vx$ distribution.
In contrast, our model posits that task functions are derived from a Gaussian mixture distribution, and tasks employ varying $\vx$ distributions, as illustrated in Fig.~\ref{fig:datacompare}.
We formally formulate this setting in Assumption~\ref{asu:assumption}.
\begin{asu}[Gaussian/Linear Assumptions for Pretraining Data Generative Model]
\label{asu:assumption}
\,~\\
\subasu \label{asu:assumption1}
$(\vm,\vw)\sim \Dpr: P(\vm,\vw) = \sum_{m=1}^M \pi_m P(\vm,\vw \vert T_m)$, where $T_m$ is the $m^\text{th}$ mixture component\footnote{The concept ``mixture component'' is derived from Gaussian mixture models in the statistical literature and is analogous to the term ``Task Group'' depicted in the left-most panel of Fig.~\ref{fig:figure1}.} of the Gaussian mixture, \ie, $P(\vm,\vw \vert T_m) = \mathcal{N}(\vm \vert \vm_m,\nm^2\mI)\cdot\mathcal{N}(\vw \vert \vw_m,\nw^2\mI)$, and $\pi_m$ is the mixture weight.
$\sum_{m=1}^M\pi_m=1$, $0<\pi_m<1$, $(\vm_m,\vw_m)$ is the center of the mixture component $T_m$,
and all components share the same covariance matrix controlled by $\nm$ and $\nw$;\\
\subasu \label{asu:assumption2}
input: $\vx \sim \ddist(\vm), P(\vx\vert\vm) = \mathcal{N}(\vx\vert\vm,\nx^2\mI)$;\\
\subasu \label{asu:assumption3} 
label: $y\vert\vx \sim \fdist(\vw): P(y\vert\vx,\vw) = \mathcal{N}(y\vert\vw^\top\vx,\ny^2)$;\\
\subasu \label{asu:assumption4} 
$\|\vm_m\| = \|\vw_m\| = 1, \forall m\in [M]$;\\
\subasu \label{asu:furtherasu3} $\exists r>1$ that $\forall \alpha,\beta \in [M], \frac{1}{r} \leq \frac{\pi_\alpha}{\pi_\beta} \leq r$;\\\
\subasu \label{asu:assumption5} 
$\vx,\vm,\vm_m,\vw,\vw_m \in \mathbb{R}^d, \mI \in \mathbb{R}^{d\times d}$.
\end{asu}
\begin{remark}
\label{remark:likelihood}
    Based on Assumptions~\ref{asu:assumption2} and~\ref{asu:assumption3}, we define the probability of observing a sample $(\vx,y)$ within a task $(\vm,\vw)$ as the ``noisy linear regression'' likelihood.
\end{remark}
Assumption~\ref{asu:assumption1} indicates that the pretraining dataset of an LLM consists of $M$ different task groups.
Assumption~\ref{asu:assumption2} posits that tasks have varying $\vx$ distribution with varying mean but share the same covariance matrix.
Assumption~\ref{asu:assumption3} assumes tasks as noisy linear regressions with the same noise scale in labels.
Assumption~\ref{asu:furtherasu3} posits comparable mixture weights $\pi$ across different task groups.

\begin{figure*}[th!]
\centering
    \subfigure[\textbf{The Tetrahedron setting.} An illustration of the in-context task and the prior centers. $\forall m\in\{1,2,3,4\}$, We set 
  $\vm_m=\vw_m$.]{
        \includegraphics[width=0.225\textwidth]{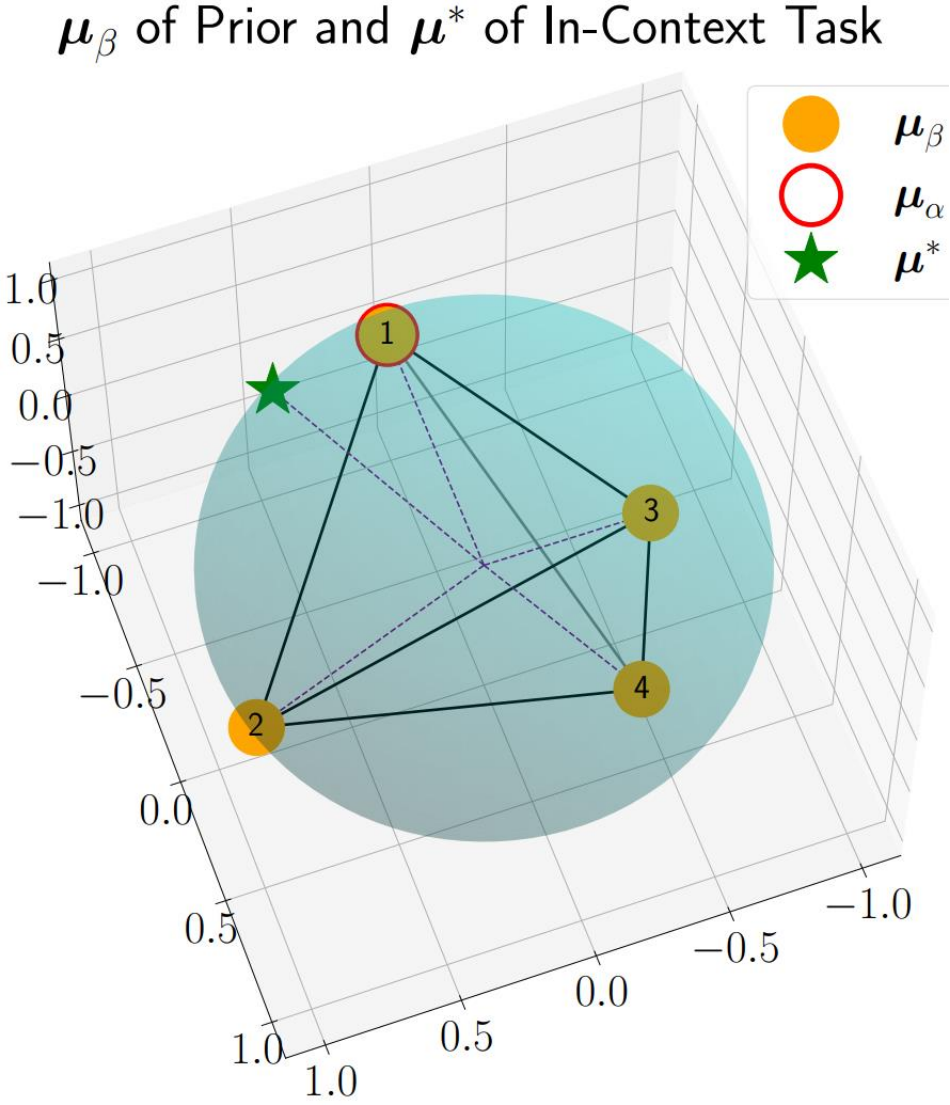}
        \label{fig:3d:4centers:half}}
    \hspace{0.5 em}
    \subfigure[\textbf{CR, CS, and risks under the Tetrahedron setting.}
    In the first two rows, we show the effects of \CS and \CR with an increasing number of in-context examples.
    In the third row, we show how far the in-context predicted function $\tvw$ is from the target function $\vwt$. 
    In the fourth row, we show the ICL risk. ]{
        \includegraphics[width=0.7\textwidth]{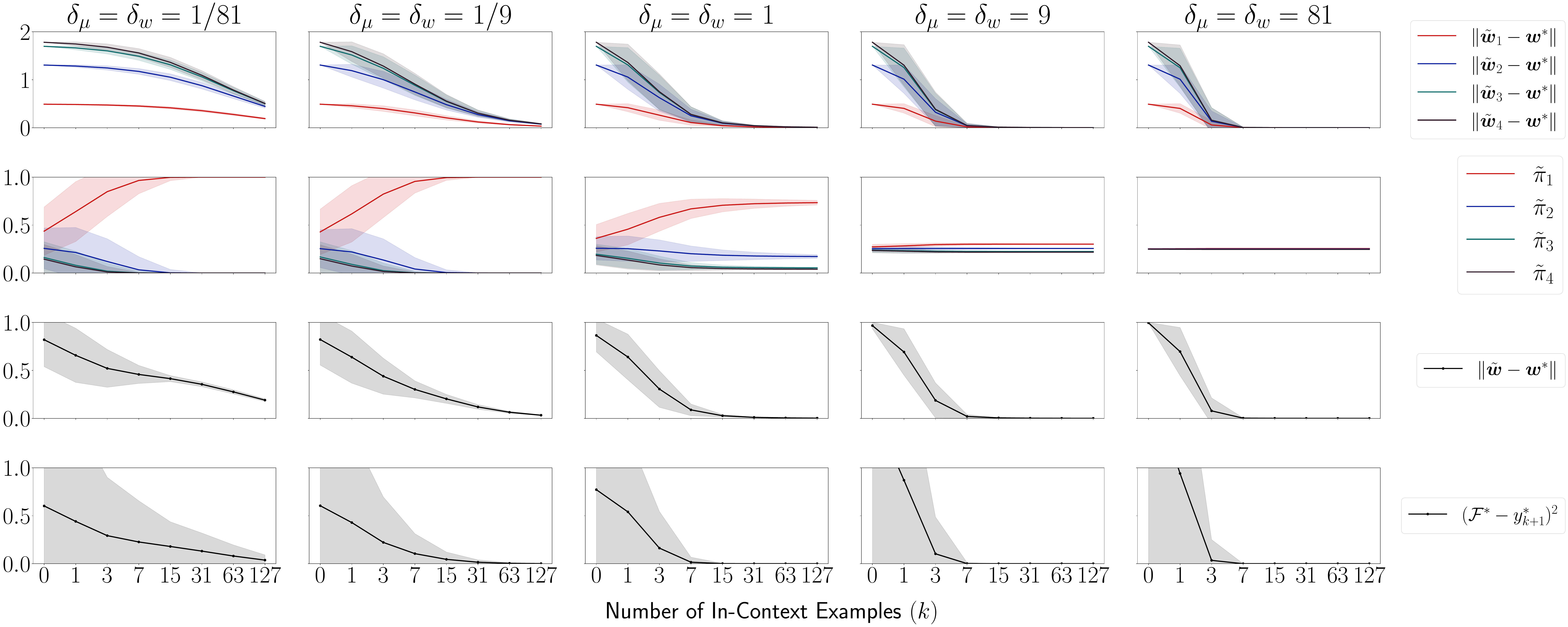}
        \label{fig:subprediction}}
    \caption{\textbf{Numerical experiments.} (left) An illustration of the pretraining priors (right) The numerical computational results}
\end{figure*}
\section{Inference and Dual Operating Modes}
\label{sec:posterior&phenomena}
The previous Sec.~\ref{subsec:connection} shows that performing ICL with the optimally pretrained \sp is equivalent to computing the posterior mean of the label.
In Sec.~\ref{subsec:ictask}, we give the generation process of in-context examples.
In Sec.~\ref{subsec:posterior}, under Assumption~\ref{asu:assumption} and treating $\Skx$ as observation, we derive a closed-form expression for the task posterior $\Dpost$, and identify two factors in the transition from prior to posterior: \CSfull and \CRfull.
In Sec.~\ref{subsec:prediction}, we derive a closed-form expression of the ICL prediction $\sF(\Skx)$.
Further, Sec.~\ref{subsec:twomodes} presents the results of numerical computation conducted under the tetrahedron setting, as illustrated in Fig.~\ref{fig:3d:4centers:half}.
The numerical computation results demonstrate the effects of component shifting and re-weighting.
Finally, Sec.~\ref{subsec:twooperating} raises the definitions of the dual operating modes with component shifting and re-weighting.

\subsection{In-Context Task and In-Context Function}
\label{subsec:ictask}
We introduce Assumption~\ref{asu:source} for the in-context task and the in-context function of in-context examples:
\begin{asu}[Gaussian/Linear Assumptions for In-Context Examples]
\label{asu:source}
\,~\\
\subasu \label{asu:source1} The input sequence $\Skx$ of ICL satisfies, $\forall i$, $\vx_i\sim\mathcal{N}(\vmt,\nxd^2\mI)$, $y_i = \langle \vx_i,\vwt \rangle$;\\
\subasu \label{asu:source2} $\|\vmt\|=\|\vwt\|=1$.
\end{asu}
Assumption~\ref{asu:source1} states that each in-context example $(\vx_i, y_i)$ is drawn from the in-context task $(\vmt, \vwt)$, with $\vwt$ representing the specific in-context function and the labels being free from noise.

\subsection{Closed-Form Expression of Posterior}
\label{subsec:posterior}
The following lemma gives the closed-form expression of posterior $\Dpo$ given any $\Skx$:
\begin{lemma}[Conjugate Distributions with Noisy Linear Regression Likelihood]
\label{lemma:posterior}
Under Assumption~\ref{asu:assumption}, the posterior probability of task $(\vm,\vw)$ given observation $\Skx$ is:
\begin{align}
\label{equation:prediction}
P(\vm,\vw \vert \Skx) &= \textstyle \sum_{m=1}^M \tpi_m P(\vm,\vw \vert \tT_m)\\
&\!\!\!\!\!\!\!\!\!\!\!\!\!\!\!\!\!\!\!\!\!\!\!\!\!\!\!\!\!\!\!\!\!\!\!\!\!\!\!\!\!\!\!\!\!\!\!\!= \textstyle \sum_{m=1}^M \tpi_m \cdot \mathcal{N}(\vm \vert \tvm_m,\tnm^2\mI)\cdot\mathcal{N}(\vw \vert \tvw_m,\tnw^2\mI).\label{shift}
\end{align}
Here, the mixture component $T_m$ in the prior is mapped to the mixture component $\tT_m$ in the posterior with mixture weight $\tpi_m$ and component means $(\tvm_m, \tvw_m)$:
\begin{align}
\label{re-weight}
\tpi_m &= \pi_m C_1 c^\vm_m c^\vw_m,&&\\
c^\vm_m  &=\exp\left(
    -
        \|\vm_m\|^2 - \nicefrac{\|\vm_m+(k+1)\dm\mm\|^2_{(\mI+(k+1)\dm\mCm)^{-1}} 
    }{2\nm^2}
\right),\\
c^\vw_m &=\exp\left(
    -
        \|\vw_m\|^2 - \nicefrac{\|\vw_m+k\dw\mw\|^2_{(\mI+k\dw\mCw)^{-1}} 
    }{2\nw^2}\right),\\
\tvm_m &= (\mI+(k+1)\dm\mCm)^{-1}(\vm_m+(k+1)\dm\mm),\\
\tvw_m &= (\mI+k\dw\mCw)^{-1}(\vw_m+k\dw\mw),\\
\tnm^2 &= \nm^2(\mI+(k+1)\dm\mCm)^{-1},\\
\tnw^2 &= \nw^2(\mI+k\dw\mCw)^{-1},
\end{align} 
where $C_1$ is a normalizing constant, \ie, $\sum_m \tpi_m=1$, $\dm  = \frac{\nm^2}{\nx^2}$,
$\dw  = \frac{\nw^2}{\ny^2}$,
$\mCm = \mI$,
$\mm  = \frac{\sum_{i=1}^{k+1} \vx_i}{k+1}$, 
$\mCw = \frac{\sum_{i=1}^k \vx_i\vx_i^\top}{k}$,
and $\mw  = \frac{\sum_{i=1}^k \vx_i y_i}{k}$.
See Appendix~\ref{sec:lemma} for the proof.
\end{lemma}
\begin{remark}
Gaussian mixture is known to be a conjugate prior to the Gaussian likelihood. 
The outlined conjugate distributions in this lemma extend the Gaussian mixture conjugate distributions by substituting the Gaussian likelihood with the ``noisy linear regression'' likelihood in Remark~\ref{remark:likelihood}.
\end{remark}
Lemma~\ref{lemma:posterior} states that the task posterior remains a Gaussian mixture, with its mixture components shifted and re-weighted from the task prior.
Therefore, understanding the impact of in-context examples on the posterior requires understanding how in-context examples affect the two factors:
\begin{itemize}[topsep=0.1em, partopsep=0em, leftmargin=*]
\setlength\itemsep{0.1em} 
\item \textbf{\CSfull (\CS).} The component center is shifted from $(\vm_m,\vw_m)$ to $(\tvm_m,\tvw_m)$.
\item \textbf{\CRfull (\CR).} The component weight is re-weighted from $\pi$ to $\tpi$.
\end{itemize}

\begin{remark}
    The term ``component'' comes from the literature on Gaussian mixtures.
    It serves as an alternative to ``Task Group'' as shown in Fig.~\ref{fig:datacompare}.
    The terminology ``Component Shifting'' and ``Component Re-weighting'' can be viewed as ``Task Group Shifting'' and ``Task Group Re-weighting''.
    We will abbreviate ``mixture component center'' to simply ``center'' when there is no ambiguity.
\end{remark}
Leveraging Assumption~\ref{asu:source}, we collected mathematical analyses of \CS and \CR in Appendix~\ref{sec:detailedCSCR}.
The analysis explores the impacts of pretraining task noises and the number of in-context examples on $\tvm_m$, $\tvw_m$, and $\tpi_m$, and examines the convergence of $\tvm_m$, $\tvw_m$, and $\tpi_m$, as $k$ approaches infinity.

\subsection{Closed-form Expression of ICL Prediction}
\label{subsec:prediction}
With Assumption~\ref{asu:assumption} and Lemma~\ref{lemma:posterior}, we have the following corollary for the prediction $\sF(\Skx)$:
\begin{corollary}
\label{corollary:prediction}
Let $\tvw = \sum_{m=1}^M \tpi_m \tvw_m$.
With pretraining data generative model~\ref{asu:setting} and Assumption~\ref{asu:assumption}, if the pretrained model $\sF$ minimizes the pretraining risk, then the prediction on any sequence $\Skx$ by $\sF$ is as follows:
$\sF(\Skx) 
= \Big\langle \vx_{k+1},\sum_{m=1}^M \tpi_m \tvw_m \Big\rangle
= \langle \vx_{k+1},\tvw \rangle.$
\end{corollary}
\begin{proof}
Apply Assumption~\ref{asu:setting} to Eq.~\ref{equation:connection},
$\sF(\Skx) = \mathop{\mathbb{E}}_{(\vm,\vw)\sim \Dpr} [\langle \vx_{k+1},\vw \rangle \vert \Skx]$. 
Using Lemma~\ref{lemma:posterior}, this reduces to
$\sum_{m=1}^M \tpi_m \displaystyle\mathop{\mathbb{E}}_{(\vm,\vw)\sim \tT_m} [\langle \vx_{k+1},\vw \rangle]$.
Due to the linearity of expectation and inner product, the prediction can be simplified as
$\langle \vx_{k+1},\sum_{m=1}^M \tpi_m\tvw_m \rangle = \langle \vx_{k+1},\tvw \rangle$.
\end{proof}
Thus, the prediction is a convex combination of predictions by the centers of those shifted and re-weighted mixture components in the posterior.
We are interested in how $\pi_m$ and $\vw_m$ change to $\tpi_m$ and $\tvw_m$ with increasing $k$ and how the pretraining prior distribution properties affect these changes.

\subsection{Prior Task Noises, \CS, \CR, and ICL Prediction}
\label{subsec:twomodes}
We numerically compute how $\tpi_m$, $\tvw_m$, and the prediction $\sF(\Skx)$ evolve as $k$ increases under different prior task noise conditions.
The numerical computation is based on the tetrahedron setting with four prior mixture components as illustrated in Fig.~\ref{fig:3d:4centers:half}. See Appendix~\ref{sec:3d} for details of the tetrahedron setting.
Fig.~\ref{fig:subprediction} shows the computational results.
The first row shows the \CS effect, demonstrating the impact of increasing $k$ on $\tvw_m$.
The second row shows the \CR effect, illustrating the impact of increasing $k$ on $\tpi_m$.
The third and fourth rows depict how increasing $k$ influences the risk of learning the function $\vwt$.
We observe that with low task noises and a small $k$ value, the \CR effect initially prevails, significantly boosting the mixture weight of component $1$ over others.
Then, as $k$ increases further, the \CS effect aligns all component centers with $(\vmt,\vwt)$.

\subsection{Dual Operating Modes}
\label{subsec:twooperating}
The \textbf{\emph{``task retrieval''}} mode describes a scenario where the impact of component re-weighting surpasses that of component shifting, leading to the prediction that is primarily influenced by the interplay between pretraining priors and in-context examples.
An illustration of this is shown in the first column of Fig.~\ref{fig:subprediction}, where the re-weighting of $\tpi_m$ is more pronounced than the shifting of $\tvw_m$, indicating that \CR plays a pivotal role in altering the prediction. In contrast, the \textbf{\emph{``task learning''}} mode refers to situations where component shifting dominates over component re-weighting, resulting in the prediction almost depending on in-context examples and neglecting the pretraining priors.

\section{Early Ascent}
\label{sec:theory}
We now explain the early ascent phenomenon by analyzing a finegrained risk bound of ICL.
(See Appendix~\ref{subsec:coarse} Theorem~\ref{the:generallearning} for the coarser bound.)

\subsection{Finegrained Upper Bound}
\label{subsec:TaskLearning}
The finegrained upper bound for ICL risk is shown below:
\begin{theorem}[Finegrained Upper Bound for ICL Risk]
\label{the:finegrainedlearning}
Consider a \sp attaining the optimal pretraining risk.
As $k\rightarrow\infty$, ICL risk is upper bounded by:
\begin{align}
    \E[\mathcal{L}_k^*]
    <&
    \sum_{m=1}^M \|\vw_m-\vwt\|^2 \E_{\Skx}[\tpi_m \|\vx_{k+1}\|^2 \lambda_1(\mA)^2],
\end{align}
where $\mathcal{L}_k^*=(\F(\Skx)-y_{k+1}^*)^2=(\F(\Skx)-\langle \vx_{k+1},\vwt \rangle)^2$, $\|\vw_m-\vwt\|$ is the distance between the in-context function $\vwt$ and the function $\vw_m$ of center $m$, $\tpi_m$ is the posterior mixture weight, and $\mA = (\mI+\dw\sum_{i=1}^k\vx_i\vx_i^\top)^{-1}$.
See Appendix~\ref{proof:learning} and Eq.~\ref{equation:fine-grained} for proof details.
In Appendix~\ref{app:lowrank}, we further refine the bound for cases when in-context $\vx_i$ only spans in a subspace of $\mathbb{R}^d$, resulting in $\lambda_1(\mA)=1$ constantly.
\end{theorem}
In-context examples affect the upper bound by affecting the two factors $\tpi_\beta$ and $\lambda_1(\mA)$, corresponding to CR and CS introduced in Sec.~\ref{subsec:posterior}.
When ignoring the CR effect and only considering CS, the finegrained upper bound degrades to the general coarse bound in Appendix~\ref{subsec:coarse} Theorem~\ref{the:generallearning}.

\begin{figure*}[th!]
    \centering
    \includegraphics[width=0.8\textwidth]{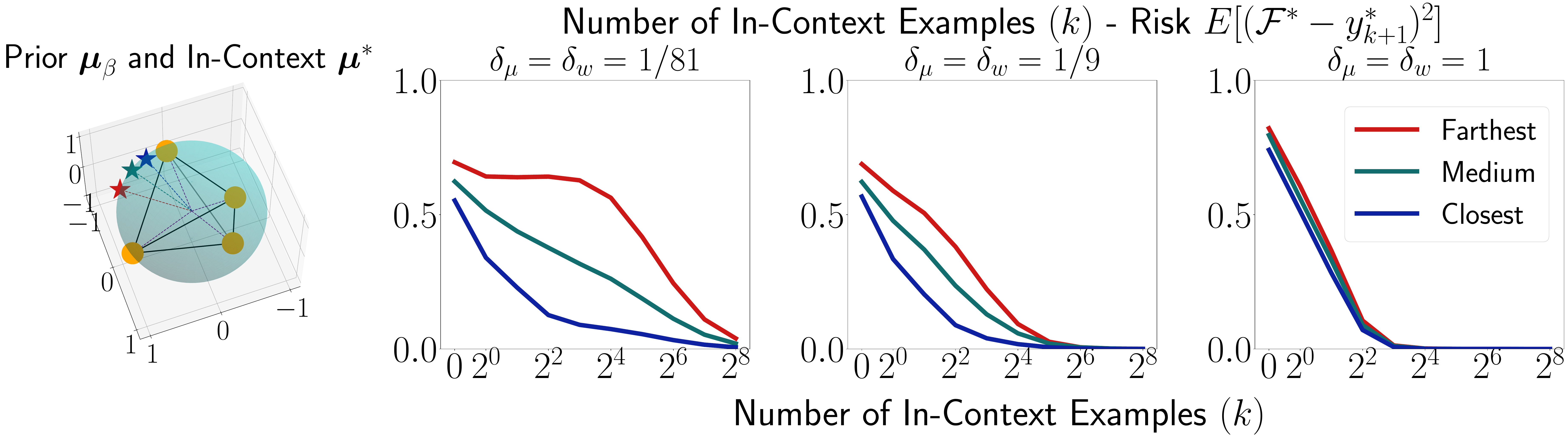}
    \caption{\textbf{Distance to the closest prior vs ICL risk.} 
    We compute ICL risks of three target tasks colored red (farthest), green, and blue (closest), under the tetrahedron setting, illustrated in the left-most figure.
    The red target task has the longest distance to the closest prior center, and the blue target task has the shortest distance to the closest prior center.
    We can observe that the target task is easier to learn when the distance to the closest prior is smaller.
    }
    \label{fig:Compare}
\end{figure*}
\subsection{The Effect of Dual Operating Modes on ICL Risk}
\label{subsubsec:compare}
We numerically compute ICL risk under varied settings to explore the effect of the dual operating modes on the risk in Fig.~\ref{fig:Compare}.
When pretraining task noises are low, \ie, $\dm$ and $\dw$ are small, the task retrieval mode happens with a small number of in-context examples, and the upper bound is affected by how $(\vmt,\vwt)$ is close to a prior center.
Specifically, the task prior boosts the learning process of ICL if the in-context task is close to a prior center, due to the task retrieval mode quickly retrieving the task of the nearest prior center.
\begin{figure*}[th!]
    \centering
    \subfigure[Risk and $\pi_m$ as $k$ increases under $d\in\{1,3,8\}$.]{
        \includegraphics[width=0.4\textwidth]{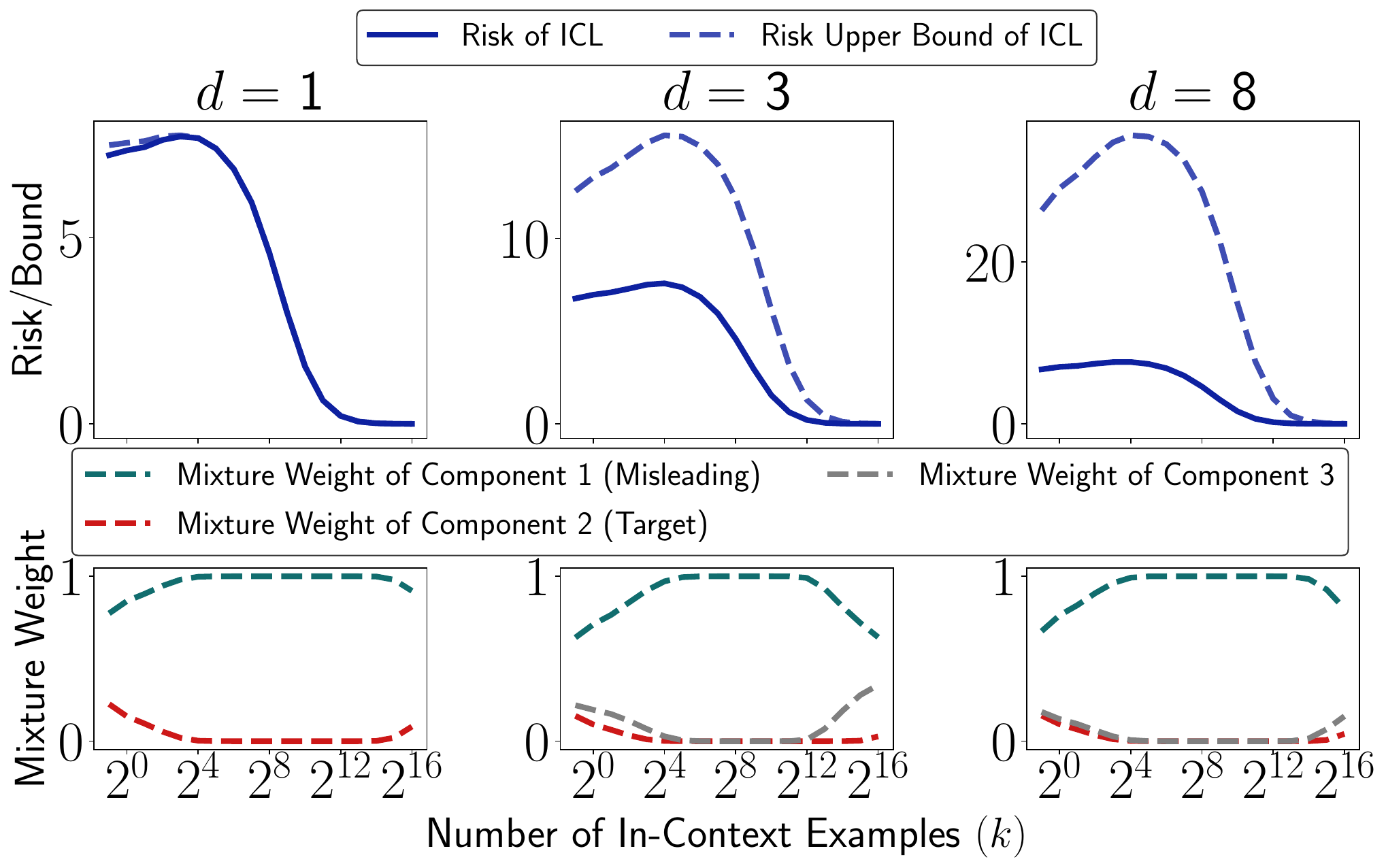}
        \label{fig:UB}}
    \hspace{0.5em}
    \subfigure[Expectation of $\tilde{\vw}$ as $k$ increases under $d\in\{2,3\}$.]{
        \includegraphics[width=0.5\textwidth]{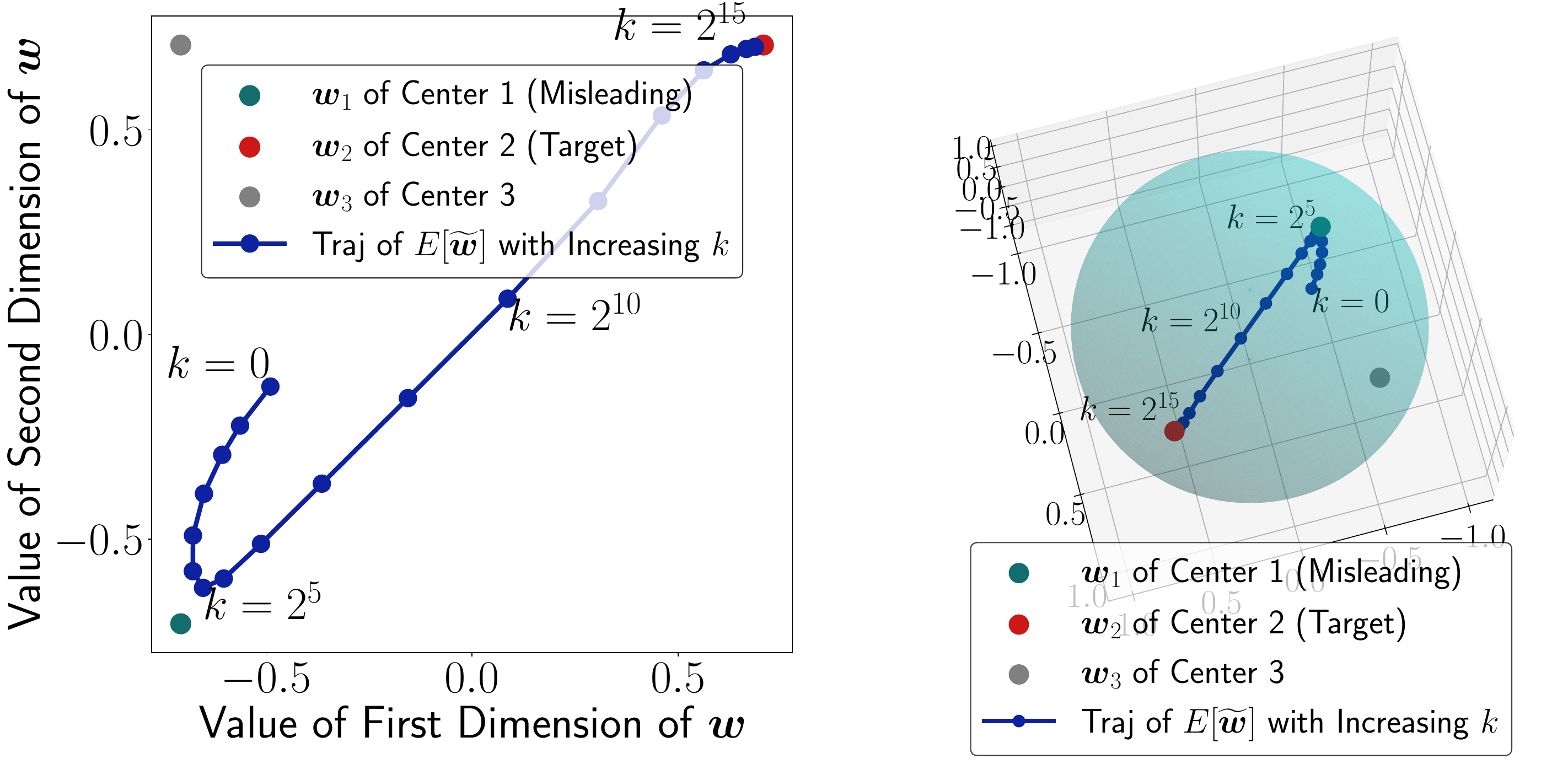}
        \label{fig:traj}}
    \caption{\textbf{The early ascent phenomenon.} 
    Fig.~\ref{fig:UB} and Fig.~\ref{fig:traj} show that the task retrieval mode is dominant up to $k=32$, and component 1's mixture weight increases ($\E[\tilde{\vw}]$ approaches $\vw_1$).
    Since this component is farther than the other one, the risk starts increasing. 
    At larger $k$ values, the risk starts decreasing ($\E[\tilde{\vw}]$ approaches $\vw_2$) via task learning. See Appendix~\ref{subsec:early} for setting details. We further examine the early ascent phenomenon under linear regression with varied levels of label noises in Appendix~\ref{app:expnoisy}, and under non-linear regression and discrete token prediction in Appendix~\ref{app:earlyascent:nonlinear&discrete}.}
    \label{fig:Ushape}
\end{figure*}

\subsection{Early Ascent with Biased $x$ Distribution}
\label{subsec:finegrained}
However, the task retrieval mode may not always benefit ICL.
We notice a weird phenomenon is observed by~\citet{BrownMRSKDNSSAA20} and~\citet{xie2021explanation}.
As the number of in-context samples increased, the performance of ICL first decreased and then increased.
\citet{BrownMRSKDNSSAA20} reports that GPT-3 on LAMBADA shows a lower one-shot accuracy (72.5\%) than zero-shot accuracy (76.2\%), but the few-shot accuracy (86.4\%) is higher than the zero-shot accuracy.
\citet{xie2021explanation} also replicated this phenomenon with their synthetic dataset.
\citet{xie2021explanation} explains this by ``the few-shot setting introduces the distracting prompt structure, which can initially lower accuracy.''

To obtain some insights, we present a simple scenario where $\vx$ misleads the prediction by an LLM. 
Consider the following one-shot prompt for English-to-Korean translation: ``What is the color of apple? \begin{CJK}{UTF8}{mj}사과의 색깔은 무엇인가?\end{CJK}\footnote{``What is the color of apple?'' in Korean.} What is the color of banana?''
The correct answer should be ``\begin{CJK}{UTF8}{mj}바나나의 색깔은  무엇인가?\end{CJK}''\footnote{``What is the color of banana?'' in Korean.}
However, GPT-3.5 generates ``\begin{CJK}{UTF8}{mj}바나나의 색깔은 노란색 입니다\end{CJK},'' which means ``The color of bananas is yellow.''
This shows that pretrained LLMs could retrieve an incorrect skill (question answering in this example) by observing misleading input ($\vx$). 

Based on our analysis, we further show that the early ascent phenomenon provably occurs under a certain assumption Appendix~\ref{app:earlyascent:example}. 
We also reproduce early ascent in Fig.~\ref{fig:UB}, where the upper bound and the risk initially increase due to the misleading task (of center 1) is retrieved first.
Fig.~\ref{fig:traj} further demonstrates the relative locations of the retrieved functions to functions of prior centers.
Finally, we give the formal theorem on the early ascent phenomenon:
\begin{theorem}[Early Ascent]
\label{the:earlyascent}
Assume $\alpha = \argmin\limits_m 
\frac{
    \|\vm_m-\vmt\|^2
}{2\nx^2}
+
\frac{
    \|(\vw_m-\vwt)^\top\vmt\|^2 + d\tau_x^2\|\vw_m-\vwt\|^2
}{2\ny^2}
$ is the most misleading task and the task $\alpha$ satisfies $\E_{\vx_1}\left[\left(\F^*(\vx_1)-\langle \vwt,\vx_1 \rangle\right)^2\right] < \E_{\vx_1}\left[\langle\vx_1, \vw_\alpha-\vwt \rangle^2 \right]$.
Then, when $\dm$ and $\dw$ are small enough, $\exists k\geq1 \text{ s.t.}$:
\begin{align}
&    
    \E_{\vx_1}\left[\left(\F^*(\vx_1)-\langle \vwt,\vx_1 \rangle\right)^2\right]
\\&<
    \E_{\Skx}\left[\left(\F^*(\Skx)-\langle \vwt,\vx_{k+1} \rangle\right)^2\right],
\end{align}
where $\E_{\vx_1}\left[\langle\vx_1, \vw_\alpha-\vwt \rangle^2 \right]$ equals to the risk when the prediction fully depends on the misleading task function $\vw_\alpha$ of prior center $\alpha$.
See Appendix~\ref{app:earlyascent:theory} for proof details.
\end{theorem}
Theorem~\ref{the:earlyascent} shows that, if the misleading task $\alpha$ has a higher risk than the zero-shot risk, then when $\dm$ and $\dw$ are small enough, the early ascent phenomenon happens.

\section{Bounded Efficacy of Biased-Label ICL}
\label{subsec:TaskRetrieval}
We further predict the bounded efficacy phenomenon by examining the bound of ICL with biased labels.
The assumption for ICL with biased labels is described as follows:
\begin{asu}[ICL with Biased Labels]
\label{asu:misaligned}
The function $\vwt$ of ICL with biased labels is different from the target function $\vw_\alpha$, \ie, $\vwt\neq\vw_\alpha$ where $\vw_\alpha$ is a function of a pretraining task prior center.
The in-context task is closer to the prior center $\alpha$  compared to all the other prior centers $\beta\neq\alpha$:\\
$\forall \beta \neq \alpha, \|\vm_\beta-\vmt\|^2 - \|\vm_\alpha-\vmt\|^2 \geq d^2_\vm, \|\vw_\beta-\vwt\|^2 - \|\vw_\alpha-\vwt\|^2 \geq d^2_\vw$, and $\nxd^2 \|\vw_\beta-\vwt\|^2 - (1+\nxd^2)\|\vw_\alpha-\vwt\|^2 \geq \nxd^2 u^2_\vw$.
\end{asu}
Assumption~\ref{asu:misaligned} depicts that to retrieve $\vw_\alpha$ associated with the prior center $\alpha$, the in-context task is selected based on its proximity to center $\alpha$, ensuring it is closer to center $\alpha$.

\subsection{Upper Bound for ICL Risk with Biased Labels}
\label{subsubsec:retrieval}
The following theorem shows an upper bound for ICL risk with biased labels to retrieve a task:
\begin{theorem}[Upper Bound for ICL Risk with Biased Labels]
\label{the:retrieval}
Consider a \sp attaining the optimal pretraining risk.
As $k\rightarrow\infty$, ICL risk with biased labels is upper bounded by:
\begin{align}
    \E_{\Sk}[\mathcal{L}_k^\alpha]
&<
    \|\vw_\alpha-\vwt\|^2 (1+d\nxd^2)
\\&~~
    +
    \frac{C_1}{\redb{k}\blueb{\dw}} 
    \exp\left(C_2 \redb{k}^{\frac{\delta}{2}-\frac{3}{4}}
    \right)
    +O(\redb{k}^{-2})
\end{align}
where $\mathcal{L}_k^\alpha=(\F(\Skx)-y_{k+1}^\alpha)^2=(\F(\Skx)-\langle \vx_{k+1},\vw_\alpha \rangle)^2$.
When $\dm$ and $\dw$ are sufficiently small, exists a particular interval for $k$ s.t.:
\begin{align}
&    \E_{\Sk}[\mathcal{L}_k^\alpha]
<
\|\vw_\alpha-\vwt\|^2 (1+d\nxd^2)\min\{1,4\redb{k}^2\blueb{\dw}^2(1+\nxd^2)^2\}
\\&\quad
    +
    C_3
    \exp\left(-\redb{k}\left(\frac{
        d_\vm^2
    }{
        8\nx^2
    }+\frac{
        u_\vw^2\nxd^2
    }{
        8\ny^2
    }\right)\right)
    +
    C_4\exp\left(-\frac{\redb{k}^{\frac{1}{2}}}{8}\right).
\end{align}
As $k$ increases, the second and third terms dominate and exponential decay when $k$ is small,
and the first term dominates and increases when $k$ is large.
$C_1, C_2, C_3$, and $C_4$ are constants depending on the prior setting, $\nxd$, and $(\vmut,\vwt)$.
See Appendix~\ref{app:retrieval} for proof details.
\end{theorem}

\begin{table}[th!]
\centering
\resizebox{0.48\textwidth}{!}{
\begin{tabular}{lrrrrrr}
\toprule
$k$   & 0       & 1      & 2      & 4      & 8      & 16     \\ \midrule
$+$           & 75.0\%  & 36.2\% & \underline{\textbf{33.9\%}} & \underline{49.3\%} & \underline{79.3\%} & \underline{85.1\%} \\
Biased $+$      & 100.0\% & 98.3\% & 95.9\% & 60.5\% & 24.4\% & \textbf{16.8\%}\\ \bottomrule
\end{tabular}}
\caption{\textbf{Bounded efficacy in GPT-4.} Error rate measured with respect to ``addition ($+$)'' and ``biased $+$''.
The bounded efficacy phenomenon: the error rate goes down to $k=2$, but it increases afterward. Experiment details in Appendix~\ref{app:GPT4U}.}
\label{table:U}
\vspace{-0.1 cm}
\end{table}
\subsection{Bounded Efficacy of Biased-Label ICL in GPT-4}
\label{subsec:GPT4U}
This section further shows that the bounded efficacy phenomenon exists in GPT-4 in Table~\ref{table:U}.
With the task ``biased addition ($+$)'' as the in-context task corresponding to $\vwt$, as the number of in-context examples increases, ICL will first retrieve the skill ``addition ($+$)'' corresponding to $\vw_\alpha$ which has a strong pretraining prior. Later, it will learn the ``biased $+$'' task, leading to the bounded efficacy phenomenon.

\begin{figure}[th!]
    \centering
    \includegraphics[width=0.48\textwidth]{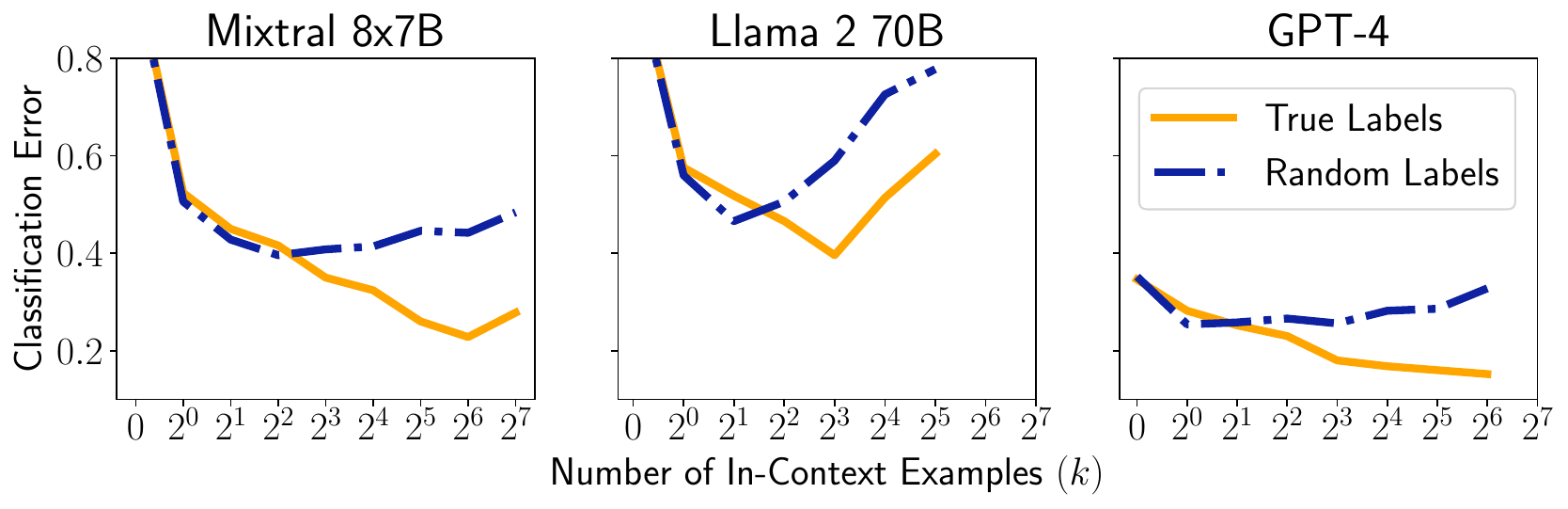}
    \caption{
    \textbf{Bounded efficacy.} The error rates of ICL with random labels start increasing at large $k$. 
    See Appendix~\ref{app:exp:zeroicl} for more experimental results.
    }
    \label{fig:ZeroICL}
\end{figure}

\subsection{Bounded Efficacy for Zero-Shot ICL}
\label{subsubsec:zicl}
We further introduce Lemma~\ref{lemma:zeroicl}, a variation of the previous Theorem~\ref{the:retrieval}, to explain zero-shot ICL, an ICL algorithm capable of functioning with random labels~\citep{LyuMBZH23}.
\begin{lemma}[(informal) Upper Bound for Zero-Shot ICL]
\label{lemma:zeroicl}
Assume a \sp attains the optimal pretraining risk, the risk of ICL with random labels (provide no information) will reveal a bounded efficacy phenomenon. See Appendix~\ref{app:zeroiclproof} for proof details.
\end{lemma}
Lemma~\ref{lemma:zeroicl} says that as the number of in-context examples increases, the loss curve of zero-shot ICL with random labels will have the bounded efficacy phenomenon, which conflicts with the observation from~\citet{MinLHALHZ22} that ICL with random labels has very similar performance as ICL with true labels for the number of in-context examples ranging from $1$ to $32$.
We believe this observation is due to the small number of in-context examples.
Thus, we extend the experiment of~\citet{MinLHALHZ22} to explore the number of in-context examples beyond 32.
Due to LLMs' context lengths constraining the maximum number of in-context examples, we choose different LLMs from~\citet{MinLHALHZ22} for a larger context length capacity.

Fig.~\ref{fig:ZeroICL} highlights the bounded efficacy phenomenon in the error curve associated with random labels.
Compared with true labels, the error rate of ICL with random labels increases at a much smaller $k$ value, clearly exhibiting the bounded efficacy phenomenon we predicted.

\section{Conclusion}
In this paper, we introduced a probabilistic model for understanding the dual operating modes of in-context learning: task learning and task retrieval. 
Our analysis allowed us to explain the existing early ascent phenomenon observed in real-world ICL applications, and predict a new bounded efficacy phenomenon of biased-label ICL.
We validated our findings and predictions via experiments involving large language models.
Our work lays the groundwork for future research in further exploration and improvement of ICL.

We conclude our paper with the limitations of our current framework: (i) the gap between our assumed pretraining linear regression tasks and complex, non-linear, categorical, real-world pretraining tasks of LLMs; (ii) the labels of in-context samples are assumed to be noiseless.


\section*{Acknowledgements}
This work was supported by the NSF Award DMS-2023239, NSF CAREER Award CCF-2339978, Amazon Research Award, and a grant from FuriosaAI.

We would like to express our sincere gratitude to Kartik Sreenivasan for his invaluable discussions for this research.
His insights and expertise have been instrumental in shaping this study.
Additionally, we sincerely thank Andrew Geng for his contributions to coding for the initial experimental setup.
His skills and dedication have been pivotal in the early stages of our research.

\section*{Impact Statement}
This paper presents work whose goal is to advance the field of Machine Learning. There are many potential societal consequences of our work, none which we feel must be specifically highlighted here.

\bibliography{iclr2024_conference}

\begin{thebibliography}{37}
\providecommand{\natexlab}[1]{#1}
\providecommand{\url}[1]{\texttt{#1}}
\expandafter\ifx\csname urlstyle\endcsname\relax
  \providecommand{\doi}[1]{doi: #1}\else
  \providecommand{\doi}{doi: \begingroup \urlstyle{rm}\Url}\fi

\bibitem[Ahn et~al.(2023)Ahn, Cheng, Daneshmand, and Sra]{ahn23learnGD}
Ahn, K., Cheng, X., Daneshmand, H., and Sra, S.
\newblock Transformers learn to implement preconditioned gradient descent for in-context learning.
\newblock In \emph{Advances in Neural Information Processing Systems (NeurIPS)}, 2023.

\bibitem[Aky{\"{u}}rek et~al.(2023)Aky{\"{u}}rek, Schuurmans, Andreas, Ma, and Zhou]{akyurek2022learning}
Aky{\"{u}}rek, E., Schuurmans, D., Andreas, J., Ma, T., and Zhou, D.
\newblock What learning algorithm is in-context learning? {I}nvestigations with linear models.
\newblock In \emph{International Conference on Learning Representations (ICLR)}, 2023.

\bibitem[Bai et~al.(2023)Bai, Chen, Wang, Xiong, and Mei]{yu23algorithm}
Bai, Y., Chen, F., Wang, H., Xiong, C., and Mei, S.
\newblock Transformers as statisticians: {P}rovable in-context learning with in-context algorithm selection.
\newblock In \emph{Advances in Neural Information Processing Systems (NeurIPS)}, 2023.

\bibitem[Barbieri et~al.(2020)Barbieri, Camacho{-}Collados, Anke, and Neves]{BarbieriCAN20}
Barbieri, F., Camacho{-}Collados, J., Anke, L.~E., and Neves, L.
\newblock Tweeteval: Unified benchmark and comparative evaluation for tweet classification.
\newblock In \emph{Findings of the Association for Computational Linguistics: EMNLP}, 2020.

\bibitem[Boucheron et~al.(2013)Boucheron, Lugosi, and Massart]{boucheron2013concentration}
Boucheron, S., Lugosi, G., and Massart, P.
\newblock \emph{Concentration inequalities: A nonasymptotic theory of independence}.
\newblock Oxford University Press, 2013.

\bibitem[Brown et~al.(2020)Brown, Mann, Ryder, Subbiah, Kaplan, Dhariwal, Neelakantan, Shyam, Sastry, Askell, Agarwal, Herbert{-}Voss, Krueger, Henighan, Child, Ramesh, Ziegler, Wu, Winter, Hesse, Chen, Sigler, Litwin, Gray, Chess, Clark, Berner, McCandlish, Radford, Sutskever, and Amodei]{BrownMRSKDNSSAA20}
Brown, T.~B., Mann, B., Ryder, N., Subbiah, M., Kaplan, J., Dhariwal, P., Neelakantan, A., Shyam, P., Sastry, G., Askell, A., Agarwal, S., Herbert{-}Voss, A., Krueger, G., Henighan, T., Child, R., Ramesh, A., Ziegler, D.~M., Wu, J., Winter, C., Hesse, C., Chen, M., Sigler, E., Litwin, M., Gray, S., Chess, B., Clark, J., Berner, C., McCandlish, S., Radford, A., Sutskever, I., and Amodei, D.
\newblock Language models are few-shot learners.
\newblock In \emph{Advances in Neural Information Processing Systems (NeurIPS)}, 2020.

\bibitem[Dagan et~al.(2005)Dagan, Glickman, and Magnini]{DaganGM05}
Dagan, I., Glickman, O., and Magnini, B.
\newblock The {PASCAL} recognising textual entailment challenge.
\newblock In \emph{{PASCAL} Machine Learning Challenges Workshop (MLCW)}, 2005.

\bibitem[Dai et~al.(2023)Dai, Sun, Dong, Hao, Ma, Sui, and Wei]{dai2023metaopt}
Dai, D., Sun, Y., Dong, L., Hao, Y., Ma, S., Sui, Z., and Wei, F.
\newblock Why can {GPT} learn in-context? {L}anguage models secretly perform gradient descent as meta-optimizers.
\newblock In \emph{Findings of the Association for Computational Linguistics (ACL)}, 2023.

\bibitem[Dolan \& Brockett(2005)Dolan and Brockett]{DolanB05}
Dolan, W.~B. and Brockett, C.
\newblock Automatically constructing a corpus of sentential paraphrases.
\newblock In \emph{International Workshop on Paraphrasing (IWP@IJCNLP)}, 2005.

\bibitem[Garg et~al.(2022)Garg, Tsipras, Liang, and Valiant]{garg2022can}
Garg, S., Tsipras, D., Liang, P.~S., and Valiant, G.
\newblock What can {T}ransformers learn in-context? {A} case study of simple function classes.
\newblock In \emph{Advances in Neural Information Processing Systems (NeurIPS)}, 2022.

\bibitem[Ghahramani \& Jordan(1995)Ghahramani and Jordan]{ghahramani1995factorial}
Ghahramani, Z. and Jordan, M.
\newblock Factorial hidden markov models.
\newblock In \emph{Advances in Neural Information Processing Systems (NeurIPS)}, 1995.

\bibitem[Giannou et~al.(2023)Giannou, Rajput, Sohn, Lee, Lee, and Papailiopoulos]{giannou2023looped}
Giannou, A., Rajput, S., Sohn, J.-y., Lee, K., Lee, J.~D., and Papailiopoulos, D.
\newblock Looped {T}ransformers as programmable computers.
\newblock In \emph{International Conference on Machine Learning (ICML)}, 2023.

\bibitem[Han et~al.(2023)Han, Wang, Zhao, and Ji]{han23kernel}
Han, C., Wang, Z., Zhao, H., and Ji, H.
\newblock In-context learning of large language models explained as kernel regression.
\newblock \emph{arXiv preprint arXiv:2305.12766}, 2023.

\bibitem[Jeon et~al.(2024)Jeon, Lee, Lei, and Van~Roy]{jeon2024information}
Jeon, H.~J., Lee, J.~D., Lei, Q., and Van~Roy, B.
\newblock An information-theoretic analysis of in-context learning.
\newblock \emph{arXiv preprint arXiv:2401.15530}, 2024.

\bibitem[Jiang et~al.(2023)Jiang, Sablayrolles, Mensch, Bamford, Chaplot, Casas, Bressand, Lengyel, Lample, Saulnier, et~al.]{jiang2023mistral}
Jiang, A.~Q., Sablayrolles, A., Mensch, A., Bamford, C., Chaplot, D.~S., Casas, D. d.~l., Bressand, F., Lengyel, G., Lample, G., Saulnier, L., et~al.
\newblock Mistral 7{B}.
\newblock \emph{arXiv preprint arXiv:2310.06825}, 2023.

\bibitem[Jiang et~al.(2024)Jiang, Sablayrolles, Roux, Mensch, Savary, Bamford, Chaplot, Casas, Hanna, Bressand, et~al.]{jiang2024mixtral}
Jiang, A.~Q., Sablayrolles, A., Roux, A., Mensch, A., Savary, B., Bamford, C., Chaplot, D.~S., Casas, D. d.~l., Hanna, E.~B., Bressand, F., et~al.
\newblock Mixtral of experts.
\newblock \emph{arXiv preprint arXiv:2401.04088}, 2024.

\bibitem[Li et~al.(2023)Li, Ildiz, Papailiopoulos, and Oymak]{Li23algorithm}
Li, Y., Ildiz, M.~E., Papailiopoulos, D., and Oymak, S.
\newblock Transformers as algorithms: Generalization and stability in in-context learning.
\newblock In \emph{International Conference on Machine Learning (ICML)}, 2023.

\bibitem[Loshchilov \& Hutter(2019)Loshchilov and Hutter]{loshchilov2017decoupled}
Loshchilov, I. and Hutter, F.
\newblock Decoupled weight decay regularization.
\newblock In \emph{International Conference on Learning Representations (ICLR)}, 2019.

\bibitem[Lyu et~al.(2023)Lyu, Min, Beltagy, Zettlemoyer, and Hajishirzi]{LyuMBZH23}
Lyu, X., Min, S., Beltagy, I., Zettlemoyer, L., and Hajishirzi, H.
\newblock {Z-ICL}: Zero-shot in-context learning with pseudo-demonstrations.
\newblock In \emph{Annual Meeting of the Association for Computational Linguistics (ACL)}, 2023.

\bibitem[Mahankali et~al.(2024)Mahankali, Hashimoto, and Ma]{mahankali2023step}
Mahankali, A., Hashimoto, T.~B., and Ma, T.
\newblock One step of gradient descent is provably the optimal in-context learner with one layer of linear self-attention.
\newblock In \emph{International Conference on Learning Representations (ICLR)}, 2024.

\bibitem[Marelli et~al.(2014)Marelli, Menini, Baroni, Bentivogli, Bernardi, and Zamparelli]{MarelliMBBBZ14}
Marelli, M., Menini, S., Baroni, M., Bentivogli, L., Bernardi, R., and Zamparelli, R.
\newblock A {SICK} cure for the evaluation of compositional distributional semantic models.
\newblock In \emph{International Conference on Language Resources and Evaluation (LREC)}, 2014.

\bibitem[Min et~al.(2022)Min, Lyu, Holtzman, Artetxe, Lewis, Hajishirzi, and Zettlemoyer]{MinLHALHZ22}
Min, S., Lyu, X., Holtzman, A., Artetxe, M., Lewis, M., Hajishirzi, H., and Zettlemoyer, L.
\newblock Rethinking the role of demonstrations: What makes in-context learning work?
\newblock In \emph{Empirical Methods in Natural Language Processing (EMNLP)}, 2022.

\bibitem[OpenAI(2023)]{openai2023gpt4}
OpenAI.
\newblock {GPT}-4 technical report, 2023.

\bibitem[Pan et~al.(2023)Pan, Gao, Chen, and Chen]{pan2023context}
Pan, J., Gao, T., Chen, H., and Chen, D.
\newblock What in-context learning “learns” in-context: Disentangling task recognition and task learning.
\newblock In \emph{Findings of the Association for Computational Linguistics (ACL)}, 2023.

\bibitem[Rabiner(1989)]{rabiner1989tutorial}
Rabiner, L.~R.
\newblock A tutorial on hidden markov models and selected applications in speech recognition.
\newblock \emph{Proceedings of the IEEE}, 1989.

\bibitem[Raventos et~al.(2023)Raventos, Paul, Chen, and Ganguli]{raventos2023effects}
Raventos, A., Paul, M., Chen, F., and Ganguli, S.
\newblock The effects of pretraining task diversity on in-context learning of ridge regression.
\newblock In \emph{ICLR Workshop on Mathematical and Empirical Understanding of Foundation Models (ME-FoMo)}, 2023.

\bibitem[Razeghi et~al.(2022)Razeghi, IV, Gardner, and Singh]{RazeghiL0022}
Razeghi, Y., IV, R. L.~L., Gardner, M., and Singh, S.
\newblock Impact of pretraining term frequencies on few-shot numerical reasoning.
\newblock In \emph{Findings of the Association for Computational Linguistics: EMNLP}, 2022.

\bibitem[Sheng \& Uthus(2020)Sheng and Uthus]{sheng2020investigating}
Sheng, E. and Uthus, D.
\newblock Investigating societal biases in a poetry composition system.
\newblock In \emph{Workshop on Gender Bias in Natural Language Processing}, 2020.

\bibitem[Touvron et~al.(2023)Touvron, Martin, Stone, Albert, Almahairi, Babaei, Bashlykov, Batra, Bhargava, Bhosale, et~al.]{touvron2023llama}
Touvron, H., Martin, L., Stone, K., Albert, P., Almahairi, A., Babaei, Y., Bashlykov, N., Batra, S., Bhargava, P., Bhosale, S., et~al.
\newblock Llama 2: Open foundation and fine-tuned chat models.
\newblock \emph{arXiv preprint arXiv:2307.09288}, 2023.

\bibitem[Tsigler \& Bartlett(2023)Tsigler and Bartlett]{Alexander23}
Tsigler, A. and Bartlett, P.~L.
\newblock Benign overfitting in ridge regression.
\newblock \emph{Journal of Machine Learning Research (JMLR)}, 2023.

\bibitem[Van~Trees(2004)]{van2004detection}
Van~Trees, H.~L.
\newblock \emph{Detection, estimation, and modulation theory, {P}art {I}: Detection, estimation, and linear modulation theory}.
\newblock John Wiley \& Sons, 2004.

\bibitem[Vaswani et~al.(2017)Vaswani, Shazeer, Parmar, Uszkoreit, Jones, Gomez, Kaiser, and Polosukhin]{vaswani2017attention}
Vaswani, A., Shazeer, N., Parmar, N., Uszkoreit, J., Jones, L., Gomez, A.~N., Kaiser, {\L}., and Polosukhin, I.
\newblock Attention is all you need.
\newblock In \emph{Advances in Neural Information Processing Systems (NeurIPS)}, 2017.

\bibitem[Vershynin(2018)]{vershynin2018high}
Vershynin, R.
\newblock \emph{High-dimensional probability: An introduction with applications in data science}, volume~47.
\newblock Cambridge university press, 2018.

\bibitem[von Oswald et~al.(2023)von Oswald, Niklasson, Randazzo, Sacramento, Mordvintsev, Zhmoginov, and Vladymyrov]{von2022transformers}
von Oswald, J., Niklasson, E., Randazzo, E., Sacramento, J., Mordvintsev, A., Zhmoginov, A., and Vladymyrov, M.
\newblock Transformers learn in-context by gradient descent.
\newblock In \emph{International Conference on Machine Learning (ICML)}, 2023.

\bibitem[Wu et~al.(2024)Wu, Zou, Chen, Braverman, Gu, and Bartlett]{anonymous2023how}
Wu, J., Zou, D., Chen, Z., Braverman, V., Gu, Q., and Bartlett, P.~L.
\newblock How many pretraining tasks are needed for in-context learning of linear regression?
\newblock In \emph{International Conference on Learning Representations (ICLR)}, 2024.

\bibitem[Xie et~al.(2022)Xie, Raghunathan, Liang, and Ma]{xie2021explanation}
Xie, S.~M., Raghunathan, A., Liang, P., and Ma, T.
\newblock An explanation of in-context learning as implicit {B}ayesian inference.
\newblock In \emph{International Conference on Learning Representations (ICLR)}, 2022.

\bibitem[Zhang et~al.(2023)Zhang, Frei, and Bartlett]{zhang23learnGD}
Zhang, R., Frei, S., and Bartlett, P.~L.
\newblock Trained transformers learn linear models in-context.
\newblock In \emph{Robustness of Few-shot and Zero-shot Learning in Large Foundation Models (R0-FoMo)}, 2023.

\end{thebibliography}
\bibliographystyle{icml2024}

\appendix
\newpage
\onecolumn


\section{Notations}
\label{app:notation}
This section collects all notations used in the main paper. 

\paragraph{Notations introduced in Sec.~\ref{sec:setting&connection}:}
\begin{itemize}[topsep=0.1em, partopsep=0em, leftmargin=*]
\setlength\itemsep{0.1em} 
    \item $\F$: a \sp.
    \item $\hF$: a pretrained \sp.
    \item $\sF$: a \bosp that attains Bayes risk minimization.
    \item $\F_k$: a \sp for $k$ in-context examples.
    \item $\sF_k$: a \bosp that attains Bayes risk minimization for $k$ in-context examples.
    \item $\vx$ and $y$: input and label for a task, \eg, $\vx$ and $y$ of a linear regression task $y=\vx^\top\vw$.
    \item $k$: the number of in-context examples.
    \item $K$: the max number of examples in a sequence.
    \item $\S_k$: a sequence of $k$ in-context examples, $[\vx_1,y_1,\ldots,\vx_k,y_k]$.
    \item $\S_K$: a sequence of $K$ in-context examples, $[\vx_1,y_1,\ldots,\vx_K,y_K]$.
    \item $\Skx$: $\Skx=[\vx_1,y_1,\ldots,\vx_k,y_k,\vx_{k+1}]$, which is a sequence of $k$ in-context examples appended with $\vx_{k+1}$.
    \item $\vm$ and $\vw$: the parameters that jointly specify a task. $\vm$ specifies the distribution of $\vx$, and $\vw$ specifies the function mapping $\vx$ to $y$.
    \item $\Dpr$ and $\mathcal{D}_{\vm,\vw}$: $\Dpr=\mathcal{D}_{\vm,\vw}$, and they represent the task prior distribution where each task is specified by parameters $\vm$ and $\vw$. The task prior is also named pretraining prior, pretraining task prior, pretraining prior distribution, pretraining task prior distribution, or simply prior. 
    \item $\mathcal{D}_\vx(\vm)$: the conditional distribution of $\vx$ conditioned on $\vm$ of the task $(\vm,\vw)$.
    \item $\mathcal{D}_{\vx,y}(\vm,\vw)$: the joint distribution of $(\vx,y)$ in the task $(\vm,\vw)$.
    \item $\mathcal{D}_{y\vert\vx}(\vw)$: $y$ distribution conditioned on the input $\vx$ and parameter $\vw$ of the task $(\vm,\vw)$.
    \item $P(\vm,\vw)$: the task probability of $(\vm,\vw)$ in the task prior $\Dpr$.
    \item $P(\vx \vert \vm)$: the probability of $\vx$ in $\mathcal{D}_\vx(\vm)$.
    \item $P(y \vert \vx, \vw)$: the probability of $y$ in $\mathcal{D}_{y \vert \vx}(\vw)$.
    \item $\mathcal{L}(\F)$: the risk of $\F$ on samples generated from the pretraining data generative model~\ref{asu:setting}.

    \item $M$: the number of mixture components in a Gaussian mixture prior.
    \item $\mathcal{N}(\vx; \vm,\bm{\Sigma})$: the probability of $\vx$ in the multivariate normal distribution with mean $\vm$ and covariance matrix $\bm{\Sigma}$.
    \item $m$, $\alpha$, and $\beta$: the indices of mixture components in a Gaussian mixture prior.
    \item $T_m$: the $m^\text{the}$ mixture component in a Gaussian mixture prior.
    \item $\pi_m$: the mixture weight of the $m^\text{th}$ mixture component in a Gaussian mixture prior.
    \item $\vm_m$ and $\vw_m$: $(\vm_m,\vw_m)$ is the center of the $m^\text{th}$ mixture component.
    \item $\vmt$ and $\vwt$: $(\vmt,\vwt)$ is the in-context task, \ie, in-context examples are drawn from this task without label noises.
    \item $\nm$ and $\nw$: the task noises, \ie, the noise scales of $\vm$ and $\vw$.
    \item $\nx$ and $\ny$: the sample noises, \ie, the noise scales of $\vx$ and $y$ of pretraining samples.
    \item $\nxd$: the sample noise, \ie, the noise scale of $\vx$ of in-context examples.
    \item $d$: the dimension of $\vx$.
    \item $r$: the max ratio of two mixture weights of two mixture components.
\end{itemize}

\paragraph{Notations introduced in Sec.~\ref{sec:posterior&phenomena}:}
\begin{itemize}[topsep=0.1em, partopsep=0em, leftmargin=*]
\setlength\itemsep{0.1em} 
    \item $\Dpo$: The posterior distribution of the pretraining prior $\Dpr$ after observing $\Skx$.
    \item $\|\cdot\|$: the $L_2$ norm.
    \item $\|\vx\|^2$: for any vector $\vx$, $\|\vx\|^2=\vx^\top\vx$.
    \item $\|\vx\|^2_\mA$: for any vector $\vx$ and matrix $\mA$, $\|\vx\|^2_\mA=\vx^\top\mA\vx$.
    \item $P(\vm, \vw \vert \Skx)$: the probability of task $(\vm,\vw)$ in the posterior after observing $\Skx$.
    \item $\tT_m$: the $m^\text{th}$ mixture component in the Gaussian mixture posterior.
    \item $\tpi_m$: the mixture weight of the $m^\text{th}$ mixture component in the Gaussian mixture posterior.
    \item $\tvm_m$ and $\tvw_m$: $(\tvm_m,\tvw_m)$ is the center of the $m^\text{th}$ mixture component in the Gaussian mixture posterior.
    \item $P(\vm, \vw \vert \tT_m)$: the probability of task $(\vm,\vw)$ in the $m^\text{th}$ mixture component of posterior.
    \item $\dm$ and $\dw$: the ratios of squared task noises over squared sample noises. $\dm=\frac{\nm^2}{\nx^2}$, and $\dw=\frac{\nw^2}{\ny^2}$.
    \item $\mCm$: $\mCm = \mI$.
    \item $\mCw$: $\mCw = \frac{\sum_{i=1}^k \vx_i\vx_i^\top}{k}$.
    \item $\mm$: $\mm = \frac{\sum_{i=1}^{k+1} \vx_i}{k+1}$.
    \item $\mw$: $\mw = \frac{\sum_{i=1}^k \vx_i y_i}{k}$.
    \item $\tvw$: the mean of $\vw$ in the task posterior, \ie, the predicted function by \bosp. $\sF(\Skx)=\langle \vx_{k+1}, \tvw \rangle = \Big\langle \vx_{k+1},\sum_{m=1}^M \tpi_m \tvw_m \Big\rangle$.
    \item $c_m^\vm$ and $c_m^\vw$: parts of the re-weighting coefficient of \CRfull.
    \item $\advm(\alpha,\beta)$ and $\advw(\alpha,\beta)$: functions to help analyze the phenomenon of \CRfull. 
    \item $r(\alpha,\beta)$: the ratio of the mixture weight $\tpi_\alpha$ of $\tT_\alpha$ over the mixture weight $\tpi_\beta$ of $\tT_\beta$. 
    \item $\lambda_d(\mA)$: the $d^\text{th}$ largest eigenvalue of matrix $\mA$. In this paper $\mA\in\mathbb{R}^{d\times d}$, thus $\lambda_d(\mA)$ represents the smallest eigenvalue of matrix $\mA$.
    \item $\lambda_1(\mA)$: the $1^\text{st}$, the largest eigenvalue of matrix $\mA$.
    \item $y^*_{k+1}$: the label of learning the function $\vwt$. $y^*_{k+1}=\langle \vx_{k+1}, \vwt \rangle$.
\end{itemize}

\paragraph{Notations introduced in 
Sec.~\ref{sec:theory}:}
\begin{itemize}[topsep=0.1em, partopsep=0em, leftmargin=*]
\setlength\itemsep{0.1em} 
    \item The L2 loss of ICL learning to learn the function $\vwt$. $\mathcal{L}_k^*=(\F(\Skx)-y_{k+1}^*)^2=(\F(\Skx)-\langle \vx_{k+1},\vwt \rangle)^2$.
\end{itemize}

\paragraph{Notations introduced in Sec.~\ref{subsec:TaskRetrieval}:}
\begin{itemize}[topsep=0.1em, partopsep=0em, leftmargin=*]
\setlength\itemsep{0.1em} 
    \item $d^2_\vm$: $\forall \beta\neq\alpha, \|\vm_\beta-\vmt\|^2 - \|\vm_\alpha-\vmt\|^2 \geq d^2_\vm$, the $\vm$-margin of any other $\vm_\beta$ over $\vm_\alpha$.
    \item $d^2_\vw$: $\forall \beta\neq\alpha, \|\vw_\beta-\vwt\|^2 - \|\vw_\alpha-\vwt\|^2 \geq d^2_\vw$, the $\vw$-margin of any other $\vw_\beta$ over $\vw_\alpha$.
    \item $u^2_\vw$: $\forall \beta\neq\alpha, \nxd^2 \|\vw_\beta-\vwt\|^2 - (1+\nxd^2)\|\vw_\alpha-\vwt\|^2 \geq \nxd^2 u^2_\vw$, the weighted $\vw$-margin of any other $\vw_\beta$ over $\vw_\alpha$.
    \item $y^\alpha_{k+1}$: the label of retrieving the function $\vw_\alpha$. $y^\alpha_{k+1}=\langle \vx_{k+1}, \vw_\alpha \rangle$.
    \item The L2 loss of ICL learning to retrieve the function $\vw_\alpha$ of the pretraining prior center $\alpha$. $\mathcal{L}_k^\alpha=(\F(\Skx)-y_{k+1}^\alpha)^2=(\F(\Skx)-\langle \vx_{k+1},\vw_\alpha \rangle)^2$.
\end{itemize}



\begin{figure}[t]
    \centering
    \includegraphics[width = 0.9\textwidth]{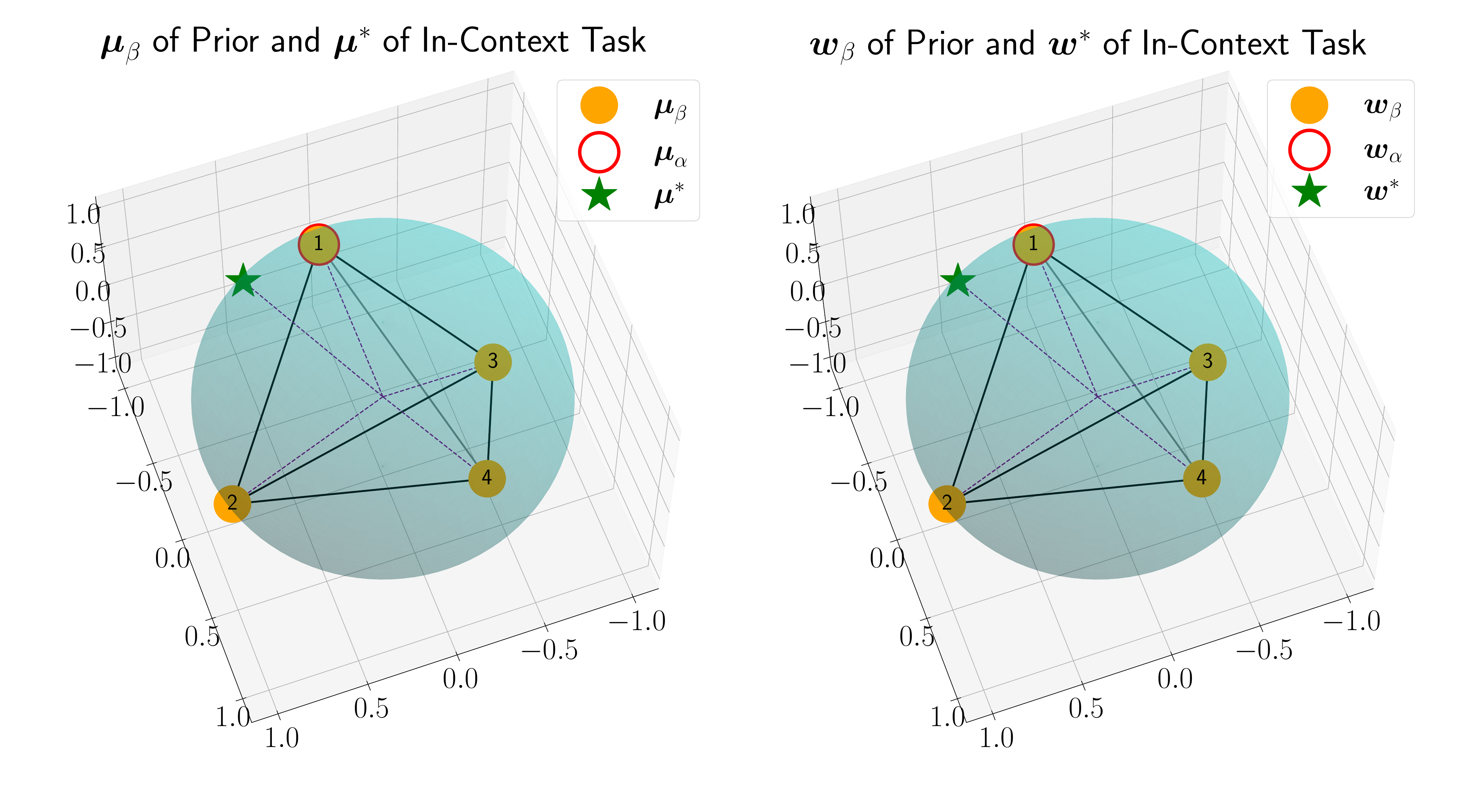}
    \caption{\textbf{Visualization of the tetrahedron setting.} The figure shows the pretraining prior centers and the in-context task. 
    For $\beta\in\{1,2,3,4\},(\vm_\beta,\vw_\beta)$ is a mixture component center in the prior.
    $(\vm_\alpha,\vw_\alpha)$ for $\alpha=1$ (numbers are noted in the center of circles) is the center of the target task for ICL with biased labels, while $(\vmt,\vwt)$ is the in-context task. The dotted purple lines highlight the distance of 1 from the origin $(0,0,0)$ to any point denoted by $\vm$ or $\vw$.}
    \label{fig:3d:4centers}
\end{figure}
\section{Prior Examples}
\label{sec:example}
This section outlines our configurations of prior settings in numerical computations and preliminary Transformer experiments, focusing on the geometrical arrangement of the centers in the priors. Specifically, we detail the configurations where the centers form shapes of 3-dimensional regular polyhedra in Sec.~\ref{sec:3d}, extend to configurations in $d$-dimensional spaces in Sec.~\ref{sec:dd}, and discuss a unique setup related to the early ascent phenomenon in Sec.~\ref{subsec:early}. 

\subsection{Regular Polyhedrons}
\label{sec:3d}
Taking into account the centers of the mixture components from the pretraining prior, which manifest as distinct points forming the vertices of various shapes, we examine  3-dimensional regular polyhedrons.
These include tetrahedron (4 vertices/centers), octahedron (6 vertices/centers), hexahedron (8 vertices/centers), icosahedron (12 vertices/centers), and dodecahedron (20 vertices/centers), listed with increasing density of the centers on a sphere.

The configuration of a regular polyhedron with $M$ centers is established in accordance with the parameters outlined in Assumption~\ref{asu:assumption}, as detailed below:
\begin{itemize}[topsep=0.1em, partopsep=0em, leftmargin=*]
\setlength\itemsep{0.1em} 
    \item Dimension $d=3$, the number of mixture components equals to $M$;
    \item The centers of mixture components form a regular polyhedron with $M$ vertices;
    \item All components' mixture weights are the same, $\pi_m = 1/M$, and $\vm_m=\vw_m$, for all $m \in [M]$;
    \item For noises of $\vx$ and $y$, we have $\nx = \ny = 1$, and $\nxd=1$;
    \item For noises of $\vm$ and $\vw$, we have $\sm=\sw=0.25$ if not specified;
    \item For the in-context task, $\vmt = \frac{2\vm_1+\vm_2}{\|2\vm_1+\vm_2\|}$ and $\vwt = \frac{2\vw_1+\vw_2}{\|2\vw_1+\vw_2\|}$ if not specified, where $\vm_2$ is one of the the closest centers to $\vm_1$.
\end{itemize}

We mainly use the \textbf{tetrahedron} setting in the paper. Therefore, we further visualize the setting and note down the parameters.
The 3D visualization of mixture component centers in the prior and the in-context task are shown in Fig.~\ref{fig:3d:4centers}.
The parameters are noted as follows:
\begin{itemize}[topsep=0.1em, partopsep=0em, leftmargin=*]
\setlength\itemsep{0.1em} 
    \item Dimension $d=3$, number of mixture components $M=4$;
    \item The centers of topics form a tetrahedron as shown in Fig.~\ref{fig:3d:4centers}. $\vm_1 = \vw_1 = [0,0,-1]^\top$, $\vm_2 = \vw_2 = [\sqrt{\frac{8}{9}}, 0, \frac{1}{3}]^\top$, $\vm_3 = \vw_3 = [-\sqrt{\frac{2}{9}}, +\sqrt{\frac{2}{3}}, \frac{1}{3}]^\top$, and $\vm_4 = \vw_4 = [-\sqrt{\frac{2}{9}}, -\sqrt{\frac{2}{3}}, \frac{1}{3}]^\top$;
    \item All components' mixture weights are the same, $\pi_m = 1/4$, and $\vm_m=\vw_m$, for all $m \in \{1,2,3,4\}$;
    \item For noise of $\vx$ and $y$, we have $\nx = \ny = 1$, and $\nxd=1$;
    \item For noises of $\vm$ and $\vw$, we have $\sm=\sw=0.25$ if not specified;
    \item For in-context task, we have $\vmt = \frac{2\vm_1+\vm_2+0.2\vm_3}{\|2\vm_1+\vm_2+0.2\vm_3\|}$ and $\vwt = \frac{2\vw_1+\vw_2+0.2\vw_3}{\|2\vw_1+\vw_2+0.2\vw_3\|}$. We slightly shift the in-context task $(\vmt,\vwt)$ towards $(\vm_3,\vw_3)$ for visualization purposes, to make $m=3$ and $m=4$ produce slightly different curves.
\end{itemize}

\subsection{$d$-Dimensional Examples}
\label{sec:dd}
We consider $d$-dimensional examples with $d$ centers for $d\in\{2,4,8,16,32\}$. 
A $d$-dimensional example with $d$ vertices is parametered as follows:
\begin{itemize}[topsep=0.1em, partopsep=0em, leftmargin=*]
\setlength\itemsep{0.1em} 
    \item Dimension equals to $d$, number of mixture component $M=d$;
    \item For all $m \in [M]$, $\vm_m=\bm{e}_m$ and
$\vm_{m,i} = 
\begin{cases} 
1 & \text{if } i = m \\
0 & \text{if } i \neq m
\end{cases}
$, \ie, $\vm_m$ is the $m^\text{th}$ vector in the standard basis of $\R^m$, characterized by having all elements equal to $0$ except for the $m^\text{th}$ element, which is $1$.
    \item All components' mixture weights are the same, $\pi_m = 1/d$, and $\vm_m=\vw_m$, for all $m \in [M]$;
    \item For noise of $\vx$ and $y$, we have $\nx = \ny = 1$, and $\nxd=1$;
    \item For noises of $\vm$ and $\vw$, we have $\sm=\sw=0.25$;
    \item For the in-context task, we have $\vmt = \frac{2\vm_1+\vm_2}{\|2\vm_1+\vm_2\|}$ and $\vwt = \frac{2\vw_1+\vw_2}{\|2\vw_1+\vw_2\|}$.
\end{itemize}

\begin{table}[t!]
\centering
\caption{\textbf{Prior settings for early ascent.} The pretraining task prior comprises two components for one dimension and three for two or more dimensions.
ICL aims to predict following the in-context function $\vwt$, equivalent to prior center $2$'s function $\vw_2$ ($\vwt=\vw_2$).
The in-context task is characterized by having a closer $\vx$ distribution to the task of prior center $1$ but a closer $\vx\rightarrow  y$ mapping to the prior center $2$.
The parameters for all cases are set to $\sm=\sw=0.05$, $\nx=\nxd=1$, and $\ny=2$. Refer to Fig.~\ref{fig:traj_} for visualization of the prior centers under dimension $d\in\{1,2,3\}$.}
\resizebox{\columnwidth}{!}{%
\begin{tabular}{lccll}
\toprule
Case                               & \begin{tabular}[c]{@{}c@{}}Component\\/Task\end{tabular}    & \begin{tabular}[c]{@{}c@{}}Mixture\\Weight\end{tabular} & \multicolumn{1}{c}{$\vm$}        & \multicolumn{1}{c}{$\vw$}    \\ \midrule
\multirow{4}{*}{$d=1$}               & Component 1                                                  & \nicefrac{1}{2}                                                      & $\vm_1=[+1]$                     & $\vw_1=[-1]$    \\
                                   & Component 2                                                  & \nicefrac{1}{2}                                                      & $\vm_2=[-1]$                     & $\vw_2=[+1]$    \\
                                   & Component 3                                                  & /                                                        & /                                 & /               \\
                                   & In-context Task                                              & /                                                        & $\vm^*=[+1]$                     & $\vw^*=[+1]$    \\ \hdashline
\multirow{4}{*}{$d=2$}               & Component 1                                                  & \nicefrac{1}{3}                                                      & $\vm_1=[+1,+1]$                  & $\vw_1=[-1,-1]$ \\
                                   & Component 2                                                  & \nicefrac{1}{3}                                                      & $\vm_2=[-1,-1]$                  & $\vw_2=[+1,+1]$ \\
                                   & Component 3                                                  & \nicefrac{1}{3}                                                      & $\vm_3=[+1,-1]$                  & $\vw_3=[-1,+1]$ \\
                                   & In-context Task                                              & /                                                        & $\vm^*=[+1,+1]$                  & $\vw^*=[+1,+1]$ \\ \hdashline
\multirow{4}{*}{$d\geq2$} & Component 1                                                  & \nicefrac{1}{3}                                                     & $\vm_1=[+1]+[+1]\times(d-1)$     & $\vw_1=[-1]+[-1]\times(d-1)$ \\
                                   & Component 2                                                  & \nicefrac{1}{3}                                                      & $\vm_2=[-1]+[-1]\times(d-1)$     & $\vw_2=[+1]+[+1]\times(d-1)$ \\
                                   & Component 3                                                  & \nicefrac{1}{3}                                                      & $\vm_3=[+1]+[-1]\times(d-1)$     & $\vw_3=[-1]+[+1]\times(d-1)$ \\
                                   & In-context Task                                              & /                                                        & $\vm^*=[+1]\times d$             & $\vw^*=[+1]\times d$         \\ \bottomrule
\end{tabular}}
\label{table:earlysetting}
\end{table}
\begin{figure*}[ht!]
    \centering
    \subfigure[First row: expected L2 loss and upper bound with increasing in-context samples $k$ under varied dimensions $d$.
    Second row: expected mixture weights with increasing in-context samples $k$ under varied dimensions $d$. We further examine the early ascent phenomenon under linear regression with varied levels of label noises in Appendix~\ref{app:expnoisy}, and under non-linear regression and discrete token prediction in Appendix~\ref{app:earlyascent:nonlinear&discrete}.]{
        \includegraphics[width=0.85\textwidth]{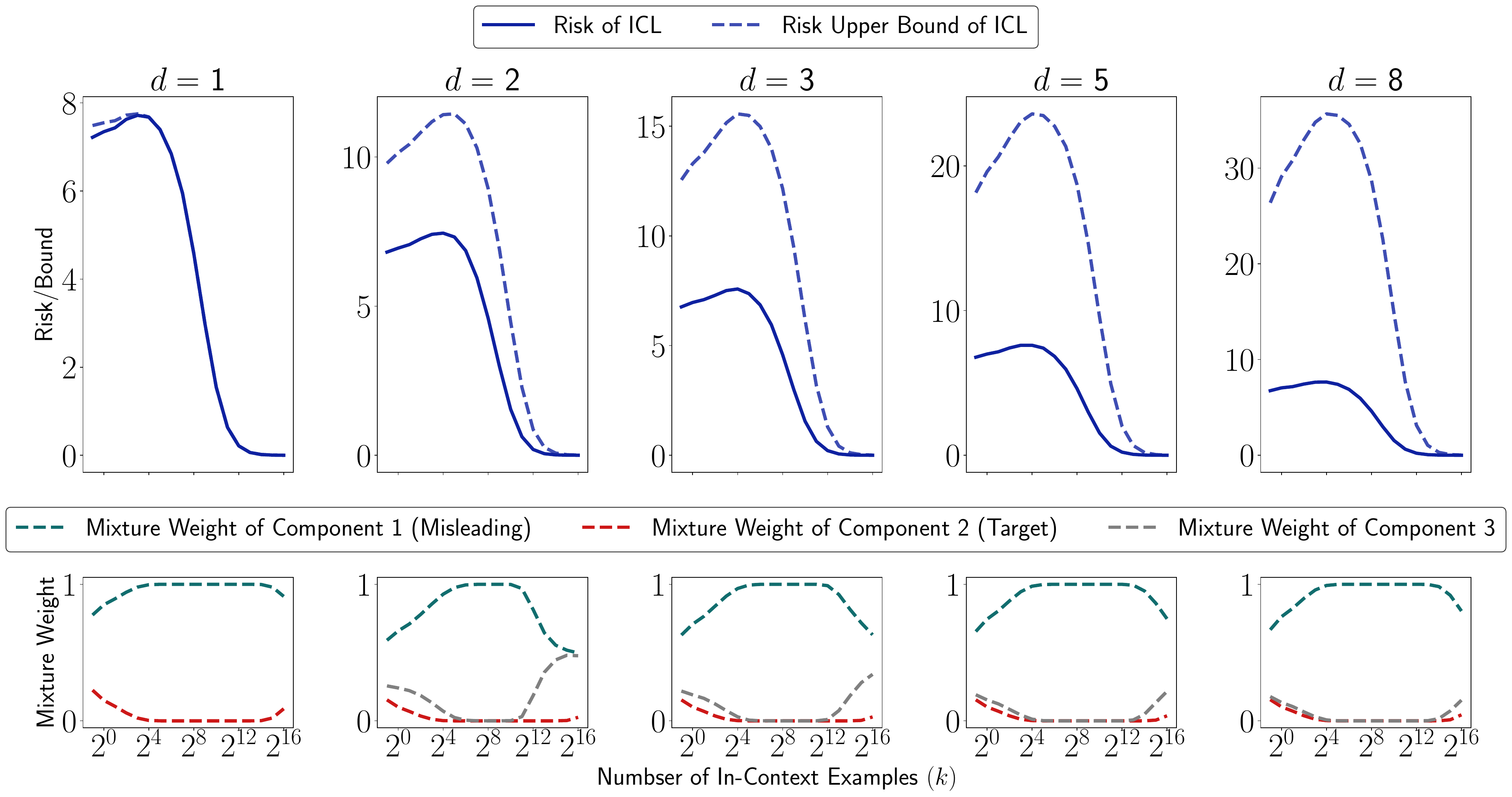}
        \label{fig:UB_}}
    \hspace{1em}
    \subfigure[The trajectory of the expectation of $\tilde{\vw}$ with increasing $k$ under $d$ equal to 1, 2 and 3.]{
        \includegraphics[width=0.80\textwidth]{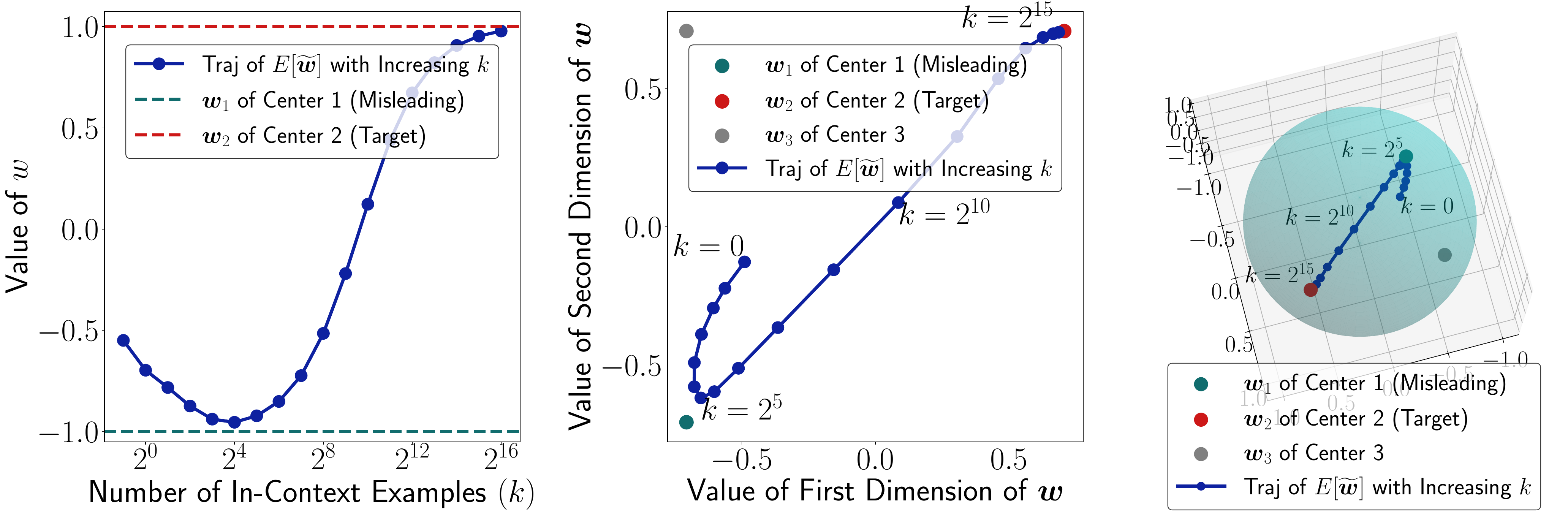}
        \label{fig:traj_}}
    \caption{\textbf{The early ascent phenomenon.} 
    Fig.~\ref{fig:UB_} displays the trends of expected losses, upper bounds, and mixture weights, while Fig.~\ref{fig:traj_} presents the trend of the expectation of $\tvw$.
    We can see that the task retrieval mode is dominant up to $k=32$, and component 1's mixture weight increases ($\E[\tilde{\vw}]$ approaches $\vw_1$).
    Since this misleading component 1 is far from the target component 2, the risk starts increasing. 
    At larger $k$ values, the risk starts decreasing ($\E[\tilde{\vw}]$ approaches $\vw_2$) via task learning. 
    }
    \label{fig:Ushape_}
\end{figure*}
\subsection{Early Ascent Examples}
\label{subsec:early}
Table~\ref{table:earlysetting} outlines the prior configuration used to produce the early ascent phenomenon, where the in-context task is designed with a distribution of $\vx$ close to a misleading task. The full results are shown in Fig.~\ref{fig:Ushape_}.
\section{Coarse Upper Bound for ICL Risk}
\label{subsec:coarse}
The following theorem shows a coarse upper bound of the ICL risk parallel to Theorem~\ref{the:finegrainedlearning}:
\begin{theorem}[Coarse Upper Bound for ICL Risk]
\label{the:generallearning}
Consider a \sp attaining the optimal pretraining risk.
As $k\rightarrow\infty$, the ICL risk is upper bounded by:
\begin{align}
    \E_{\Skx}[\mathcal{L}_k^*] 
    <&
    \frac{4(1+d\nxd^2)}{\nxd^4 \blue{\dw}^2 \red{k}^2} + O(\red{k}^{\delta-\frac{5}{2}}),
\end{align}
where $\mathcal{L}_k^*=(\F(\Skx)-y_{k+1}^*)^2=(\F(\Skx)-\langle \vx_{k+1},\vwt \rangle)^2$ and $\delta$ is an arbitrarily small positive constant.
See Appendix~\ref{proof:learning} for proof details.
The upper bound decreases as the square of the inverse of $k$. 
Notice there is no noise for $y$ labels of in-context examples under our setting, which leads to a faster decay rate than standard $1/k$ for ridge regression~\citep{Alexander23}.
\end{theorem}
The notations $\blue{\dw}$ and $\red{k}$ are colored for easier observation.

\begin{figure}[th!]
    \centering
    \includegraphics[width=1.0\textwidth]{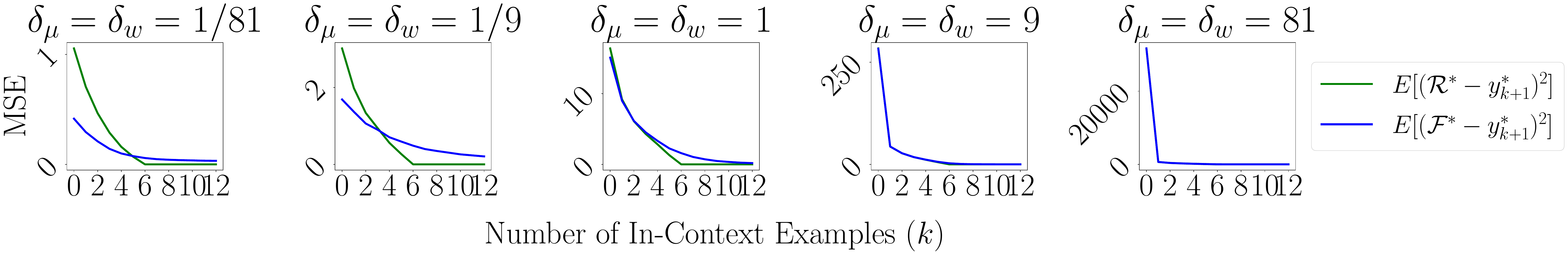}
    \caption{\textbf{In-context learning vs ridge regression.} $\mathcal{R}^*$ indicates the prediction by ridge regression, $\mathcal{F}^*$ indicates the prediction by ICL with a \bosp, and $y^{*}_{k+1}=\langle \vx_{k+1},\vwt \rangle$.
    Let the $k$ samples draw from a task $(\vmt,\vwt)$, which is drawn from the pretraining prior distribution.
    The dimension $d$ of $\vx$ equals 6.
    We observe that ICL performs better than ridge regression when $k$ is small, and ridge regression performs better than ICL when $k\geq d$.
    Especially, when the task prior distribution has high task variance (big $\dm$ and $\dw$ values), ICL and ridge regression have very similar performance.}
    \label{fig:LASSOcompare}
\end{figure}
We further compare the risk $\E_{\Skx}[\mathcal{L}_k^*]$ and the risk under ridge regression with L2 regularization parameter equal to $10^{-6}$, where the same $k$ samples without label noises are used as in-context examples for ICL and training samples for ridge regression. 
Fig.~\ref{fig:LASSOcompare} shows the experiment results.
Under certain settings for the task prior $\mathcal{D}_{\vm,\vw}$, when the task prior has low task variances, ICL performs better than ridge regression with a fixed regularization parameter under small $k$.

\section{Transformer Performance in Approximating Bayesian Inference}
\label{sec:transformer}
We examine if a Transformer network pretrained on samples generated from our pretraining data generative model matches the performance of Bayesian inference.
We consider three factors of the task prior in our experiment: \emph{prior task noises}, \emph{number of components}, and \emph{feature dimension}.
For scalar $y$, we transform it to a $d$-dimensional vector $[y, 0, \ldots, 0]$.
Thus, $\Skx$ forms a $(2k+1)\times d$ matrix, comprising $\vx_{k+1}$ and $k$ pairs of $(\vx_i,y_i)$.

\paragraph{Experiment Setting.} We conduct experiments based on the module GPT2Model from the package Transformers supported by HuggingFace\footnote{https://huggingface.co/}. 
We use a 10-layer, 8-head Transformer decoder with 1024-dimensional feedforward layers, and the input dimension is set to $d$, equal to the dimension of $\vx$.
We train the model over three epochs, each consisting of 10,000 batches, with every batch containing 256 samples.
We use AdamW~\citep{loshchilov2017decoupled} as the optimizer with weight decay as $0.00001$ and set the learning rate to $0.00001$.

\begin{figure}[th!]
    \centering
    \includegraphics[width=.90\textwidth]{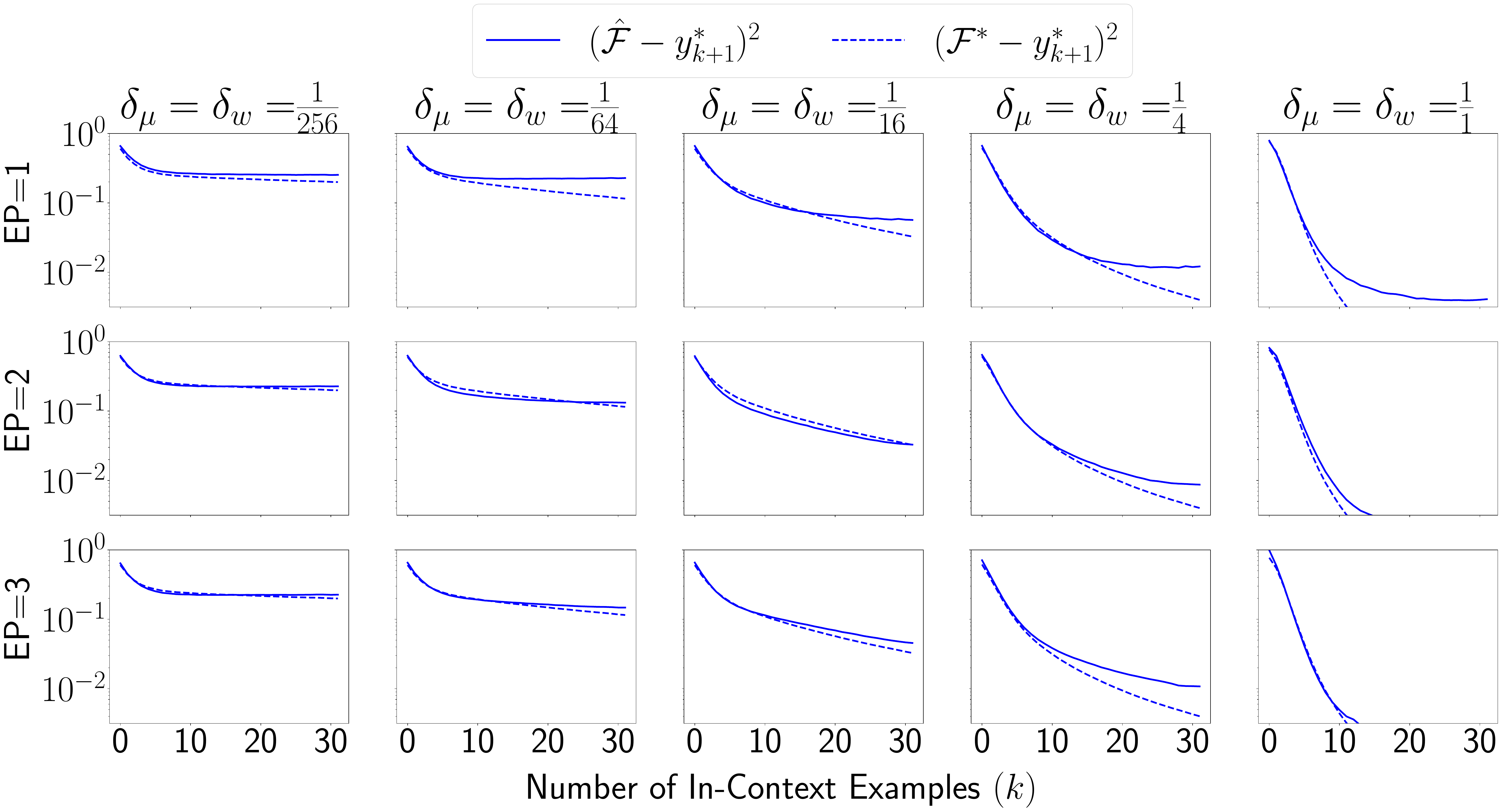}
    \caption{
    \textbf{Prior task noises.} The figure shows the experiment results under varied noise levels.
    $\dm$ and $\dw$ indicate the noise levels of the pretraining task prior.
    $\sF$ indicates the prediction of Bayesian inference while $\hat{\F}$ indicates the prediction of the trained Transformer network.
    The results show that the trained Transformer network's performance can approach the performance of Bayesian inference.
    }
    \label{fig:regular4}
\end{figure}
\begin{figure}[th!]
    \centering
    \includegraphics[width=0.90\textwidth]{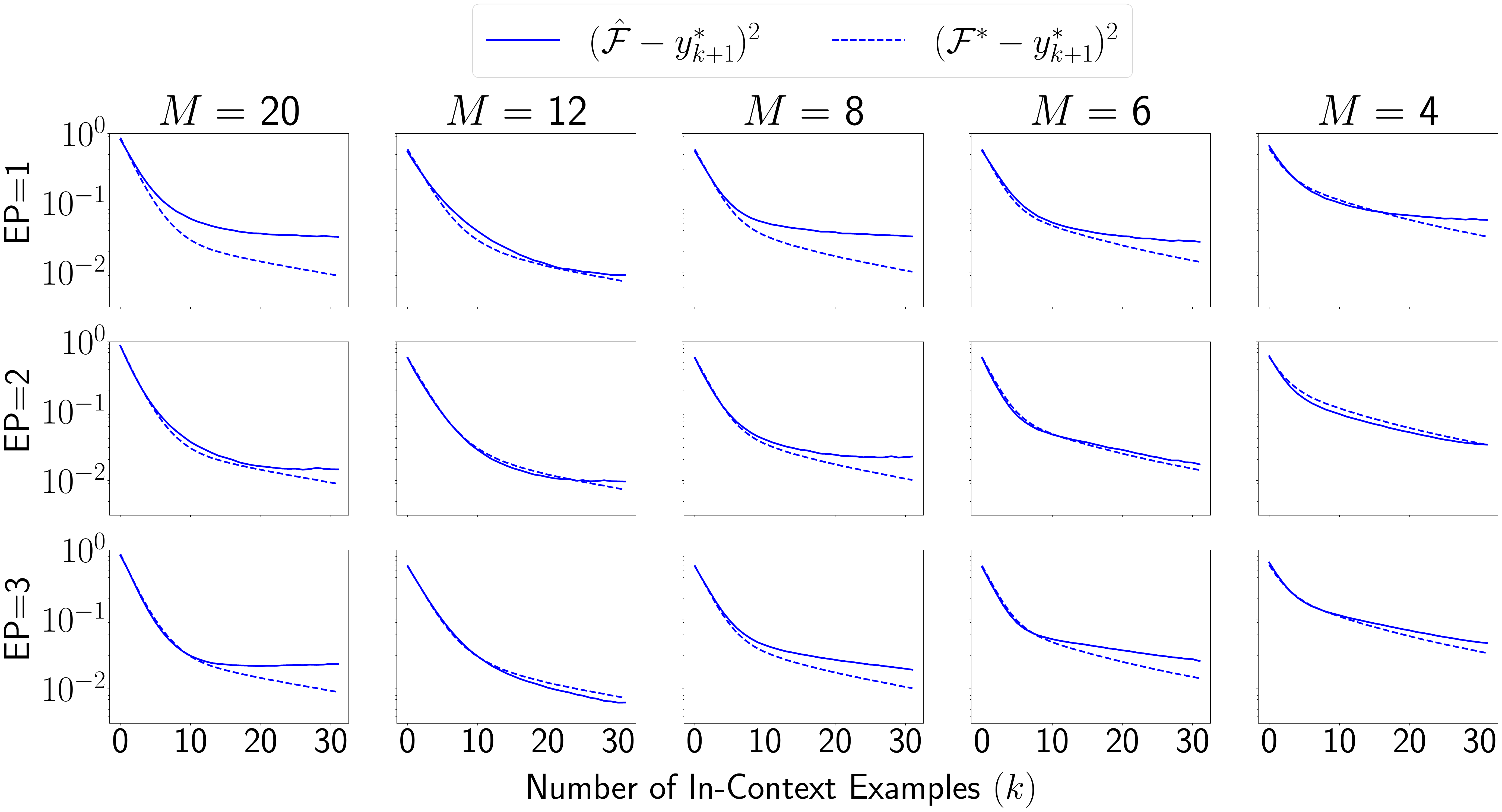}
    \caption{
    \textbf{Number of components.} The figure shows the experiment results under varied component densities.
    $M$ indicates the number of mixture components corresponding to different 3D regular polyhedrons described in Appendix~\ref{sec:3d}, and $\dm=\dw=\frac{1}{16}$.
    $\sF$ indicates the prediction of Bayesian inference while $\hat{\F}$ indicates the prediction of the trained Transformer network.
    The higher the component density is, the harder it is for the Transformer network to approach Bayesian inference.
    }
    \label{fig:regularM}
\end{figure}
\begin{figure}[th!]
    \centering
    \includegraphics[width=0.90\textwidth]{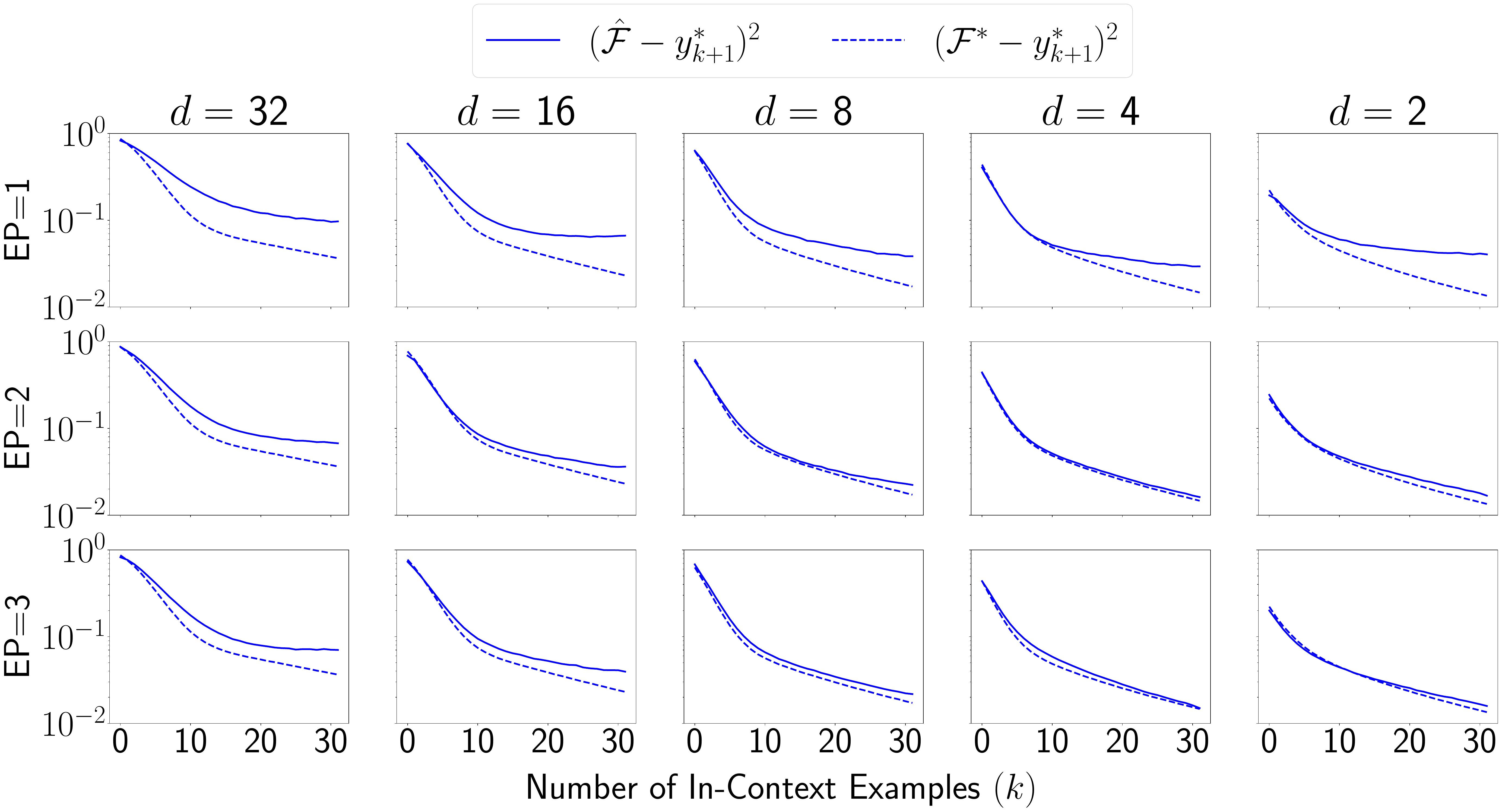}
    \caption{
    \textbf{Feature dimension.} The figure shows the experiment results under varied dimensions.
    $d$ indicates the dimension and the number of mixture components (see Appendix~\ref{sec:dd} for setting details), and $\dm=\dw=\frac{1}{16}$.
    $\sF$ indicates the prediction of Bayesian inference while $\hat{\F}$ indicates the prediction of the trained Transformer network.
    The higher the feature dimension is, the harder it is for the Transformer network to approach Bayesian inference.}
    \label{fig:D_d}
\end{figure}

\paragraph{Experiment Results.} Fig.~\ref{fig:regular4},~\ref{fig:regularM}, and~\ref{fig:D_d} show the experimental results, where $\hat{\F}$ denotes the prediction of the Transformer network, $\sF$ denotes the prediction of Bayesian inference, and $y_{k+1}^*=\langle \vx_{k+1}, \vwt \rangle$ is the label of learning the in-context function.
In Fig.~\ref{fig:regular4}, we consider the tetrahedron setting (see Apendix~\ref{sec:3d} for setting details) under varied task noises ($\dm=\dw \in \{1/256, 1/64, 1/16, 1/4, 1\}$).
In Fig.~\ref{fig:regularM}, we consider settings of regular shapes (see Appendix~\ref{sec:3d} for setting details) with different numbers of vertices/components ($M\in\{4,6,8,12,20\}$).
In Fig.~\ref{fig:D_d}, we consider settings with varied dimensions (see Appendix~\ref{sec:dd} for setting details, $d\in\{2,4,8,16,32\}$).
We observe that the trained Transformer network can approximate the Bayes-optimal predictor under varied settings, and the larger the number of dimensions and the number of mixture components, the harder it is for the Transformer network to approximate Bayesian prediction.
\section{Additional Information for Bounded Efficacy in GPT-4}
\begin{table}[th!]
\centering
\caption{Experiment setting to reveal the bounded efficacy phenomenon of biased-label ICL in GPT-4.}
\begin{tabular}{cl}
\toprule
Setting                & \multicolumn{1}{c}{Desciption}                                                                                                                                                                                                                                                                           \\ \midrule
LLM                    & \multicolumn{1}{c}{GPT-4}                                                                                                                                                                                                                                                                                \\ \hline
System Message         & \begin{tabular}[c]{@{}l@{}}You are a mathematician. Consider the following math problem and\\ follow the exact instruction.\end{tabular}                                                                                                                                                                 \\ \hline
Prompt                 & \begin{tabular}[c]{@{}l@{}}You are given examples. Each example has two integers as input and\\ one integer as output. Please provide an answer for the last problems \\ in the math exercise:\\ $\GGG{a_1}$(?)$\GGG{b_1}$=$\GGG{c_1}$\\ ...\\ $\GGG{a_k}$(?)$\GGG{b_k}$=$\GGG{c_2}$\\ $\GGG{a_{k+1}}$(?)$\GGG{b_{k+1}}$=\\ Provide your answer directly.\end{tabular} \\ \hline
In-Context Task       & \multicolumn{1}{c}{$\GGG{a_i}$ and $\GGG{b_i}$ are uniformly sampled from $[10,99]$, and $\GGG{c_i}=\GGG{a_i}+\GGG{b_i}+1$.}                                                                                                                                                                                                                                       \\ \hline
\begin{tabular}[c]{@{}c@{}} Goal of Learning the\\``biased $+$'' \\Task with True Labels \end{tabular} & \begin{tabular}[c]{@{}l@{}}Aiming to learn the ``biased $+$'' task, a(?)b=(a+b+1), with \\ in-context examples following the same ``biased $+$'' task, \\ a(?)b=(a+b+1).\end{tabular}                                                                                     \\ \hline
\begin{tabular}[c]{@{}c@{}} Goal of Retrieving the\\``addition ($+$)'' \\Task with Biased Labels \end{tabular} & \begin{tabular}[c]{@{}l@{}}Aiming to retrieve the ``addition ($+$)'' task, a(?)b=(a+b). However, the \\ in-context examples are provided with a slightly different task\\ ``biased $+$'', a(?)b=(a+b+1).\end{tabular}   \\ \bottomrule                   
\end{tabular}
\label{table:Usetting}
\end{table}
\subsection{Experimental Setting}
\label{app:GPT4U}
Table~\ref{table:Usetting} introduces the experiment setting of GPT-4, including the system message, the prompt, the in-context task, the ``biased $+$'' task, and the ``addition ($+$)'' task.
Designating the ``biased $+$'' task as the in-context task, \ie, $\GGG{c_i}=\GGG{a_i}+\GGG{b_i}+1$, we measure the performances on two goals, including learning the ``biased $+$'' task and retrieving the ``addition ($+$)'' task. 


\begin{table}[t]
\centering
\caption{
Zero in-context example, $k=0$. 
Prediction is colored \red{red} if it is correct for task retrieval ($a(?)b=(a+b)$), and colored \blue{blue} if it is correct for task learning ($a(?)b=(a+b+1)$).
``...'' denotes the hidden part of the prompt.
Please refer to Table~\ref{table:Usetting} for the whole prompt.}
\begin{tabular}{cllll}
\toprule
Prompt  & \begin{tabular}[c]{@{}l@{}}...\\ 51(?)36=\\ ...\end{tabular}                                                                                                                                                                                                               & \begin{tabular}[c]{@{}l@{}}...\\ 27(?)15=\\ ...\end{tabular}                                                                                                                                  & \begin{tabular}[c]{@{}l@{}}...\\ 76(?)82=\\ ...\end{tabular}                                                                                                                                                                                                                         & \begin{tabular}[c]{@{}l@{}}...\\ 55(?)15=\\ ...\end{tabular} \\ \midrule
Results & \begin{tabular}[c]{@{}l@{}}Without knowing the \\ operation or rule that \\ connects the two \\ input integers to \\ the output integer in \\ the examples, it's \\ impossible to provide \\ a correct answer. \\ Please provide the \\ examples or the rule.\end{tabular} & \begin{tabular}[c]{@{}l@{}}Sorry, but your \\ questionis not\\ clear. Could \\ you please \\ provide more \\ information \\ about the \\ operation \\ between the\\ two numbers?\end{tabular} & \begin{tabular}[c]{@{}l@{}}Your question seems to \\ be missing some \\ information. Could you \\ please provide the \\ examples you mentioned? \\ They are necessary to \\ understand the relationship \\ between the two input \\ integers and the output \\ integer.\end{tabular} & \red{70}     \\ \bottomrule                                                     
\end{tabular}
\label{table:k0}
\end{table}

\begin{table}[t]
\centering
\caption{Two in-context examples, $k=2$. Prediction is colored \red{red} if it is correct for task retrieval ($a(?)b=(a+b)$), and colored \blue{blue} if it is correct for task learning ($a(?)b=(a+b+1)$). ``...'' denotes the hidden part of the prompt.
Please refer to Table~\ref{table:Usetting} for the whole prompt.}
\begin{tabular}{cllll} \toprule
Prompt  & \begin{tabular}[c]{@{}l@{}}...\\ 73(?)80=154\\ 59(?)22=82\\ 54(?)97=\\ ...\end{tabular} & \begin{tabular}[c]{@{}l@{}}...\\ 48(?)73=122\\ 78(?)80=159\\ 21(?)33=\\ ...\end{tabular} & \begin{tabular}[c]{@{}l@{}}...\\ 21(?)28=50\\ 69(?)29=99\\ 47(?)10=\\ ...\end{tabular} & \begin{tabular}[c]{@{}l@{}}...\\ 94(?)43=138\\ 98(?)70=169\\ 96(?)41=\\ ...\end{tabular} \\ \midrule
Results & \red{151}                                                                                     & \red{54}                                                                                       & \red{57}                                                                                        & 187       \\ \bottomrule                                                                                 
\end{tabular}
\label{table:k2}
\end{table}

\begin{table}[th!]
\centering
\caption{Eight in-context examples, $k=8$. Prediction is colored \red{red} if it is correct for task retrieval ($a(?)b=(a+b)$), and colored \blue{blue} if it is correct for task learning ($a(?)b=(a+b+1)$). ``...'' denotes the hidden part of the prompt.
Please refer to Table~\ref{table:Usetting} for the whole prompt.}

\begin{tabular}{cllll}
\toprule
Prompt  & \begin{tabular}[c]{@{}l@{}}...\\ 37(?)70=108\\ 41(?)18=60\\ 19(?)12=32\\ 82(?)67=150\\ 42(?)13=56\\ 26(?)41=68\\ 80(?)39=120\\ 58(?)23=82\\ 40(?)90=\\ ...\end{tabular} & \begin{tabular}[c]{@{}l@{}}...\\ 60(?)76=137\\ 69(?)26=96\\ 72(?)85=158\\ 39(?)10=50\\ 50(?)47=98\\ 19(?)63=83\\ 45(?)95=141\\ 69(?)41=111\\ 81(?)36=\\ ...\end{tabular} & \begin{tabular}[c]{@{}l@{}}...\\ 66(?)40=107\\ 46(?)81=128\\ 63(?)31=95\\ 41(?)24=66\\ 70(?)43=114\\ 89(?)84=174\\ 76(?)82=159\\ 46(?)28=75\\ 49(?)46=\\ ...\end{tabular} & \begin{tabular}[c]{@{}l@{}}...\\ 68(?)88=157\\ 34(?)18=53\\ 70(?)70=141\\ 13(?)35=49\\ 52(?)50=103\\ 72(?)32=105\\ 98(?)82=181\\ 55(?)51=107\\ 50(?)31=\\ ...\end{tabular} \\ \midrule
Results & \red{130}                                                                                                                                                                     & \blue{118}                                                                                                                                                                      & \blue{96}                                                                                                                                                                        & \blue{82}         \\ \bottomrule                                                                                                                                                               
\end{tabular}
\label{table:k8}
\end{table}
\subsection{Additional Results}
\label{app:additionGPT4}
This section collects four pairs of prompts and predictions for $k=0,2,8$ in Tables~\ref{table:k0}, ~\ref{table:k2}, and~\ref{table:k8}.
The results show that ICL with biased labels will initially retrieve a commonsense pretraining task due to task retrieval, and finally learn the in-context task because of task learning.
\begin{figure}[th!]
    \centering
    \includegraphics[width = 0.75\textwidth]{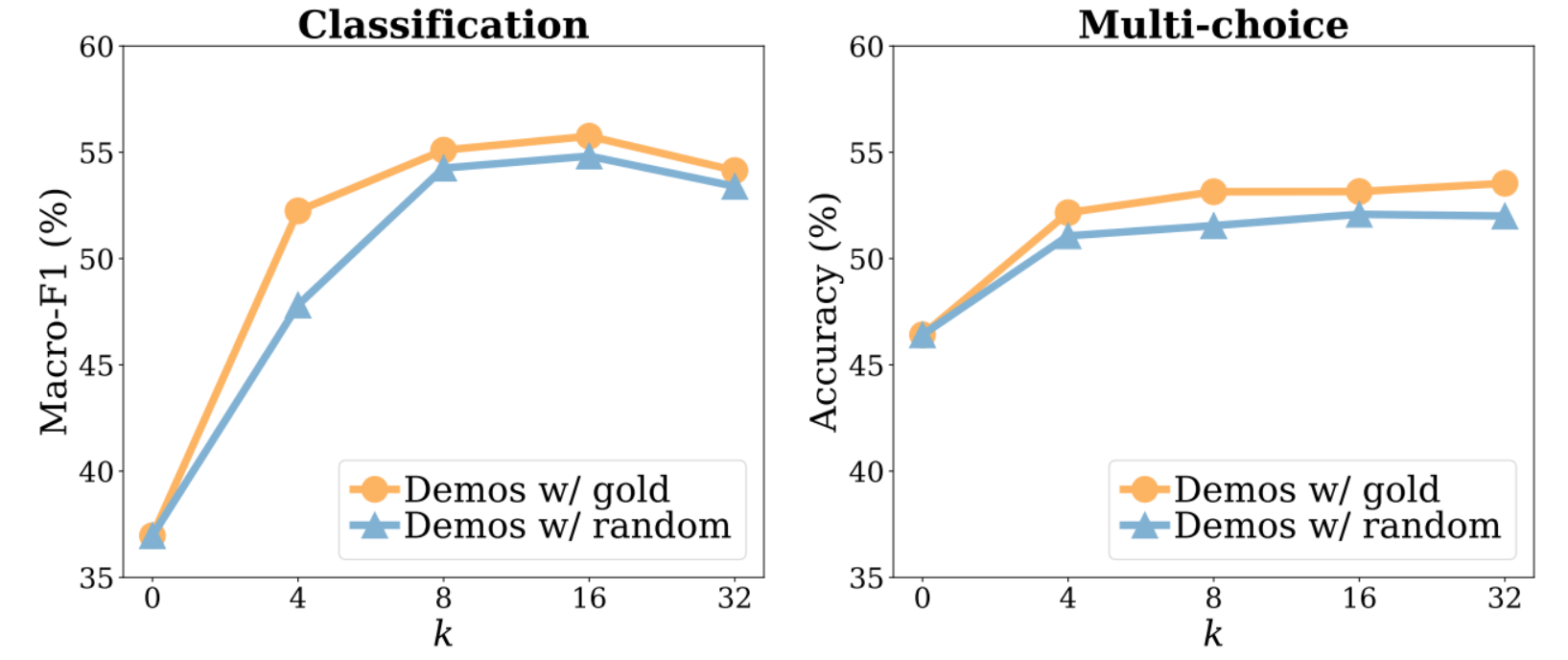}
    \caption{Ablations on varying numbers of examples in the demonstrations $(k)$. Models that are the best under 13B in each task category (Channel MetaICL and Direct GPT-J, respectively) are used.}
    \label{fig:copy:zeroicl}
\end{figure}
\begin{figure}[th!]
    \centering
    \includegraphics[width=0.8\textwidth]{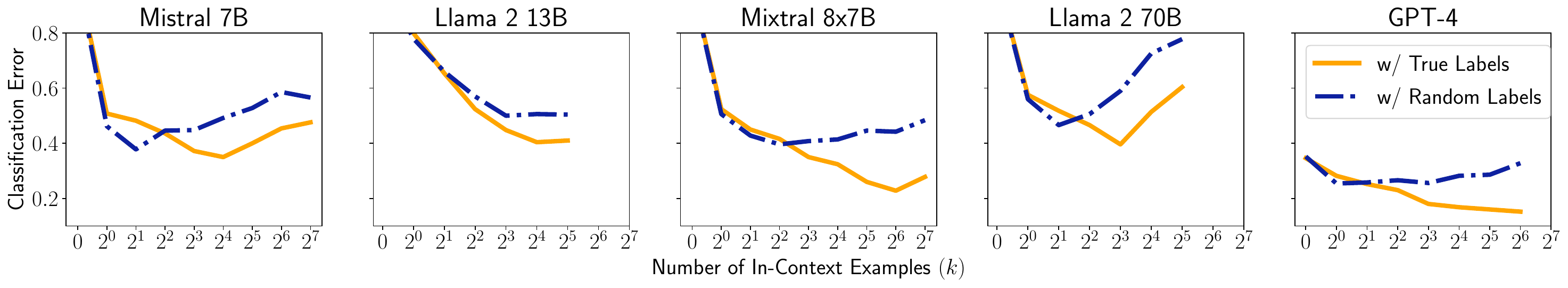}
    \caption{
    As $k$ increases, the classification error curve of ICL with random labels exhibits the bounded efficacy phenomenon.
    The curve with true labels further confirms that this phenomenon is not due to models tending to perform worse on long sequences.
    }
    \label{fig:ZeroICLall}
\end{figure}
\section{Bounded Efficacy in Zero-shot ICL}
\label{app:exp:zeroicl}
This section introduces the experiment setting of Fig.~\ref{fig:ZeroICL}.
We start by introducing the experiment results in Fig.~\ref{fig:copy:zeroicl} copied and pasted from the work of~\citet{MinLHALHZ22}.
While our theory shows the bounded efficacy phenomenon for ICL with non-informative labels (Lemma~\ref{lemma:zeroicl}), Fig.~\ref{fig:copy:zeroicl} seems to imply a conflict phenomenon.
Thus, we further extend the number of in-context examples in Fig.~\ref{fig:copy:zeroicl} left.
The classification task adopts five datasets including (i) glue-mrpc~\citep{DolanB05}, (ii) glue-rte~\citep{DaganGM05}, (iii) tweet\_eval-hate~\citep{BarbieriCAN20},
(iv) sick~\citep{MarelliMBBBZ14}, and (v) poem-sentiment~\citep{sheng2020investigating}.
We use the GitHub code\footnote{https://github.com/Alrope123/rethinking-demonstrations} released by \citet{MinLHALHZ22} to generate the same data and evaluate LLMs with a larger context length capacity aiming at a larger number of in-context examples.
We selected Mistral 7B (32768), Mixtral 8$\times$7B (32768), Llama2 13B (4096), Llama2 70B (4096), and GPT-4 (8192) for our experiments, with the integers in parentheses indicating the maximum context length for each model. We perform inference on large models with 8 H100 with the package vllm\footnote{https://docs.vllm.ai/en/latest/}.

\section{The Derivation of Posterior}
\label{sec:lemma}
This section provides detailed derivations for Lemma~\ref{lemma:posterior}.
We begin by showing the posterior is potentially still a Gaussian mixture in Sec.~\ref{app:p2p1}.
Then, in Sec.~\ref{app:p2p2}, we show how Eq.~\ref{prior} is proportion to Eq.~\ref{posterior}, which is precisely a Gaussian mixture.
\subsection{Prior to Posterior}
We start by showing the posterior is potentially still a Gaussian mixture.
For fixed $\Skx$:
\label{app:p2p1}
\begin{align}
&
    P(\vm,\vw \vert \Skx) 
\\&\propto
    P(\vm,\vw \vert \Skx)P(\Skx) 
\\&=
    P(\vm,\vw,\Skx)
\\&=
    P(\vm,\vw) P(\Skx \vert \vm,\vw)
\\&=
    \bigg(\sum_{m=1}^M \pi_m P(\vm,\vw \vert T_m)\bigg) P(\Skx \vert \vm,\vw)
\\&=
    \label{prior}
    \sum_{m=1}^M \pi_m P(\vm,\vw \vert T_m) P(\Skx \vert \vm,\vw)
\\&\propto
    \label{posterior}
    \sum_{m=1}^M \tpi_m P(\vm,\vw \vert \tT_m).
\end{align}
We give the derivation from Eq.~\ref{prior} to Eq.~\ref{posterior} in the next section.

\subsection{Closed-form Solution from Eq.~\ref{prior} to Eq.~\ref{posterior}}
\label{app:p2p2}
We analyze each component (indicated by a specific $m$) in Eq.~\ref{prior}.
Given fixed $\Skx$, for all $m\in[M]$ and all $(\vm,\vw)$, we have:
\begin{align}
    &\log(P(\vm,\vw \vert T_m) P(\Skx \vert \vm,\vw))
    \\&= 
    -\frac{\|\vm_m-\vm\|^2}{2\nm^2}
    -\frac{\|\vw_m-\vw\|^2}{2\nw^2}
    -\frac{\sum_{i=1}^{k+1}  \|\vm-\vx_i\|^2 }{2\nx^2}
    -\frac{\sum_{i=1}^k\|\vx_i^\top\vw-y_i\|^2}{2\ny^2}
    \\&~~~~~
    +\log\left(\frac{(2\pi)^{-d/2}}{\nm^d}\right)
    +\log\left(\frac{(2\pi)^{-d/2}}{\nw^d}\right)
    +(k+1)\log\left(\frac{(2\pi)^{-d/2}}{\nx^d}\right)
    +k\log\left(\frac{(2\pi)^{-1/2}}{\ny}\right)
    \\&
    (\text{Let } C_3 = \log\left(\frac{(2\pi)^{-d/2}}{\nm^d}\right)
    +\log\left(\frac{(2\pi)^{-d/2}}{\nw^d}\right)
    +(k+1)\log\left(\frac{(2\pi)^{-d/2}}{\nx^d}\right)
    +k\log\left(\frac{(2\pi)^{-1/2}}{\ny}\right).)
    \\&=
    C_3
    -\frac{\|\vm_m-\vm\|^2}{2\nm^2}
    -\frac{\|\vw_m-\vw\|^2}{2\nw^2}
    -\frac{\sum_{i=1}^{k+1}  \|\vm-\vx_i\|^2 }{2\nx^2}
    -\frac{\sum_{i=1}^k\|\vx_i^\top\vw-y_i\|^2}{2\ny^2}
    \\&=
    C_3
    -(\frac{\|\vm_m-\vm\|^2}{2\nm^2} + \frac{\sum_{i=1}^{k+1}  \|\vm-\vx_i\|^2 }{2\nx^2})
    -(\frac{\|\vw_m-\vw\|^2}{2\nw^2} + \frac{\sum_{i=1}^k\|\vx_i^\top\vw-y_i\|^2}{2\ny^2})
    \\&(\text{Let } \dm = \frac{\nm^2}{\nx^2} \text{ and } \dw  = \frac{\nw ^2}{\ny^2}.)
    \\&=
    C_3
    -\frac{1}{2\nm^2}\left((\|\vm_m\|^2-2\vm_m^\top\vm+\|\vm\|^2)
               +\dm\bigg((k+1)\|\vm\|^2 - 2\vm^\top\sum_{i=1}^{k+1}\vx_i+\sum_{i=1}^{k+1}\|\vx_i\|^2\bigg)\right)
    \\&~~~~~
    -\frac{1}{2\nm^2}\left((\|\vw_m\|^2-2\vw_m^\top\vw+\|\vw\|^2)
               +\dw\bigg(\sum_{i=1}^k\vw^\top\vx_i\vx_i^\top\vw - 2\vw^\top\sum_{i=1}^k\vx_i y_i+\sum_{i=1}^k y_i^2\bigg)\right)
    \\&=
    C_3-\frac{1}{2\nm^2}\left(\|\vm_m\|^2+(1+(k+1)\dm)\|\vm\|^2 - 2\vm\bigg(\vm_m+\dm\sum_{i=1}^{k+1}\vx_i\bigg)+\dm\sum_{i=1}^{k+1}\|\vx_i\|^2\right)
    \\&~~~~~
    -\frac{1}{2\nw^2}\left(\|\vw_m\|^2 +\vw^\top\bigg(\mI+\dw\sum_{i=1}^k\vx_i\vx_i^\top\bigg)\vw - 2\vw\bigg(\vw_m+\dw\sum_{i=1}^k\vx_i y_i\bigg)+\dw\sum_{i=1}^k y_i^2\right)
    \\&
    (\text{Let } C_4 = C_3 - \frac{\dm}{2\nm^2}\sum_{i=1}^{k+1}\|\vx_i\|^2 - \frac{\dw}{2\nw^2}\sum_{i=1}^k y_i^2.)
    \\&=
    C_4-\frac{1}{2\nm^2}\left(\|\vm_m\|^2+(1+(k+1)\dm)\|\vm\|^2 - 2\vm\bigg(\vm_m+\dm\sum_{i=1}^{k+1}\vx_i\bigg)\right)
    \\&~~~~~
    -\frac{1}{2\nw^2}\left(\|\vw_m\|^2 +\vw^\top\bigg(\mI+\dw\sum_{i=1}^k\vx_i\vx_i^\top\bigg)\vw - 2\vw\bigg(\vw_m+\dw\sum_{i=1}^k\vx_i y_i\bigg)\right)
    \\&
    (\text{Let } \mCm = \mI \text{ and } \mCw = \frac{\sum_{i=1}^k\vx_i\vx_i^\top}{k}.)
    \\&=
    C_4
    -\frac{1}{2\nm^2}\left(\|\vm_m\|^2+\|\vm\|^2_{\mI+(k+1)\dm\mCm} - 2\vm^\top\bigg(\vm_m+\dm\sum_{i=1}^{k+1}\vx_i\bigg)\right)
    \\&~~~~~
    -\frac{1}{2\nw^2}\left(\|\vw_m\|^2+\|\vw\|^2_{\mI+k\dw\mCw} - 2\vw^\top\bigg(\vw_m+\dw\sum_{i=1}^k\vx_i y_i\bigg)\right)
    \\&
    (\text{Let } \mm = \sum_{i=1}^{k+1} \vx_i \text{ and } \mw = \frac{\sum_{i=1}^k\vx_i y_i}{k}.)
\\&= 
    C_4
    -\frac{1}{2\nm^2}(\|\vm_m\|^2+\|\vm\|^2_{\mI+(k+1)\dm\mCm} - 2\vm^\top(\vm_m+(k+1)\dm\mm))
    \\&~~~~~
    -\frac{1}{2\nw^2}(\|\vw_m\|^2+\|\vw\|^2_{\mI+k\dw\mCw} - 2\vw^\top(\vw_m+k\dw\mw))
    \\&
    (\text{Let } \Delta_\mu = (k+1)\dm \text{ and } \Delta_w = k\dw.)
\\&= 
    C_4
    -\frac{1}{2\nm^2}(\|\vm_m\|^2+\|\vm\|^2_{\mI+\Delta_\mu\mCm} - 2\vm^\top(\vm_m+\Delta_\mu\mm))
    \\&~~~~~
    -\frac{1}{2\nw^2}(\|\vw_m\|^2+\|\vw\|^2_{\mI+\Delta_w\mCw} - 2\vw^\top(\vw_m+\Delta_w\mw))
\\&= 
    C_4
    -\nicefrac{\left(
        \|\vm_m\|^2+\left(\|\vm\|^2_{\mI+\DmCm} - 2\vm^\top(\vm_m+\Dmm) + \|\vm_m+\Dmm\|^2_{(\mI+\DmCm)^{-1}}\right)
        -\|\vm_m+\Dmm\|^2_{(\mI+\DmCm)^{-1}}
    \right)}{2\nm^2}
    \\&~~~
    -\nicefrac{\left(
        \|\vw_m\|^2+\left(\|\vw\|^2_{\mI+\DmCw} - 2\vw^\top(\vw_m+\Dmw) + \|\vw_m+\Dmw\|^2_{(\mI+\DmCw)^{-1}}\right)
        -\|\vw_m+\Dmw\|^2_{(\mI+\DmCw)^{-1}}
    \right)}{2\nw^2}
\\&= 
    C_4
        -\frac{1}{2\nm^2}
        \bigg(
            \left(\|\vm_m\|^2
            -\|\vm_m+\Dmm\|^2_{(\mI+\DmCm)^{-1}}\right)
            +
            \|\vm - (\mI+\DmCm)^{-1}(\vm_m+\Dmm) \|^2_{\mI+\DmCm}
        \bigg)
    \\&~~~~~
        -\frac{1}{2\nw^2}
        \bigg(
            \left(\|\vw_m\|^2
            -\|\vw_m+\Dmw\|^2_{(\mI+\DmCw)^{-1}}\right)
            +
            \|\vw - (\mI+\DmCw)^{-1}(\vw_m+\Dmw) \|^2_{\mI+\DmCw}
        \bigg).
\end{align}
Notice $C_4$ is independent to $m$, $\vm$, and $\vw$, thus we have:
\begin{align}
    &P(\vm,\vw \vert T_m) P(\Skx \vert \vm,\vw)
\\&\propto 
        \exp\Bigg(
        -\frac{1}{2\nm^2}
        \bigg(
            \left(\|\vm_m\|^2
            -\|\vm_m+\Dmm\|^2_{(\mI+\DmCm)^{-1}}\right)
            +
            \|\vm - (\mI+\DmCm)^{-1}(\vm_m+\Dmm) \|^2_{\mI+\DmCm}
        \bigg)
        \Bigg)
    \\&~~~~~
        \cdot
        \exp\Bigg(
        -\frac{1}{2\nw^2}
        \bigg(
            \left(\|\vw_m\|^2
            -\|\vw_m+\Dmw\|^2_{(\mI+\DmCw)^{-1}}\right)
            +
            \|\vw - (\mI+\DmCw)^{-1}(\vw_m+\Dmw) \|^2_{\mI+\DmCw}
        \bigg)
        \Bigg)
\\&\propto
    \underbrace{\exp\left(
        -\frac{
            \|\vm_m\|^2
            -\|\vm_m+(k+1)\dm\mm\|^2_{(\mI+(k+1)\dm\mCm)^{-1}}
        }{2\nm^2}\right)}_{c_m^\vm}
    \underbrace{\exp\left(
        -\frac{
            \|\vw_m\|^2
            -\|\vw_m+k\dw\mw\|^2_{(\mI+k\dw\mCw)^{-1}}
        }{2\nw^2}\right)}_{c_m^\vw}
    \\&~~~~~
    \cdot
    \mathcal{N}(\vm \vert (\mI+(k+1)\dm\mCm)^{-1}(\vm_m+(k+1)\dm\mm), \nm^2(\mI+(k+1)\dm\mCm)^{-1})
    \\&~~~~~
    \cdot
    \mathcal{N}(\vw \vert (\mI+k\dw\mCw)^{-1}(\vw_m+k\dw\mw), \sw^2(\mI+k\dw\mCw)^{-1}).
\end{align}
By defining $P(\vm,\vw \vert \tT)=\mathcal{N}(\vm \vert (\mI+(k+1)\dm\mCm)^{-1}(\vm_m+(k+1)\dm\mm), \nm^2(\mI+(k+1)\dm\mCm)^{-1})
\cdot
\mathcal{N}(\vw \vert (\mI+k\dw\mCw)^{-1}(\vw_m+k\dw\mw), \sw^2(\mI+k\dw\mCw)^{-1})$ and $\tpi_m = \pi_m c_m^\vm c_m^\vw$. We have:
\begin{align}
    \pi_m P(\vm,\vw \vert T_m) P(\Skx \vert \vm,\vw)
\propto
    \tpi_m P(\vm,\vw \vert \tT_m).
\end{align}
Therefore,
\begin{align}
    \sum_{m=1}^M \pi_m P(\vm,\vw \vert T_m) P(\Skx \vert \vm,\vw)
\propto
    \sum_{m=1}^M \tpi_m P(\vm,\vw \vert \tT_m).
\end{align}
\begin{figure}[t]
    \centering
    \includegraphics[width=1.00\textwidth]{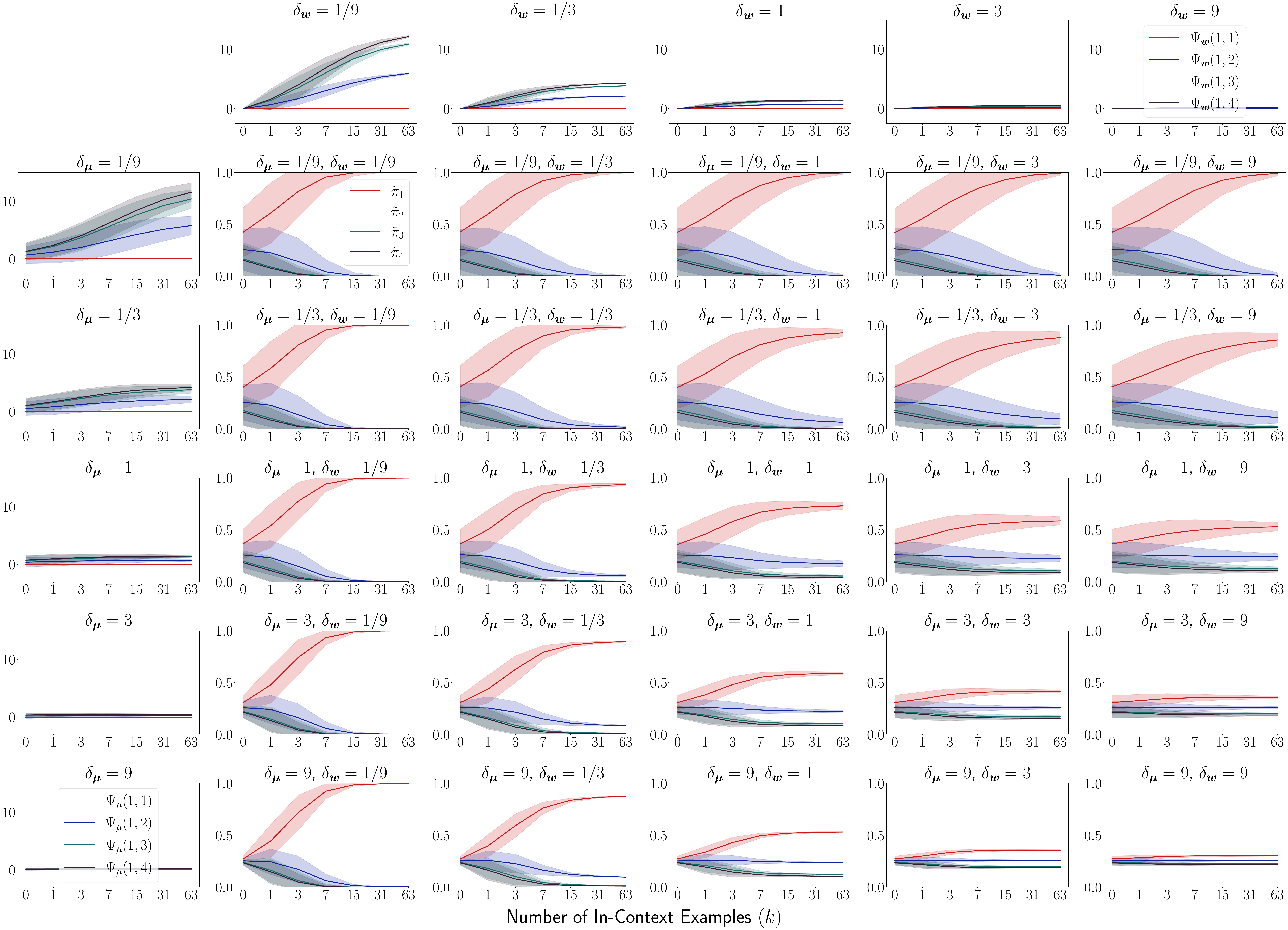}
    \caption{\textbf{Numerical analysis on component re-weighting.} The trends of $\advm$, $\advw$, and $\pi_m$ for \CR with increasing $k$ under varying task noise parameters.}
    \label{fig:TR_together}
\end{figure}
\section{Detailed Analysis of Component Shifting and Re-weighting}
\label{sec:detailedCSCR}
\subsection{Analysis of Component Re-weighting}
\label{subsec:TR}
This section analyzes the \CR effect on $\tpi_\beta$ as $k$ increases.
We focus on whether $\tpi_\alpha$ of $\tT_\alpha$ surpasses $\tpi_\beta$ of any other $\tT_\beta$ with $\beta\neq\alpha$, where $\alpha$ is the index of the closest prior center to the in-context task as described in Assumption~\ref{asu:source}.
We assess this via the ratio $r(\alpha,\beta)$ of $\tpi_\alpha$ to $\tpi_\beta$:
\begin{align}
\label{equation:ratio}
r(\alpha,\beta) 
= \frac{\tpi_\alpha}{\tpi_\beta}
= \frac{\pi_\alpha C_0 c^\vm_\alpha c^\vw_\alpha}{\pi_\beta C_0 c^\vm_\beta c^\vw_\beta}
=\frac{\pi_\alpha}{\pi_\beta}\exp(\advm(\alpha,\beta)+\advw(\alpha,\beta)),
\end{align}
where we define two functions $\Psi_\vm(\alpha,\beta)=\log(c^\vm_\alpha/c^\vm_\beta)$ and $\Psi_\vw(\alpha,\beta)=\log(c^\vw_\alpha/c^\vw_\beta)$ to facilitate the analyses of how $r(\alpha,\beta)$ changes with increasing $k$.

\paragraph{Analysis of $\advm(\alpha,\beta)$.}
\label{subsubsec:advm}
We further simplify the function $\advm(\alpha,\beta)$ as follows:
\begin{align}
\label{re-weight:vmu}
\advm(\alpha,\beta)
=
(
    \sum_{i=1}^{k+1} \|\vm_\beta-\vx_i\|^2 - \sum_{i=1}^{k+1} \|\vm_\alpha-\vx_i\|^2
)/(
    2\nx^2(1+(k+1)\dm)
).
\end{align}
(\text{See Appendix~\ref{bppsub:advm} for derivation}.) Since $\vx_i \sim \mathcal{N}(\vmt,\nxd^2\mI)$, choosing $\vmt$ closer to $\vm_\alpha$ tends to make $\advm(\alpha,\beta)$ positive and increase faster with increasing $k$.
However, as $k$ approaches infinity, $\advm(\alpha,\beta)$ stabilizes rather than increasing infinitely, \ie,
$    \lim_{k\rightarrow\infty}\advm(\alpha,\beta)
    =
    (\|\vm_\beta-\vmt\|^2 - \|\vm_\alpha-\vmt\|^2)/        (2\nm^2).
$
The leftmost column of Fig.~\ref{fig:TR_together} shows the numerical computation of $\advm(\alpha,\beta)$ with varied task noises under the tetrahedron setting (see Appendix~\ref{sec:3d} for setting details).
The smaller the value of $\dm$ ($= \frac{\nm^2}{\nx^2}$) is, the easier for $\advm(\alpha,\beta)$ to increase as $k$ increases.

Meanwhile, we also have:
\begin{align}
\label{converge:dm2advm}
    \lim_{\sm \rightarrow 0} \advm(\alpha, \beta)
=
    (
    \sum_{i=1}^{k+1} \|\vm_\beta-\vx_i\|^2 - \sum_{i=1}^{k+1} \|\vm_\alpha-\vx_i\|^2
    )/(2\nx^2)
\end{align}

\paragraph{Analysis of $\advw(\alpha,\beta)$.}
\label{subsubsec:advw}
We further simplify the function $\advw(\alpha,\beta)$ as follows:
\begin{align}
\label{re-weight:vw}
\advw(\alpha,\beta)
=
(
    \|\vw_\beta-\vwt\|^2_{\mI-(\mI+k\dw\mCw)^{-1}}-\|\vw_\alpha-\vwt\|^2_{\mI-(\mI+k\dw\mCw)^{-1}}
)/
    (2\nw^2)
.
\end{align}
(\text{See Appendix~\ref{bppsub:advw} for derivation}.) Since $k\dw\mCw$ ($= \dw \sum_{i=1}^k \vx_i\vx_i^\top$, see definition of $\mCw$ in Lemma~\ref{lemma:posterior}) is semi-positive definite, thus choosing $\vwt$ closer to $\vw_\alpha$ tends to make $\advw(\alpha,\beta)$ positive and increase faster as $k$ increases.
However, as $k$ approaches infinity, $\lim_{k\rightarrow\infty} k\dw\mCw = \lim_{k\rightarrow\infty} k\dw \frac{\sum_{i=1}^k \vx_i\vx_i^\top}{k} = k\dw (\vmt\vmt^\top+\nxd^2\mI)$. 
Thus, $\lim_{k\rightarrow\infty} \mI-(\mI+k\dw\mCw)^{-1} = \mI$ and $\advw(\alpha,\beta)$ stabilizes rather than increasing infinitely, \ie,
$    \lim_{k\rightarrow\infty}\advw(\alpha,\beta)
    =
    (
        \|\vw_\beta-\vwt\|^2-\|\vw_\alpha-\vwt\|^2
    )/
        (2\nw^2)
$.
The topmost row of Fig.~\ref{fig:TR_together} shows the numerical computation of $\advw(\alpha,\beta)$ with varied task noises under the tetrahedron setting (see Appendix~\ref{sec:3d} for setting details).
The smaller the value of $\dw$ ($= \frac{\nw^2}{\ny^2}$) is, the easier for $\advw(\alpha,\beta)$ to increase as $k$ increases.
However, one should note that $\|\vw_\beta-\vwt\|^2 \geq \|\vw_\alpha-\vwt\|^2$ does not necessarily imply $\|\vw_\beta-\vwt\|^2_{\mI-(\mI+k\dw\mCw)^{-1}} \geq \|\vw_\alpha-\vwt\|^2_{\mI-(\mI+k\dw\mCw)^{-1}}$.

Meanwhile, we also have:
\begin{align}
    \lim_{\sw \rightarrow 0} \advw(\alpha, \beta)
&=
    (\|\vw_\beta-\vwt\|^2_{k\dw\mCw} - \|\vw_\alpha-\vwt\|^2_{k\dw\mCw}
    )/(2\nw^2)
\\&=
    (\|\vm_\beta-\vx_i\|^2_{k\mCw} - \|\vm_\alpha-\vx_i\|^2_{k\mCw}
    )/(2\ny^2)
\\&\label{converge:dw2advw}
    =
    (\sum_{i=1}^k\|y_i^\beta-y_i^*\|^2 - \sum_{i=1}^k\|y_i^\alpha-y_i^*\|^2
    )/(2\ny^2),
\end{align}
where $y_i^\beta = \langle \vx_i, \vw_\beta \rangle$, $y_i^\alpha = \langle \vx_i, \vw_\alpha \rangle$, and $y_i^* = \langle \vx_i, \vwt \rangle$.

Therefore, combine Eqs. \ref{converge:dm2advm} and \ref{converge:dw2advw} and we have:
\begin{align}
&
    \lim_{\sm,\sw \rightarrow 0} \advm(\alpha, \beta) + \advw(\alpha, \beta)
\\\label{advconverge}&=
    \frac{
        \|\vm_\beta-\vx_{k+1}\|^2 - \|\vm_\alpha-\vx_{k+1}\|^2
    }{2\nx^2}
    + \sum_{i=1}^k
    (
        \frac{
        \|\vm_\beta-\vx_i\|^2 - \|\vm_\alpha-\vx_i\|^2
    }{2\nx^2}
    +
        \frac{
        \|y_i^\beta-y^*_i\|^2 - \|y_i^\alpha-y^*_i\|^2
    }{2\ny^2}
    )
\end{align}

\paragraph{Numerical Computations of Component Re-weighting.}
We have seen how noises $\nm$ and $\nw$ of the task prior affect the values of $\advm$ and $\advw$ with increasing $k$.
We further show the numerical computation of $\tpi_\beta$ in the center of Fig.~\ref{fig:TR_together}.
The figure shows that the smaller $\dm$ and $\dw$ are, the larger $\advm(\alpha,\beta)$ and $\advw(\alpha,\beta)$ will be with increasing $k$, and the easier for the mixture component $\tT_\alpha$ to dominates in the posterior with an increasing number of in-context examples.

\begin{figure}[t]
    \centering
    \includegraphics[width=.90\textwidth]{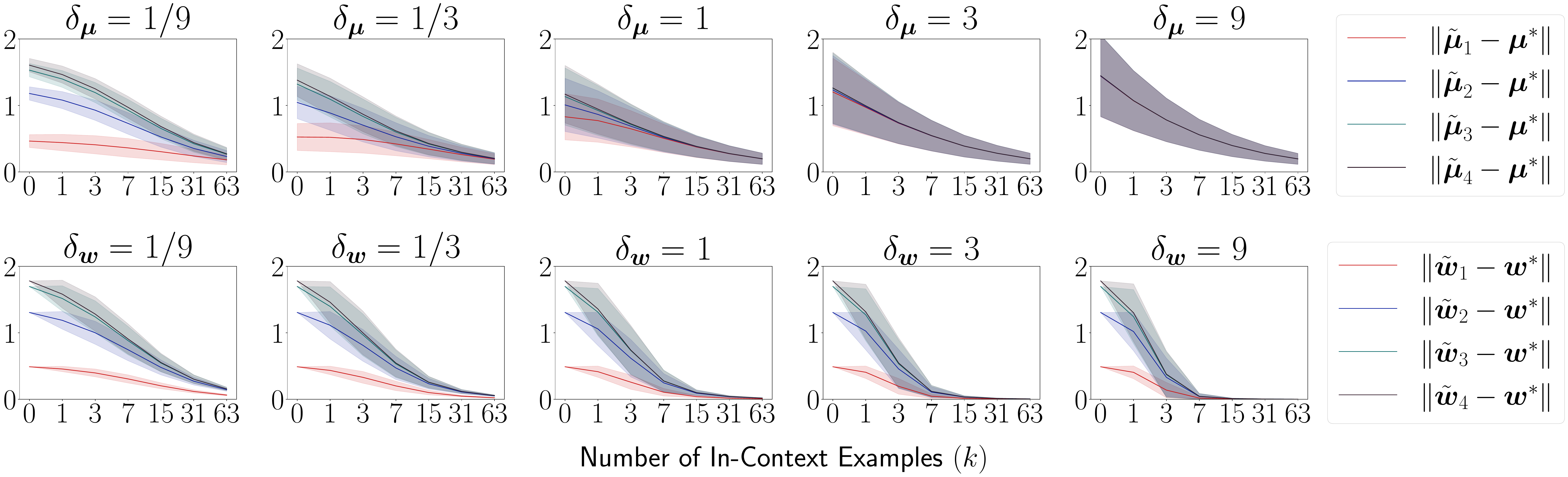}
    \caption{Numerical computations of $\|\tvm_m-\vmt\|$, $\|\tvw_m-\vwt\|$ for \CSfull (\CS).}
    \label{fig:TS_separate}
\end{figure}
\subsection{Analysis of \CSfull}
\label{subsec:TS}
The \CSfull effect in Lemma~\ref{lemma:posterior} involves shifting the variables $\tvm_m$ and $\tvw_m$:
\begin{align}
\label{equation: shift1}
\tvm_m &= 
(\mI+(k+1)\dm\mCm)^{-1}(\vm_m+(k+1)\dm\mm),\\
\label{equation: shift2}
\tvw_m &= 
(\mI+k\dw\mCw)^{-1}(\vw_m+k\dw\mw).
\end{align}
The following analyses examine these two variables with increasing $k$.

\paragraph{Analysis of $\tvm_m$.}
We provide the derivation of $\tvm_m$ in Eq.~\ref{equation: shift1} (see Appendix~\ref{bppsub:shiftingm} for details):
\begin{align}
    \label{shift:vmu}
    \tvm_m 
    = (\vm_m+k\dm\mm)/(1+(k+1)\dm).
\end{align}
Thus, when $k$ increases, $\tvm_m$ moves close to the value of $\frac{\sum_{i=1}^k \vx_i}{k}$ and $\lim_{k\rightarrow \infty} \tvm_m=\vmt$. 
We also show the numerical computation of the distance between shifted $\tvm_m$ and $\vmt$ in the first row of Fig.~\ref{fig:TS_separate}.

\paragraph{Analysis of $\tvw_m$.}
We provide the derivation of $\tvw_m$ in Eq.~\ref{equation: shift2} (see Appendix~\ref{bppsub:shiftingw} for details):
\begin{align}
\label{shift:vw}
\tvw_m
= (\mI+k\dw\mCw)^{-1}(\vw_m-\vwt) + \vwt.
\end{align}
Notice when $k\rightarrow \infty$, $k\dw\mCw = k\dw \frac{\sum_{i=1}^k \vx_i\vx_i^\top}{k} \rightarrow k\dw (\nxd^2\mI + \vwt\vwt^\top)$, thus $\lambda_d(k\dw\mCw) \rightarrow \infty$, $\lambda_1((\mI+k\dw\mCw)^{-1}) \rightarrow 0$, $\lim_{k\rightarrow \infty}(\mI+k\dw\mCw)^{-1}(\vw_m-\vwt) \leq \lim_{k\rightarrow \infty}\lambda_1((\mI+k\dw\mCw)^{-1})\cdot\|\vw_m-\vwt\| = 0$ and $\lim_{k\rightarrow \infty} \tvw_m=\vwt$, where $\lambda_d(\mA)$ indicates the minimum eigenvalue of $\mA$.
We also show the numerical computed distance between $\tvw_m$ and $\vwt$ in the second row of Fig.~\ref{fig:TS_separate}.

\subsection{Derivation Collection of $\Psi_\mu(\alpha,\beta)$ and $\Psi_w(\alpha,\beta)$}
\label{bpp:adv}
This section collects derivations for $\Psi_\vm(\alpha,\beta)$ and $\Psi_\vw(\alpha,\beta)$.
The derivation of $\Psi_\vm(\alpha,\beta)$ is collected in Sec~\ref{bppsub:advm} and the derivation of $\Psi_\vw(\alpha,\beta)$ is collected in Sec~\ref{bppsub:advw}.

\subsubsection{Derivation of $\Psi_\mu(\alpha,\beta)$}
\label{bppsub:advm}
This section collects the derivation of $\Psi_\vm(\alpha,\beta)$ in Eq.~\ref{re-weight:vmu} of Sec.~\ref{subsubsec:advm}:
\begin{align}
&\Psi_\vm(\alpha,\beta)
\\&=
    \log(c^\vm_\alpha/c^\vm_\beta)
\\&=
\log\left(
\frac{
    \exp\left(
    -\frac{
         \|\vm_\beta\|^2 - \|\vm_\beta+(k+1)\dm\mm\|^2_{(\mI+(k+1)\dm\mCm)^{-1}}
    }{2\smu^2}
    \right)
}{
    \exp\left(
    -\frac{
        \|\vm_\alpha\|^2 - \|\vm_\alpha+(k+1)\dm\mm\|^2_{(\mI+(k+1)\dm\mCm)^{-1}} 
    }{2\smu^2}
    \right)
}
\right)\\
&=
\frac{(1+(k+1)\dm)\|\vm_\beta\|^2 - \|\vm_\beta+\dm \sum_{i=1}^{k+1} \vx_i\|^2}{2\smu^2(1+(k+1)\dm)}
-
\frac{(1+(k+1)\dm)\|\vm_\alpha\|^2 - \|\vm_\alpha+\dm \sum_{i=1}^{k+1} \vx_i\|^2}{2\smu^2(1+(k+1)\dm)}\\
&=
\frac{ - \|\vm_\beta+\dm \sum_{i=1}^{k+1} \vx_i\|^2}{2\smu^2(1+(k+1)\dm)}
-
\frac{ - \|\vm_\alpha+\dm \sum_{i=1}^{k+1} \vx_i\|^2}{2\smu^2(1+(k+1)\dm)}\\
&=
\frac{
    - \|\vm_\beta\|^2 
    - 2\vm_\beta^\top(\dm \sum_{i=1}^{k+1} \vx_i)
    - \|\dm \sum_{i=1}^{k+1} \vx_i\|^2
}{2\smu^2(1+(k+1)\dm)}
-
\frac{
    - \|\vm_\alpha\|^2 
    - 2\vm_\alpha^\top(\dm \sum_{i=1}^{k+1} \vx_i)
    - \|\dm \sum_{i=1}^{k+1} \vx_i\|^2
}{2\smu^2(1+(k+1)\dm)}
\\
&=
\frac{
    (k+1)\dm\|\vm_\beta\|^2 
    - 2\vm_\beta^\top(\dm \sum_{i=1}^{k+1} \vx_i)
    + \dm \sum_{i=1}^{k+1} \|\vx_i\|^2
}{2\smu^2(1+(k+1)\dm)}
-
\frac{
    (k+1)\dm\|\vm_\alpha\|^2 
    - 2\vm_\alpha^\top(\dm \sum_{i=1}^{k+1} \vx_i)
    + \dm \sum_{i=1}^{k+1} \|\vx_i\|^2
}{2\smu^2(1+(k+1)\dm)}
\\
&=
\frac{
    \sum_{i=1}^{k+1} \dm\|\vm_\beta-\vx_i\|^2
}{2\smu^2(1+(k+1)\dm)}
-
\frac{
    \sum_{i=1}^{k+1} \dm\|\vm_\alpha-\vx_i\|^2 
}{2\smu^2(1+(k+1)\dm)}
\\
&=
\frac{
    \sum_{i=1}^{k+1} \|\vm_\beta-\vx_i\|^2 - \sum_{i=1}^{k+1} \|\vm_\alpha-\vx_i\|^2
}{
    2\nx^2(1+(k+1)\dm)
}.
\end{align}

\subsubsection{Derivation of $\Psi_w(\alpha,\beta)$}
\label{bppsub:advw}
This section collects the derivation of $\Psi_\vw(\alpha,\beta)$ in Eq.~\ref{re-weight:vw} of Sec.~\ref{subsubsec:advw}:
\begin{align}
&\Psi_\vw(\alpha,\beta)
\\&=
    \log(c^\vw_\alpha/c^\vw_\beta)
\\&=
\log\left(
\frac{
    \exp\left(
    -\frac{
        \|\vw_\alpha\|^2 - \|\vw_\alpha+k\dw\mw\|^2_{(\mI+k\dw\mCw)^{-1}} 
    }{2\sw^2}
    \right)
}{
    \exp\left(
    -\frac{
        \|\vw_\beta\|^2 - \|\vw_\beta+k\dw\mw\|^2_{(\mI+k\dw\mCw)^{-1}} 
    }{2\sw^2}
    \right)
}
\right)
\\&=
\frac{
    \|\vw_\beta\|^2-\|\vw_\beta +k\dw\mw\|^2_{(\mI+k\dw\mCw)^{-1}}
}{
    2\sw^2
}
-
\frac{
    \|\vw_\alpha\|^2-\|\vw_\alpha+k\dw\mw\|^2_{(\mI+k\dw\mCw)^{-1}} 
}{
    2\sw^2
}
\\&
(\text{Note } k\dw\mw = \dw \sum_{i=1}^k \vx_i y_i = \dw \sum_{i=1}^k \vx_i \vx_i^\top \vwt = k\dw\mCw \vwt 
.)
\\&=
\frac{
    \|\vw_\beta\|^2-\|\vw_\beta +k\dw\mCw \vwt\|^2_{(\mI+k\dw\mCw)^{-1}}
}{
    2\sw^2
}
-
\frac{
    \|\vw_\alpha\|-\|\vw_\alpha+k\dw\mCw \vwt\|^2_{(\mI+k\dw\mCw)^{-1}} 
}{
    2\sw^2
}
\\&=
\frac{
    \|\vw_\beta\|^2-\|(\vw_\beta-\vwt) +(\mI+k\dw\mCw) \vwt\|^2_{(\mI+k\dw\mCw)^{-1}}
}{
    2\sw^2
}
-
\frac{
    \|\vw_\alpha\|^2-\|(\vw_\alpha-\vwt)+(\mI+k\dw\mCw) \vwt\|^2_{(\mI+k\dw\mCw)^{-1}} 
}{
    2\sw^2
}
\\&=
\frac{
    \|\vw_\beta\|^2-\|\vw_\beta-\vwt\|^2_{(\mI+k\dw\mCw)^{-1}} -2(\vw_\beta-\vwt)^\top\vwt
}{
    2\sw^2
}
-
\frac{
    \|\vw_\alpha\|^2-\|\vw_\alpha-\vwt\|^2_{(\mI+k\dw\mCw)^{-1}} -2(\vw_\alpha-\vwt)^\top\vwt
}{
    2\sw^2
}
\\&=
\frac{
    \|\vw_\beta-\vwt\|^2 - \|\vw_\beta-\vwt\|^2_{(\mI+k\dw\mCw)^{-1}}
}{
    2\sw^2
}
-
\frac{
    \|\vw_\alpha-\vwt\|^2 - \|\vw_\alpha-\vwt\|^2_{(\mI+k\dw\mCw)^{-1}}
}{
    2\sw^2
}
\\&=
\frac{
    \|\vw_\beta-\vwt\|^2_{\mI-(\mI+k\dw\mCw)^{-1}}-\|\vw_\alpha-\vwt\|^2_{\mI-(\mI+k\dw\mCw)^{-1}}
}{
    2\sw^2
}.
\end{align}

\subsection{Derivation Collection of $\tilde{\vm}_m$ and $\tilde{\vw}_m$}
\label{bpp:shifting}
This section collects derivations for $\tvm_m$ and $\tvw_m$.
The derivation of $\tvm_m$ is collected in Appendix~\ref{bppsub:shiftingm}, and the derivation of $\tvw_m$ is collected in Appendix~\ref{bppsub:shiftingw}.

\subsubsection{Derivation of $\tilde{\vm}_m$}
\label{bppsub:shiftingm}
This section collects the derivation of $\tvm_m$ in Eq.~\ref{shift:vmu} of Sec.~\ref{subsubsec:advm}: 
\begin{align}
    \tvmu_m
    &= (\mI+(k+1)\dm\mCm)^{-1}(\vm_m+(k+1)\dm\mm)\\
    &= (\mI+(k+1)\dm\mI)^{-1}(\vm_m+\dm\sum_{i=1}^{k+1}\vx_i)\\
    &= \frac{\vm_m+\dm\sum_{i=1}^{k+1}\vx_i}{1+(k+1)\dm}.
\end{align}

\subsubsection{Derivation of $\tilde{\vw}_m$}
\label{bppsub:shiftingw}
This section collects the derivation of $\tvw_m$ in Eq.~\ref{shift:vw} of Sec.~\ref{subsubsec:advm}:
\begin{align}
\tvw_m &= (\mI+k\dw\mCw)^{-1}(\vw_m+k\dw\mw)\\
&(\text{Recall } k\dw\mw = \dw \sum_{i=1}^k\vx_i y_i = \dw \sum_{i=1}^k\vx_i\vx_i^\top\vwt = k\dw\mCw \vwt.)\\
&= (\mI+k\dw\mCw)^{-1}(\vw_m+k\dw\mCw \vwt)\\
&= (\mI+k\dw\mCw)^{-1}(\vw_m-\vwt+(\mI+k\dw\mCw) \vwt)\\
\label{equation:vwshift}
&= (\mI+k\dw\mCw)^{-1}(\vw_m-\vwt) + \vwt.
\end{align}

\section{Additional Experiments for Early Ascent}
\subsection{Early Ascent and Bounded Efficacy under Noisy Labels}
\label{app:expnoisy}
We further examine phenomena of early ascent and bounded efficacy with noisy labels under varied noise levels.
The results show that these two phenomena are robust to label noises to some extend.
\begin{figure}[th!]
    \centering
    \subfigure[ICL risk under label noise level $\tau_y=0.0$.]{
        \includegraphics[width=0.45\textwidth]{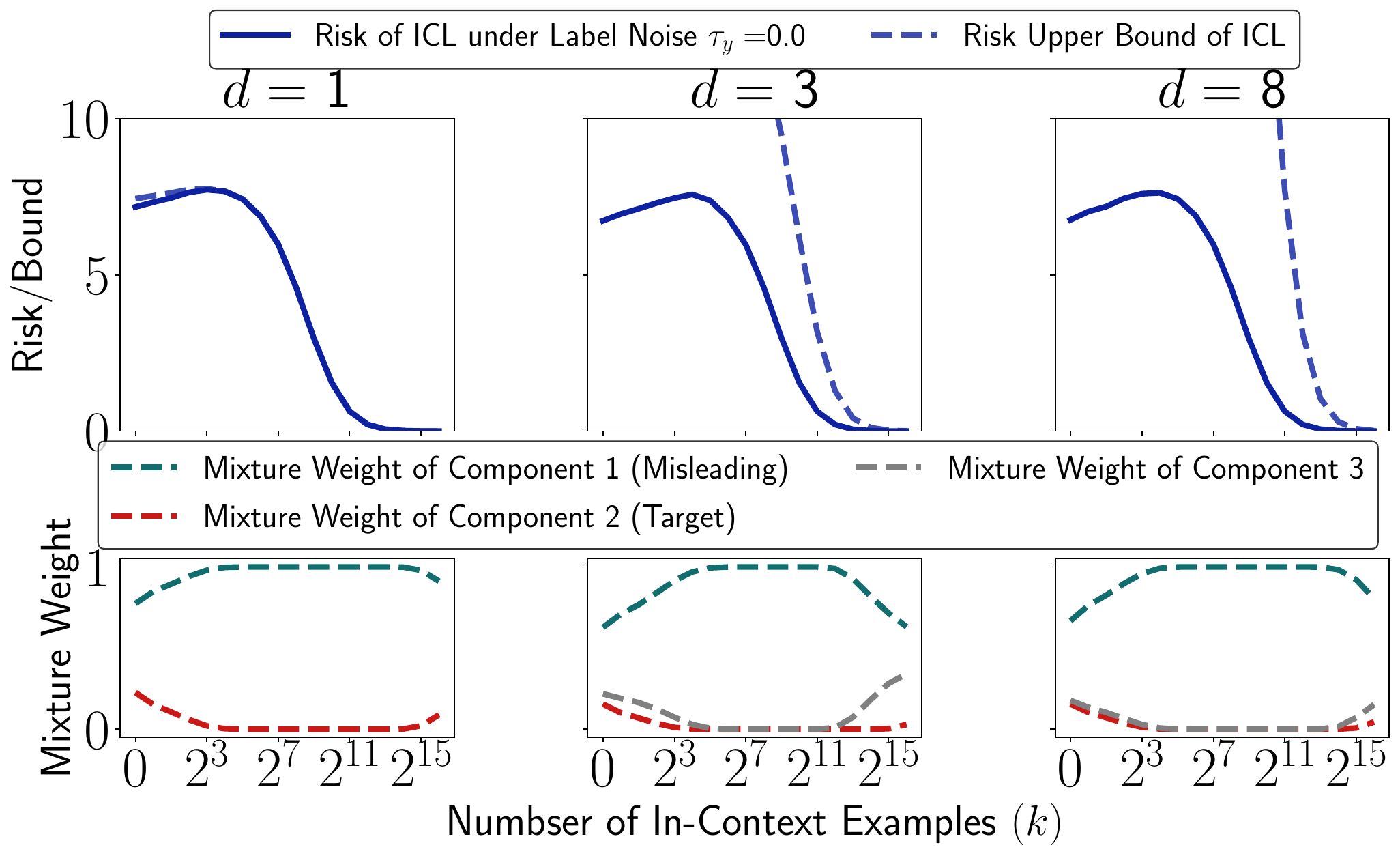}
        \label{fig:earlyascentnoise1}}
    \hspace{1em}
    \subfigure[ICL risk under label noise level $\tau_y=0.01$.]{
        \includegraphics[width=0.45\textwidth]{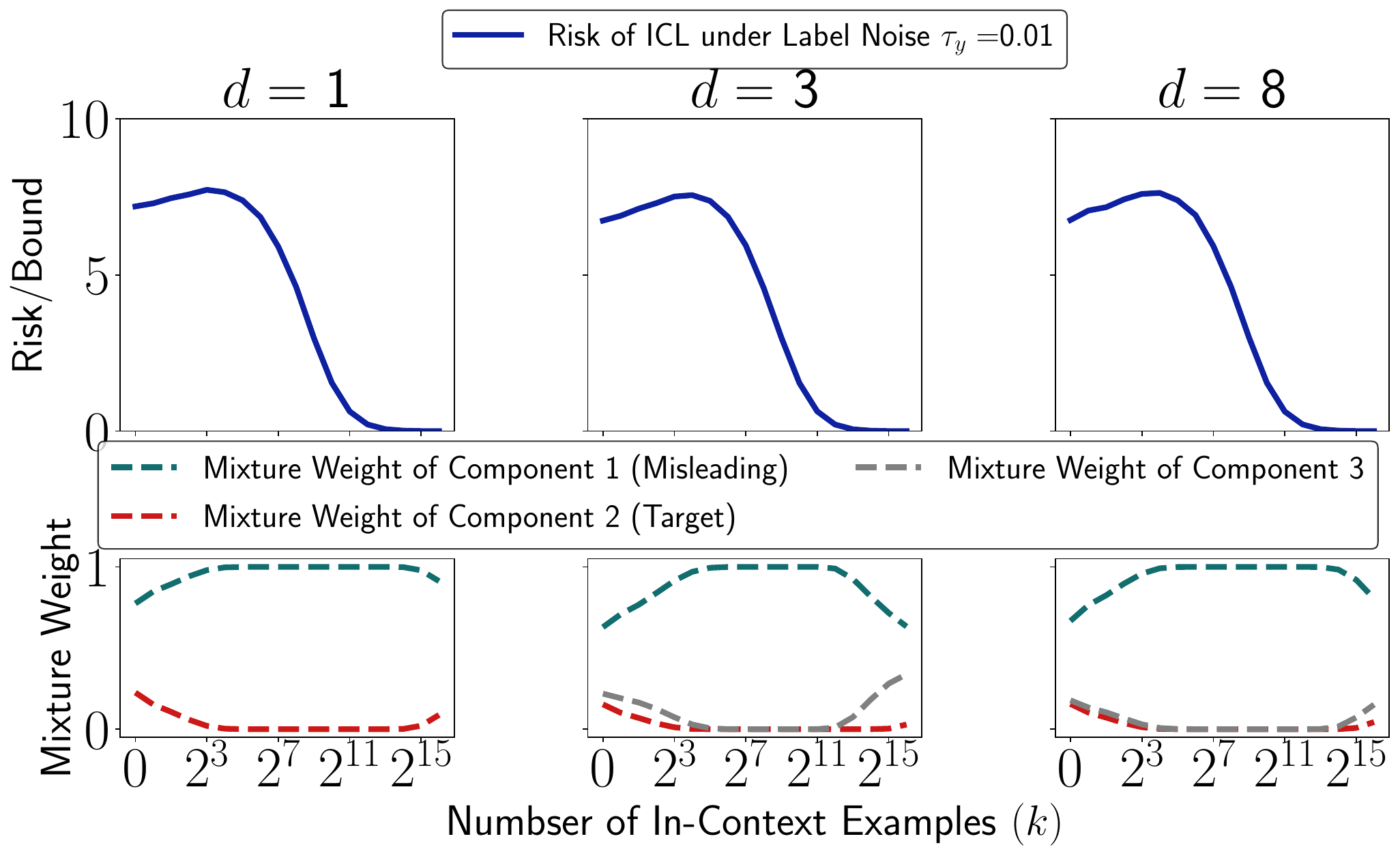}
        \label{fig:earlyascentnoise2}}
    \hspace{1em}
    \subfigure[ICL risk under label noise level $\tau_y=0.1$.]{
        \includegraphics[width=0.45\textwidth]{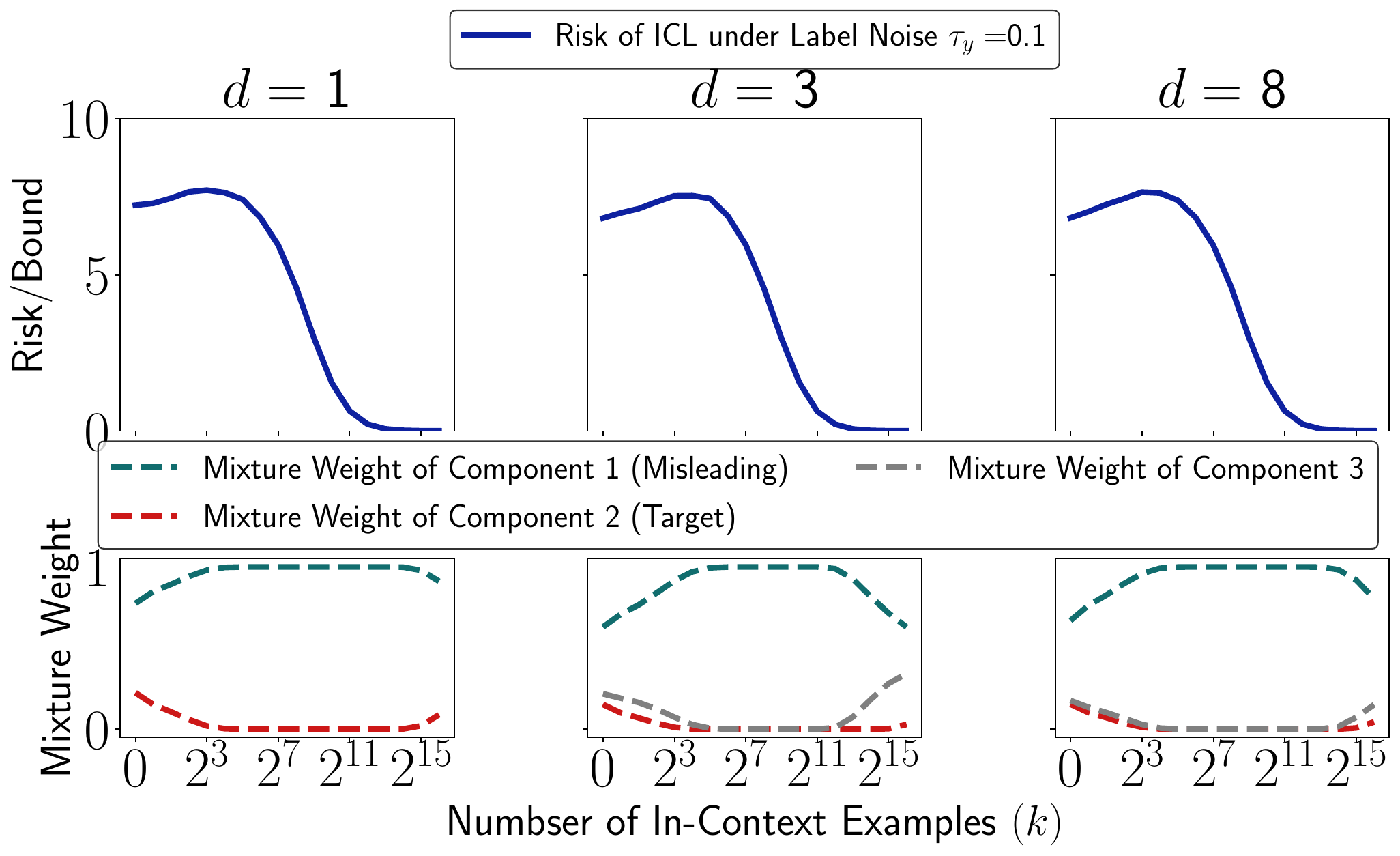}
        \label{fig:earlyascentnoise3}}
    \hspace{1em}
    \subfigure[ICL risk under label noise level $\tau_y=1.0$.]{
        \includegraphics[width=0.45\textwidth]{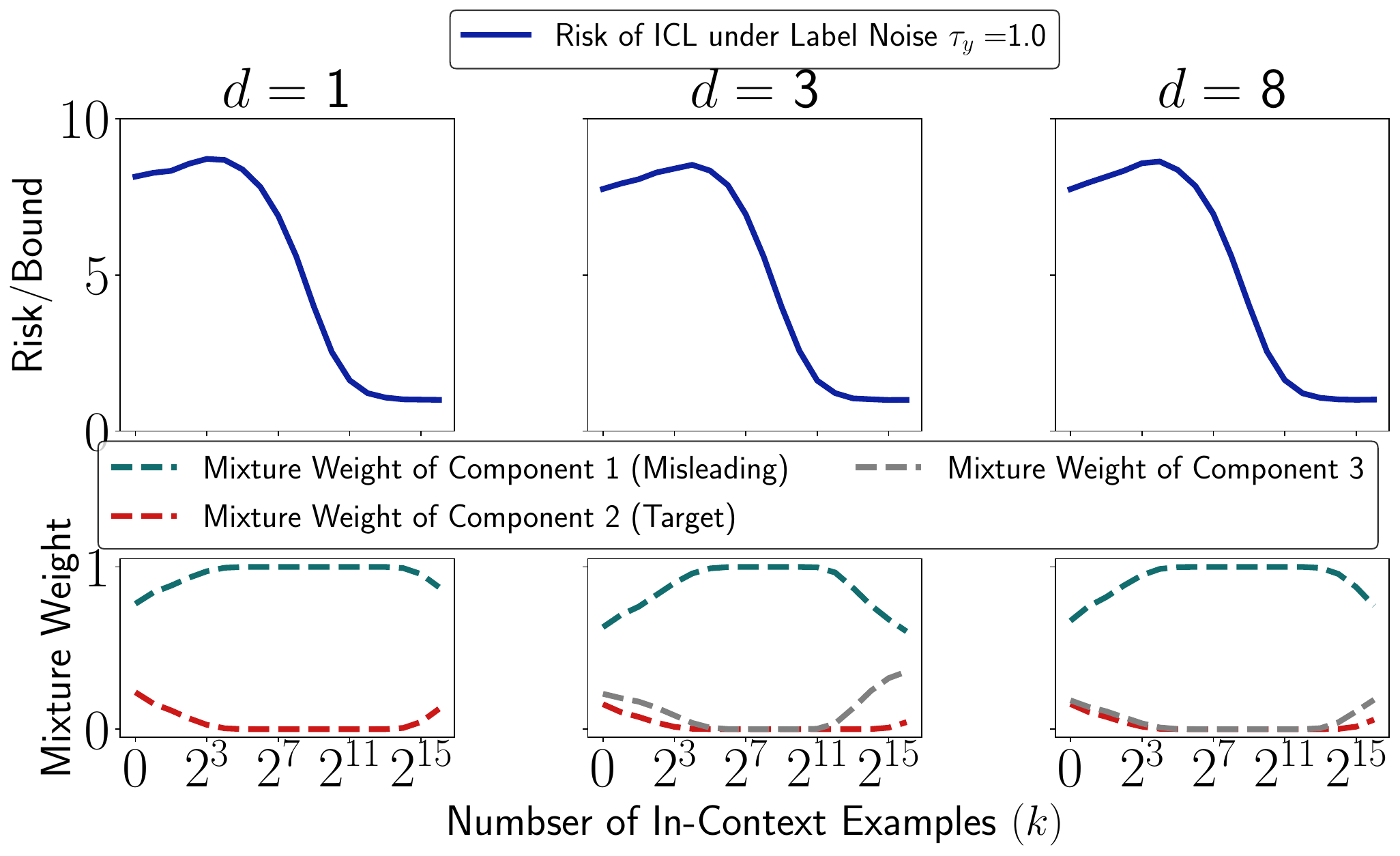}
        \label{fig:earlyascentnoise4}}
    \caption{\textbf{Early ascent under varied label noises.} Results show that the early ascent phenomenon maintains for noise level $\tau_y\in[0,1.0]$.
    Label noise level $\ny=1.0$ is used for pretraining.
    }
    \label{fig:earlyascentnoise}
\end{figure}
\begin{figure}[th!]
    \centering
    \includegraphics[width=0.9\textwidth]{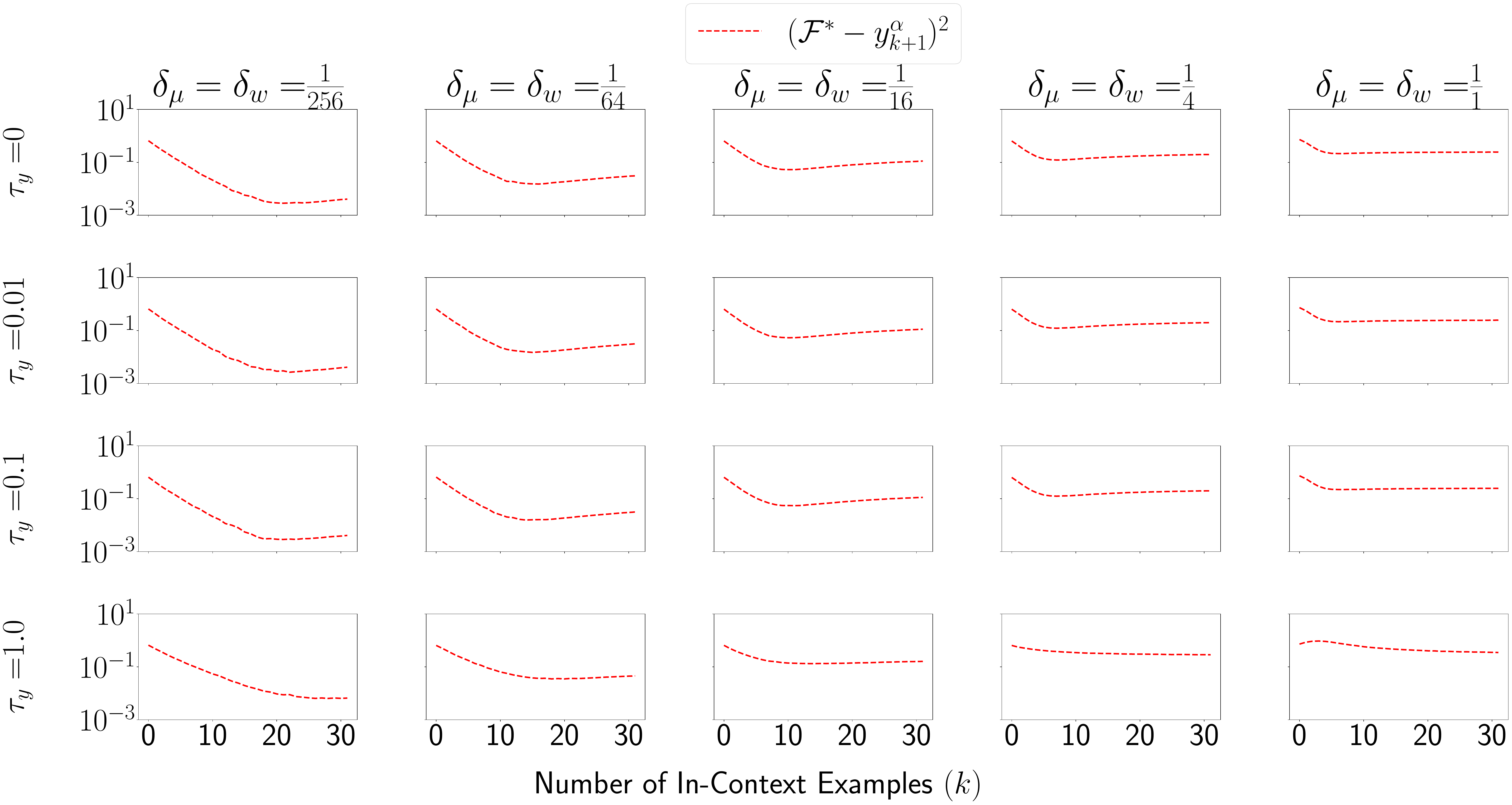}
    \caption{\textbf{Bounded efficacy under varied label noises.} Results show that the bounded efficacy phenomenon maintains for noise level $\tau_y\in[0,0.1]$.
    Label noise level $\ny=1.0$ is used for pretraining.}
    \label{fig:boundedefficacynoise}
\end{figure}

\subsection{Early Ascent under Non-Linear Regression and Discrete Token Prediction}
\label{app:earlyascent:nonlinear&discrete}
This section uses Fig.~\ref{fig:rebuttal} to show the existence of the early ascent phenomenon on non-linear regression and discrete token prediction with our designed distributions of pretraining and in-context samples.
Fig.~\ref{fig:nonlinear} shows that the early ascent phenomenon exists when a 2-layer neural network with Tanh Activation function serves as the non-linear function, and Fig.~\ref{fig:textualized} shows that the early ascent phenomenon exists when the dataset consists of sequences of tokens with discrete values rather than sequences of vectors with continuous values.
For the details of experiments including our designed distributions of pretraining and in-context samples, please refer to Sec.~\ref{non-linear} for the experiment with non-linear regression and Sec.~\ref{textualized} for the experiment with discrete token prediction.

\begin{figure}[th!]
\centering
    \subfigure[Experiment under non-linear regressions.]{
        \includegraphics[width=0.45\textwidth]{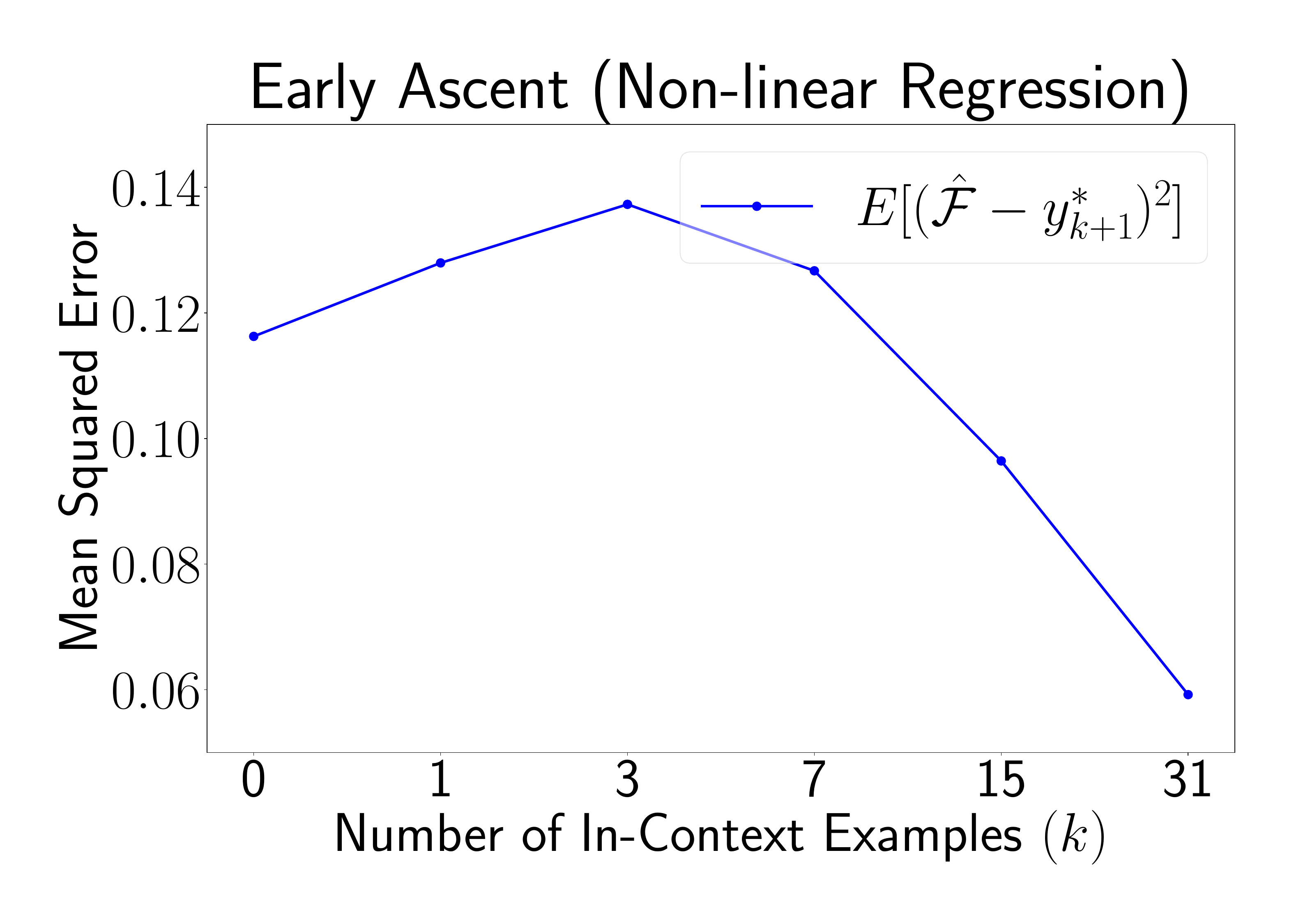}
        \label{fig:nonlinear}}
    \hspace{1em}
    \subfigure[Experiment under discrete token prediction.]{
        \includegraphics[width=0.45\textwidth]{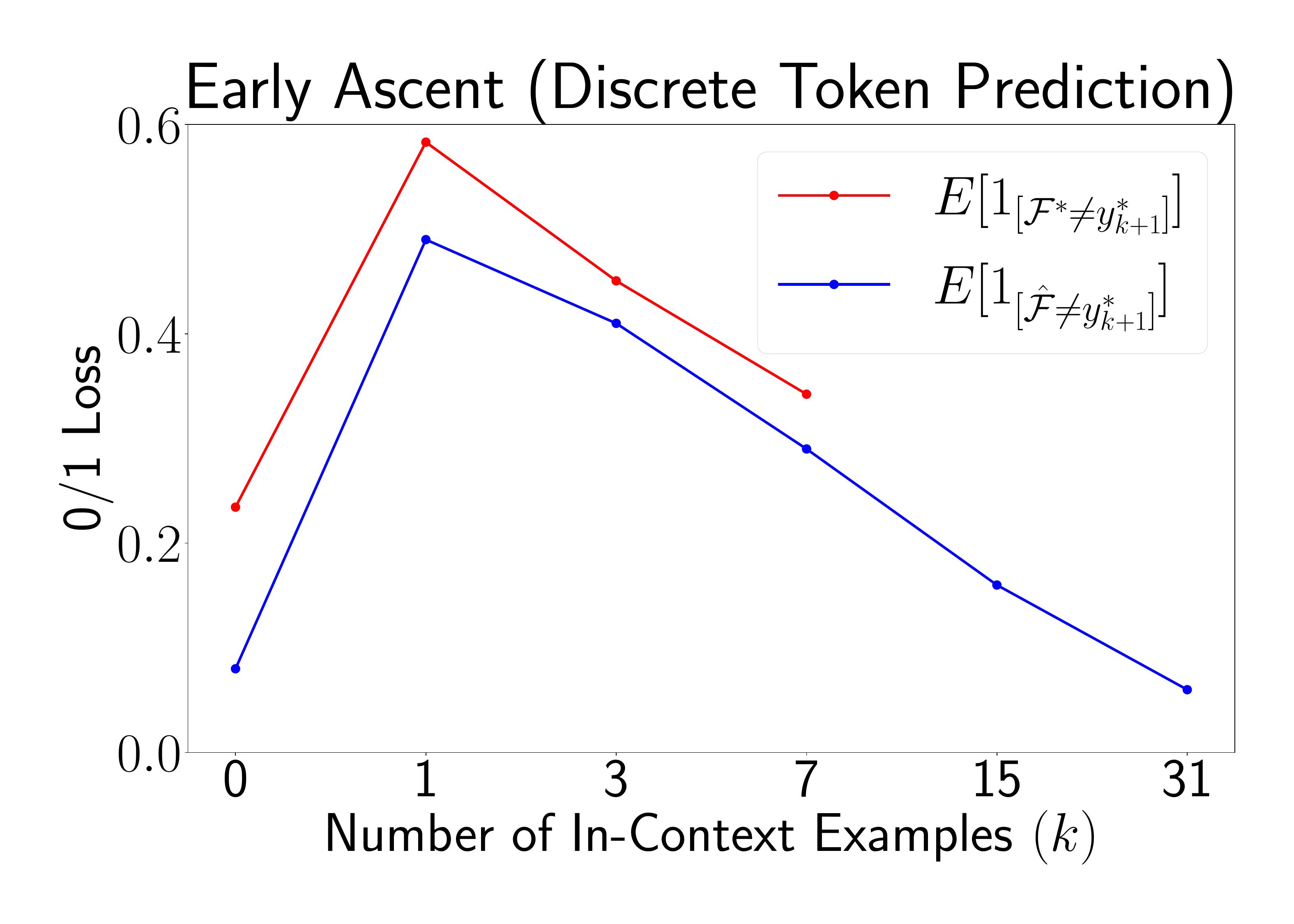}
        \label{fig:textualized}}
    \caption{$\hat{\F}$ indicates the prediction by a pretrained Transformer model and $\F^*$ indicates the prediction by numerical computation following a Bayes optimal predictor. While we cannot derive the optimal predictor under non-linear regression, we can derive the optimal predictor under discrete token prediction.}
    \label{fig:rebuttal}
\end{figure}

\subsubsection{Experiment Design for Non-Linear Regression}
\label{non-linear}
The following assumption shows the data generation model to generate a non-linear sequence $[\vx_1,y_1,\ldots,\vx_K,y_K]$, where $\vx_i$ is a vector and $y_i$ is a scalar.
The non-linear function mapping $\vx$ to $y$ is highlighted in red in the assumption.
\begin{asu}[Pretraining Data Generative Model for Non-linear Regression]
\label{asu:nonlinear_assumption}
\,~\\
\subasu \label{asu:nonlinear_assumption1}
sample a task from the task distribution: $(\vm,\mW,\vv)\sim \Dpr, P(\vm,\mW,\vv) = \sum_{m=1}^M \pi_m P(\vm,\mW,\vv \vert T_m)$, where $T_m$ is the $m^\text{th}$ mixture component, \ie, $P(\vm,\mW,\vv \vert T_m) = \mathcal{N}(\vm; \vm_m,\nm^2\mI)\cdot \frac{1}{\sqrt{(2\pi)^{d^2} \sigma_W^{d^2}}} \exp(\frac{\|\mW-\mW_m\|_F^2}{2})\cdot\mathcal{N}(\vv; \vv_m,\sigma_v^2\mI)$, and $\pi_m$ is the mixture weight.
$\mathcal{N}(\vx; \vm,\bm{\Sigma})$ denotes the probability of $\vx$ in the multivariate normal distribution with mean $\vm$ and covariance matrix $\bm{\Sigma}$, $\|\cdot\|_F$ indicates the Frobenius norm, $\sum_{m=1}^M\pi_m=1$, $0<\pi_m<1$, $(\vm_m,\vw_m)$ is the center of the mixture component $T_m$,
and all components share the same covariance matrix controlled by $\nm$, $\sigma_W$, and $\sigma_v$;\\
\subasu \label{asu:nonlinear_assumption2}
input variable distribution: within a sequence, $\forall i \in [K]$, $\vx_i \sim \ddist(\vm), P(\vx\vert\vm) = \mathcal{N}(\vx\vert\vm,\nx^2\mI)$;\\
\subasu \label{asu:nonlinear_assumption3} 
label distribution: within a sequence, $\forall i \in [K]$, $y_i\vert\vx_i \sim \mathcal{D}_{y\vert \vx_i}(\mW,\vv), P(y_i\vert\vx_i,\mW,\vv) = \mathcal{N}(y_i\vert \textcolor{black}{\langle\tanh(\mW\vx_i), \vv \rangle},\ny^2)$, where $\tanh()$ is a Tanh Activation function;\\
\subasu \label{asu:nonlinear_assumption5} 
$\vx,\vm,\vm_m,\vv,\vv_m \in \mathbb{R}^d$, and $\mW, \mW_m \in \mathbb{R}^{d\times d}$.
\end{asu}

For experimental setting of Fig.~\ref{fig:nonlinear}, we set $d=2$, $\sm=1, \sigma_W=\sigma_v=0.5$, $\sigma_x=\sigma_y=1$, $M=2$, $\pi_1=0.1, \pi_2=0.9$, $\vm_1=[1,0]^\top, \vm_2=[0,1]^\top$, $\mW_1=\begin{bmatrix}
1 & 0 \\
0 & 0
\end{bmatrix},
\mW_2=\begin{bmatrix}
0 & 0 \\
0 & 1
\end{bmatrix}$, and $\vv_1=[1,0]^\top, \vv_2=[0,1]^\top$.
In-context samples follows task $(\vm^*,\mW^*,\vv^*)$, where $\vm^*=\vm_1$, $\mW^*=\mW_2$, $\vv^*=\vv_2$, and $\sigma_y=1$.
Notice that although we add label noise to in-context samples, when evaluating the prediction, we calculate error/loss based on the clean label.

\subsubsection{Experiment Design for Discrete Token Prediction}
\label{textualized}
The following assumption shows the data generation model to generate a non-linear sequence $[x_1,y_1,\ldots,x_K,y_K]$, where $x_i$ and $y_i$ are both integers (discrete tokens). 
\begin{asu}[Pretraining Data Generative Model for Discrete Token Prediction]
\label{asu:token_assumption}
\,~\\
\subasu \label{asu:token_assumption1}
sample a task from the task distribution: $(\mu,w)\sim \Dpr, \mu\in[M],w\in[M]$, $P(\mu,w) =\sum_{m=1}^M \pi_m P(\mu,w \vert T_m)$, where $T_m$ is the $m^\text{th}$ mixture component, \ie, $P(\mu,w \vert T_m) = \text{1}_{[w=w_m]}((1-(M-1)\sigma_\mu)\text{1}_{[\mu=\mu_m]}+\sigma_\mu\text{1}_{[\mu\neq\mu_m]})$, and $\pi_m$ is the mixture weight.
\\
\subasu \label{asu:token_assumption2}
input variable distribution: within a sequence, $\forall i \in [K]$, $x_i \sim \mathcal{D}_x(\mu), P(x_i\vert\mu) = (1-(M-1)\sigma_x)\text{1}_{[x=\mu]}+\sigma_x\text{1}_{[x\neq\mu]}$;\\
\subasu \label{asu:token_assumption3} 
label distribution: within a sequence, $\forall i \in [K]$, $y_i\vert x_i \sim \mathcal{D}_{y\vert x_i}(w), P(y_i\vert x_i,w) = (1-(M-1)\sigma_y)\text{1}_{[\textcolor{black}{y_i=x_i+w \mod M}]}+\sigma_y\text{1}_{[y_i\neq x_i+w \mod M]}$.
\end{asu}

For experimental setting of Fig.~\ref{fig:textualized}, we set $M=6$,$\pi_1=0.04,\pi_3=0.481,\pi_5=0.479,\pi_2=\pi_4=\pi_6=0$, $\sigma_\mu=0.05$, $\sigma_x=0.04$, $\sigma_y=0.13$, $\mu_1=w_1=1, \mu_3=w_3=3, \mu_5=w_5=5$.
In-context samples follows task $(\mu^*, w^*)$, where $\mu^*=\mu_1$, $w^*=w_3$, and $\sigma_y=0.13$. Notice that although we add label noise to in-context samples, when evaluating the prediction, we calculate error/loss based on the clean label.
\section{Mathematical Derivation for Early Ascent}
\label{app:FlippedUMath}
We show that the early ascent phenomenon occurs under a specific setting in Sec.~\ref{app:earlyascent:example}.
Then, we give formal theory with proof to show when early ascent happens in Sec.~\ref{app:earlyascent:theory}.
\subsection{A Specific Setting of Early Ascent}
\label{app:earlyascent:example}
To have a cleaner mathematical understanding of this phenomenon, this section uses the setting of $d=1$, the first row, in Table~\ref{table:earlysetting} to show the mathematical logic. (Some parameter settings are described in Table~\ref{table:earlysetting}'s caption.)
Following Theorem~\ref{the:finegrainedlearning}, the upper bound of ICL risk is as follows:
\begin{align}
    &\E_{\Skx}[\mathcal{L}_k^*]
    \\&<
    \sum_{\beta=1}^2 \|\vw_\beta-\vwt\|^2 \E_{\Skx}[\tpi_\beta \|\vx_{k+1}\|^2 \lambda_1(\mA)^2]
    \\&=
    \|\vw_1-\vwt\|^2 \E_{\Skx}[\tpi_1 \|\vx_{k+1}\|^2 \lambda_1(\mA)^2] + \|\vw_2-\vwt\|^2 \E_{\Skx}[\tpi_2 \|\vx_{k+1}\|^2 \lambda_1(\mA)^2]
    \\&
    (\text{Notice } \vw_2=\vwt, \|\vw_1-\vwt\|^2=2^2=4.)
    \\&=
    4 \E_{\Skx}[\tpi_1 \|\vx_{k+1}\|^2 \lambda_1(\mA)^2]
    \\&
    (\text{Notice } \tpi_1+\tpi_2=1.)
    \\&=
    4 \E_{\Skx}\left[\frac{\tpi_1}{\tpi_1+\tpi_2} \|\vx_{k+1}\|^2 \lambda_1(\mA)^2\right]
    \\&
    (\text{Recall } \frac{\tpi_1}{\tpi_2} = r(1,2) \text{ as Eq.~\ref{equation:ratio}}.)
    \\&=
    4 \E_{\Skx}\left[\frac{r(1,2)}{1+r(1,2)} \|\vx_{k+1}\|^2 \lambda_1(\mA)^2\right].
\end{align}
Noticing $\dm=\frac{0.05^2}{1^2}$ and $\dw=\frac{0.05^2}{2^2}$ are very small, when $k$ is small, we have $k\dw\approx 0$ and $\lambda_1(\mA)=(\mI+\dw\sum_{i=1}^k\vx_i\vx_i^\top)^{-1}\approx \mI$, thus $\E_{\Skx}\left[\frac{r(1,2)}{1+r(1,2)} \|\vx_{k+1}\|^2 \lambda_1(\mA)^2\right]\approx \E_{\Skx}\left[\frac{r(1,2)}{1+r(1,2)} \|\vx_{k+1}\|^2\right]$ and a larger $r(1,2)$ means a larger upper bound. 
In the following, we will examine whether the increase of $k$ leads to the increase of $r(1,2)$.

Following Eq.~\ref{equation:ratio}:
\begin{align}
    r(1,2)
    &=\frac{1/2}{1/2}\exp(\advm(1,2)+\advw(1,2))
    \\
    &=\exp(\advm(1,2)+\advw(1,2)).
\end{align}

We first analyze $\advm(1,2)$, following Eq.~\ref{re-weight:vmu}:
\begin{align}
    \E[\advm(1,2)]
    &=
    \E\left[\frac{\sum_{i=1}^{k+1} \|\vm_2-\vx_i\|^2 - \sum_{i=1}^{k+1} \|\vm_1-\vx_i\|^2
    }{
        2\nx^2(1+(k+1)\dm)
    }\right]
    \\&
    (\text{Since } \dm \approx 0, \text{ thus when } k \text{ is small, we have:})
    \\&\approx
    \E\left[\frac{
    \sum_{i=1}^{k+1} \|\vm_2-\vx_i\|^2 - \sum_{i=1}^{k+1} \|\vm_1-\vx_i\|^2
    }{
        2\nx^2
    }\right]
    \\&=
    \frac{k+1}{2\nx^2}\E\left[
    \|\vm_2-\vx_1\|^2 - \|\vm_1-\vx_1\|^2
    \right]
    \\&=
    \frac{k+1}{2\nx^2}(\E[\|\vm_2-\vx_1\|^2] - \E[\|\vm_1-\vx_1\|^2])
    \\&=
    \frac{k+1}{2\nx^2}(\E[\|\vm_2-\vmt\|^2]+\nxd^2) - (\E[\|\vm_1-\vmt\|^2]+\nxd^2)
    \\
    &(\vmt \text{ is the same as } \vm_1, \text{ but different from } \vm_2.)
    \\&=
    \frac{k+1}{2\nx^2}(\E[\|\vm_2-\vmt\|^2]-0)
    \\
    &=\frac{k+1}{2\times 1^2}\times 2^2
    \\
    &=2(k+1).
\end{align}
We then analyze $\advw(1,2)$, following Eq.~\ref{re-weight:vw}:
\begin{align}
    \E[\advw(1,2)]
    &=
    \E\left[-\frac{\|\vw_1-\vwt\|^2_{\mI-(\mI+k\dw\mCw)^{-1}}}{2\nw^2}\right]
    \\&
    (\text{Since } \dw \approx 0, \text{ thus when } k \text{ is small, we have:})
    \\&\approx
    -\E\left[\frac{(\vw_1-\vwt)^\top k\dw\mCw (\vw_1-\vwt)}{2\nw^2}\right]
    \\
    &(\text{Notice the feature dimension } d=1, \mCw=\frac{\sum_{i=1}^k \|\vx_i\|^2}{k}.)
    \\&\approx
    -\E\left[\frac{\|\vw_1-\vwt\|^2 k\dw\sum_{i=1}^k \|\vx_i\|^2}{2\nw^2}\right]
    \\&=
    -\E\left[\frac{2\sum_{i=1}^k \|\vx_i\|^2}{\ny^2}\right]
    \\&=
    -\frac{2k}{\ny^2} \E\left[\|\vx_1\|^2\right]
    \\&=
    -\frac{2k}{\ny^2} (\|\vmt\|^2+\nxd^2)
    \\&=
    -\frac{2k}{2^2}\times(1+1) = -k.
\end{align}
\begin{wrapfigure}{r}{0.27\textwidth}
\vspace{-0.8in}
  \begin{center}
    \includegraphics[width=0.25\textwidth]{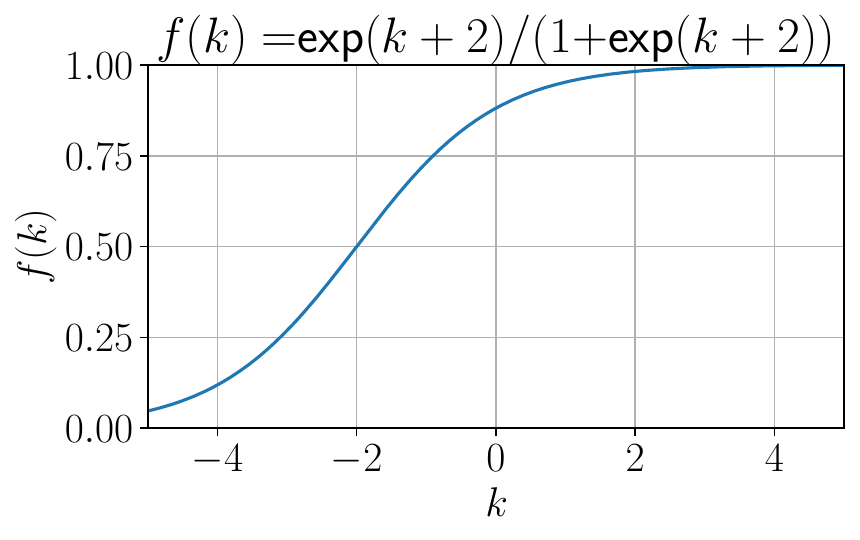}
  \end{center}
  \vspace{-0.1in}
  \caption{
  Illustration of the function $\exp(k+2)/(1+\exp(k+2))$
    }\label{fig:sigmoid}
\end{wrapfigure}
Therefore, when $k$ is small, $r(1,2)= \advm(1,2)+\advw(1,2) \approx\exp(k+2)$, and the upper bound is approximately equal to:
\begin{align}
    4 \E_{\Skx}\left[\frac{\exp(k+2)}{1+\exp(k+2)} \|\vx_{k+1}\|^2\right],
\end{align}
which increases as the number of in-context examples increases.

\subsection{Theorem of Early Ascent}
\label{app:earlyascent:theory}
\newtheorem*{mytheorem}{Theorem~\ref{the:earlyascent}}
\begin{mytheorem}[Early Ascent]
Assume $\E_{\vx_1}\left[\left(\F^*(\vx_1)-\langle \vwt,\vx_1 \rangle\right)^2\right] < \E_{\vx_1}\left[\langle\vx_1, \vw_\alpha-\vwt \rangle^2 \right]$, where $\alpha = \argmin\limits_m 
\frac{
    \|\vm_m-\vmt\|^2
}{2\nx^2}
+
\frac{
    \|(\vw_m-\vwt)^\top\vmt\|^2 + d\tau_x^2\|\vw_m-\vwt\|^2
}{2\ny^2}
$.
Then, when $\dm$ and $\dw$ are small enough, we have the early ascent phenomenon on the risk:
\begin{align}
    \exists k\geq1 \text{ s.t. } 
    \E_{\vx_1}\left[\left(\F^*(\vx_1)-\langle \vwt,\vx_1 \rangle\right)^2\right]
<
    \E_{\Skx}\left[\left(\F^*(\Skx)-\langle \vwt,\vx_{k+1} \rangle\right)^2\right].
\end{align}
\end{mytheorem}
\begin{proof}
We examine the following case, when $\sm$ and $\sw$ are small enough, and $k$ is also big enough to retrieve a task, i.e., making a center dominate:
\begin{align}
&
    \lim_{k\rightarrow  \infty} \lim_{(\sm,\sw) \rightarrow (0,0)} \E_{\Skx}\left[\left(\F^*(\Skx)-\langle \vwt,\vx_{k+1} \rangle\right)^2\right]
\\&=
    \lim_{k\rightarrow  \infty} \lim_{(\sm,\sw) \rightarrow (0,0)} \E_{\Skx}\left[\left\langle \sum\nolimits_{m=1}^M \tpi_m \mA(\vw_m-\vwt),\vx_{k+1}\right\rangle^2\right]
\\&=
    \lim_{k\rightarrow  \infty} \lim_{(\sm,\sw) \rightarrow (0,0)} \E_{\Skx}\left[\left\langle \sum\nolimits_{m=1}^M \tpi_m (\vw_m-\vwt),\vx_{k+1}\right\rangle^2\right]
\\&=
    \lim_{k\rightarrow  \infty} \lim_{(\sm,\sw) \rightarrow (0,0)} \E_{\Skx}\left[\left\langle \frac{\sum_{m=1}^M\pi_m\exp(\advm(m,1)+\advw(m,1)) (\vw_m-\vwt)}{\sum_{m=1}^M\pi_m\exp(\advm(m,1)+\advw(m,1))}, \vx_{k+1} \right\rangle^2\right]
\\&
    (\text{Following Eq.~\ref{advconverge}, we have } \lim_{(\sm,\sw) \rightarrow (0,0)} \advm(m,1)+\advw(m,1)
    =\frac{
        \|\vm_m-\vx_{k+1}\|^2 - \|\vm_1-\vx_{k+1}\|^2
    }{2\nx^2}
\\&
    \quad\quad\quad\quad\quad\quad\quad\quad\quad\quad\quad\quad\quad\quad+ \sum_{i=1}^k
    \left(
        \frac{
        \|\vm_m-\vx_i\|^2
        -
        \|\vm_1-\vx_i\|^2
    }{2\nx^2}
    +
        \frac{
        \|y_i^m-y^*_i\|^2
        -
        \|y_i^1-y^*_i\|^2
    }{2\ny^2}
    \right))
\\&=
    \lim_{k\rightarrow  \infty} \E_{\Skx}\left[
    \left\langle 
    \frac{
        \sum_{m=1}^M\pi_m
        \exp\left(
        \frac{
            \|\vm_m-\vx_{k+1}\|^2
        }{2\nx^2}
        +
        \sum_{i=1}^k
        (
            \frac{
            \|\vm_m-\vx_i\|^2
        }{2\nx^2}
        +
            \frac{
            \|y_i^m-y^*_i\|^2
        }{2\ny^2}
        )
        \right) (\vw_m-\vwt)
    }{
        \sum_{m=1}^M\pi_m\exp\left(\frac{
            \|\vm_m-\vx_{k+1}\|^2
        }{2\nx^2}
        +
        \sum_{i=1}^k
        (
            \frac{
            \|\vm_m-\vx_i\|^2
        }{2\nx^2}
        +
            \frac{
            \|y_i^m-y^*_i\|^2
    }{2\ny^2})\right)
    }
    ,\vx_{k+1}
    \right\rangle^2\right]
\\&=
    \E_{\Skx}[\left\langle \vw_\alpha-\vwt,\vx_{k+1}\right\rangle^2]
\\&=
    \E_{\vx_1}[\left\langle \vw_\alpha-\vwt, \vx_1 \right\rangle^2],
\end{align}
where $\alpha = \argmin\limits_m 
\frac{
    \|\vm_m-\vmt\|^2
}{2\nx^2}
+
\frac{
    \|(\vw_m-\vwt)^\top\vmt\|^2 + d\tau_x^2\|\vw_m-\vwt\|^2
}{2\ny^2}$.
\end{proof}
\section{Proof Tools}
This section introduces the inequalities used in our proofs for Theorems~\ref{the:finegrainedlearning} (finegrained upper bound for ICL risk),~\ref{the:retrieval} (upper bound for ICL with biased labels),~\ref{the:generallearning} (coarse upper bound for ICL risk) and Lemma~\ref{lemma:zeroicl} ((informal) upper bound for zero-shot ICL):
\subsection{Gaussian Tail Bound}
\label{app:gaussian}
If $Z_i \sim \mathcal{N}(0, 1)$, then for $t>0$ we have:
\begin{align}
    P\left(\frac{\sum_{i=1}^k Z_i}{k} > t\right) \leq \exp\left(-\frac{kt^2}{2}\right),\\
    P\left(\frac{\sum_{i=1}^k Z_i}{k} < -t\right) \leq \exp\left(-\frac{kt^2}{2}\right).
\end{align}

\subsection{Chi-squared Tail Bound}
\label{app:chi-square}
If $X \sim\chi(k)$, \ie, $X=\sum_{i=1}^k Z_i^2$ where $Z_i \sim \mathcal{N}(0,1)$ then~\citep{boucheron2013concentration}:
\begin{align}
    P\left(\frac{X}{k}-1 >  2\sqrt{t_1}+2t_1\right) \leq \exp\left(-k t_1^2\right),\\
    P\left(\frac{X}{k}-1 < -2\sqrt{t_1}\right)      \leq \exp\left(-k t_1^2\right).
\end{align}
As a looser but symmetric bound, for any $t>0$, we have:
\begin{align}
    P\left(\frac{X}{k}-1 >  t\right) \leq \exp\left(-\frac{k t^2}{8}\right),\\
    P\left(\frac{X}{k}-1 < -t\right) \leq \exp\left(-\frac{k t^2}{8}\right).
\end{align}

\subsection{Norm Tail Bound}
\label{app:norm}
If $\vep_i \sim \mathcal{N}(\bm{0}, \nxd^2\mI)$, $\vep_i\in\mathbb{R}^d,\mI\in \mathbb{R}^{d\times d}$, then for $t>0$ we have:
\begin{align}
    P\left(\bigg\|\frac{\sum_{i=1}^k\vep_i}{k}\bigg\| > \sqrt{\frac{\nxd^2 d}{k}(1+t)}\right) \leq \exp\left(-\frac{kt^2}{8}\right),
\end{align}
where $\|\cdot\|$ indicates the $L_2$ norm.
\begin{proof}
    \begin{align}
        &\bigg\|\frac{\sum_{i=1}^k\vep_i}{k}\bigg\|^2
        \\&=
        \sum_{j=1}^d \left(\frac{\sum_{i=1}^k\epsilon_{i,j}}{k}\right)^2
        \\&=
        \frac{\nxd^2}{k} \sum_{j=1}^d \left(\frac{\sum_{i=1}^k\epsilon_{i,j}}{\nxd\sqrt{k}}\right)^2
        \\&
        (\text{Notice } \epsilon_{i,j}\sim\mathcal{N}(0,\nxd^2) \text{ and let } Z_j = \frac{\sum_{i=1}^k\epsilon_{i,j}}{\nxd\sqrt{k}} \sim \mathcal{N}(0,1).)
        \\&=
        \frac{\nxd^2 d}{k} \frac{\sum_{i=1}^{d} Z_i^2}{d}.
    \end{align}
    Therefore, by applying Appendix~\ref{app:chi-square} we have:
    \begin{align}
        P\left(\frac{\nxd^2 d}{k} \frac{\sum_{i=1}^{d} Z_i^2}{d} > \frac{\nxd^2 d}{k}(1+t)\right) \leq \exp\left(-\frac{kt^2}{8}\right).
    \end{align}
\end{proof}

\subsection{Eigenvalue Concentration Bound}
\begin{lemma}
\label{lemma:eigenvalue}
    If$~~\forall i$, $\vx_i\sim\mathcal{N}(\vm, \nxd^2\mI)$, $\|\vm\|=1$, $\mA=\frac{\sum_{i=1}^k \vx_i\vx_i^\top}{k}$, and $\vep_i=\vx_i-\vm$, we have $\forall t>0$:
    \begin{align}
        P\left(\L\leq\lambda_d(\mA)\leq\lambda_1(\mA)\leq \U \text{ and } \bigg\|\frac{\sum_{i=1}^k\vep_i}{k}\bigg\| < \nxd\sqrt{\gamma(1+t)}\right)> 1-3\exp\left(-\frac{k t^2}{8}\right),
    \end{align}
where $\L=\nxd^2(1-\frac{t}{2}-\dk)^2-2\nxd\dk\sqrt{1+t}, \U=1+\nxd^2(1+\frac{t}{2}+\dk)^2+2\nxd\dk\sqrt{1+t}$, $\lambda_i(\mA)$ is the $i^\text{th}$ biggest eigenvalue of the matrix $\mA$ and $\dk=\sqrt{\frac{d}{k}}$.
\end{lemma}
We begin with decomposing $\mA$ to three components $\mA=\frac{\sum_{i=1}^k \vx_i\vx_i^\top}{k}=\frac{\sum_{i=1}^k (\vm+\vep_i)(\vm+\vep_i)^\top}{k}=\vm\vm^\top + \frac{\sum_{i=1}^k \vep_i\vep_i^\top}{k} + \frac{\sum_{i=1}^k (\vm\vep_i^\top+\vep_i\vm^\top)}{k}$, then consider the eigenvalue bound of each of them.

For the first component $\vm\vm^\top$, we have:
\begin{align}
    0\leq \lambda_d(\vm\vm^\top) < \lambda_1(\vm\vm^\top) \leq 1.
\end{align}

Then, we analyze the second component $\frac{\sum_{i=1}^k \vep_i\vep_i^\top}{k}$. 
Following \citet[Theorem 4.6.1, p.~97]{vershynin2018high}, we have for any $1-\sqrt{\frac{d}{k}}>s>0$:
\begin{align}
    P\left(\bigg(1-s-\sqrt{\frac{d}{k}}\bigg)^2\leq \frac{1}{\nxd^2}\lambda_d\bigg(\frac{\sum_{i=1}^k \vep_i\vep_i^\top}{k}\bigg) < \frac{1}{\nxd^2}\lambda_1\bigg(\frac{\sum_{i=1}^k \vep_i\vep_i^\top}{k}\bigg) \leq \bigg(1+s+\sqrt{\frac{d}{k}}\bigg)^2\right)>1-2\exp\left(-\frac{k s^2}{2}\right).
\end{align}

Finally, we examine the third component $\frac{\sum_{i=1}^k (\vm\vep_i^\top+\vep_i\vm^\top)}{k}$.
We have for all $\|\va\|=1$:
\begin{align}
    &
    \bigg\|\va^\top \frac{\sum_{i=1}^k (\vm\vep_i^\top+\vep_i\vm^\top)}{k} \va\bigg\|
    =2\bigg\|\va^\top\frac{\sum_{i=1}^k\vep_i}{k}\vm^\top\va\bigg\|
    \leq2\bigg\|\frac{\sum_{i=1}^k \vep_i}{k}\bigg\|
    \\&
    (\text{Notice by Norm Tail Bound in Appendix~\ref{app:norm}, we have }P\left(\bigg\|\frac{\sum_{i=1}^k\vep_i}{k}\bigg\| > \sqrt{\frac{\nxd^2 d}{k}(1+t)}\right) \leq \exp\left(-\frac{kt^2}{8}\right).)
    \\&
    \Longrightarrow P\left(\bigg\|\va^\top \frac{\sum_{i=1}^k (\vm\vep_i^\top+\vep_i\vm^\top)}{k} \va\bigg\| \leq 2\bigg\|\frac{\sum_{i=1}^k \vep_i}{k}\bigg\| \leq 2\sqrt{\frac{\nxd^2 d}{k}(1+t)}\right) > 1-\exp\left(-\frac{kt^2}{8}\right)
    \\&
    \text{\small $\Longrightarrow P\left(-2\nxd\sqrt{\frac{d}{k}(1+t)} \leq \lambda_d\bigg(\frac{\sum_{i=1}^k (\vm\vep_i^\top+\vep_i\vm^\top)}{k}\bigg) \leq \lambda_1\bigg(\frac{\sum_{i=1}^k (\vm\vep_i^\top+\vep_i\vm^\top)}{k}\bigg) \leq 2\nxd\sqrt{\frac{d}{k}(1+t)}\right) > 1-\exp\left(-\frac{kt^2}{8}\right).$}
\end{align}

Let $\dk=\sqrt{\frac{d}{k}}$, $s=t/2$, and summarize three components by union bound, we have:
\begin{align}
    P\left(
    \nxd^2\left(1-\frac{t}{2}-\dk\right)^2-2\nxd\dk\sqrt{1+t}
    \leq\lambda_d(\mA)\leq\lambda_1(\mA)\leq 
    1+\nxd^2\left(1+\frac{t}{2}+\dk\right)^2+2\nxd\dk\sqrt{1+t}
    \right)> 
    1-3\exp\left(-\frac{k t^2}{8}\right).
\end{align}
As a summary, we have:
\begin{align}
    P\left(\L\leq\lambda_d(\mA)\leq\lambda_1(\mA)\leq \U \text{ and } \bigg\|\frac{\sum_{i=1}^k\vep_i}{k}\bigg\| < \nxd\sqrt{\gamma(1+t)}\right)> 1-3\exp\left(-\frac{k t^2}{8}\right),
\end{align}
where $\dk=\sqrt{\frac{d}{k}}$, $\L=\nxd^2(1-\frac{t}{2}-\dk)^2-2\nxd\dk\sqrt{1+t}, \U=1+\nxd^2\left(1+\frac{t}{2}+\dk\right)^2+2\nxd\dk\sqrt{1+t}$, and $\lambda_i(\mA)$ is the $i^\text{th}$ biggest eigenvalue of the matrix $\mA$.

\section{ICL to Learn the In-Context Function}
\label{proof:learning}
This section introduces the proof of Theorem~\ref{the:generallearning} (coarse upper bound for ICL risk) and Theorem~\ref{the:finegrainedlearning} (finegrained upper bound for ICL risk).
The upper bound of Theorem~\ref{the:finegrainedlearning} is derived at Eq.~\ref{equation:fine-grained}.
\begin{proof}
    Assuming we are using in-context examples following Assumption~\ref{asu:source}, \ie, $\vx_i \sim \mathcal{N}(\vmt, \nxd^2\mI), y_i = \langle \vx_i,\vwt \rangle$, $\|\vmt\|=\|\vwt\|=1$, and we aim to have the prediction of $\Skx$ to be $\langle \vx_{k+1},\vwt \rangle$, \ie, to learn the function $(\vwt)$ of the in-context task $(\vmt,\vwt)$.
    Let $\mathcal{L}_k^*$ indicate the squared loss $(\F^*(\Skx)-\langle \vx_{k+1},\vwt \rangle)^2$, where $\F^*(\Skx)$ is the prediction of $\Skx$ by the \bosp $\F^*$ under Assumption~\ref{asu:assumption} for pretraining data generation.
    We derive the upper bound of the expected squared loss as follows:
    \begin{align}
        &
        \E_{\Skx}[\mathcal{L}_k^*]
    \\&=
        \E_{\Skx}\left[\left(\F^*(\Skx)-\langle \vwt,\vx_{k+1} \rangle\right)^2\right]
    \\&
        (\text{By Corollary~\ref{corollary:prediction}}.)
    \\&=
        \E_{\Skx}\left[\left(\sum\nolimits_{m=1}^M \tpi_m \langle \tvw_m,\vx_{k+1} \rangle - \langle \vwt,\vx_{k+1} \rangle\right)^2\right]
    \\&= 
        \E_{\Skx}\left[\left(\left\langle \sum\nolimits_{m=1}^M \tpi_m(\tvw_m-\vwt),\vx_{k+1}\right\rangle \right)^2\right]
    \\&
        (\text{See Eq.~\ref{equation:vwshift} for the derivation of } \tvw_m.)
    \\&=
        \E_{\Skx}\left[\left(\left\langle \sum\nolimits_{m=1}^M \tpi_m ((\mI+k\dw\mCw)^{-1}(\vw_m-\vwt) + \vwt- \vwt),\vx_{k+1}\right\rangle \right)^2\right]
    \\&
        (\text{Let } \mA = (\mI+k\dw\mCw)^{-1}, \text{and notice } \mA \text{ is symmetric positive definite.})
    \\&=
        \E_{\Skx}\left[\left\langle \sum\nolimits_{m=1}^M \tpi_m \mA(\vw_m-\vwt),\vx_{k+1}\right\rangle^2\right]
    \\&
        (\text{Notice } \left(\sum\nolimits_{\beta=1}^M \tpi_\beta a_\beta\right)^2 \leq \sum\nolimits_{\beta=1}^M \tpi_\beta a_\beta^2 \text{, since } \E[a]^2 \leq \E[a^2].)
    \\&\leq
        \E_{\Skx}\left[\sum\nolimits_{m=1}^M \tpi_m \langle \mA(\vw_m-\vwt),\vx_{k+1} \rangle^2\right]
    \\&=
        \sum\nolimits_{m=1}^M \E_{\Skx}\left[\tpi_m ((\vw_m-\vwt)^\top \mA \vx_{k+1})^2\right]
    \\&\leq
        \sum\nolimits_{m=1}^M  \E_{\Skx}\left[\tpi_m \|\vw_m-\vwt\|^2 \lambda_1(\mA)^2 \|\vx_{k+1}\|^2\right]
    \\&=
        \label{equation:fine-grained}
        \sum\nolimits_{m=1}^M \|\vw_m-\vwt\|^2 \E_{\Skx}\left[\tpi_m \|\vx_{k+1}\|^2 \lambda_1(\mA)^2\right]
    \\&\leq
         4 \E_{\Skx}\left[\sum\nolimits_{m=1}^M\tpi_m \|\vx_{k+1}\|^2 \lambda_1(\mA)^2\right]
    \\&=
        4 \E_{\Skx}\left[\|\vx_{k+1}\|^2 \lambda_1(\mA)^2\right]
    \\&
        (\text{Notice } \mA \text{ is a random matrix only depends on } \vx_1,\vx_2,\ldots,\vx_k, \text{ but not } \vx_{k+1}.)
    \\&=
        4 \E_{\vx_{k+1}}\left[\|\vx_{k+1}\|^2\right] \E_{\Sk}\left[\lambda_1^2(\mA)\right]
    \\&=
        4(1 + d\nxd^2) \E_{\Sk}\left[\lambda_1^2(\mA)\right].
    \end{align}
    We further simplify $\E_{\Sk}\left[\lambda_1^2(\mA)\right]$ using Lemma~\ref{lemma:eigenvalue}:
    \begin{align}
        &
        \E_{\Skx}[\mathcal{L}_k^*] 
        \\&\leq
        4(1+d\nxd^2) \E_{\Sk}\left[\lambda_1^2(\mA)\right]
        \\&\leq
        4(1+d\nxd^2) \E_{\Sk}\left[\left(\frac{1}{1+k\dw\lambda_d(\frac{\sum_{i=1}^k \vx_i\vx_i^\top}{k})}\right)^2\right]
        \\&
        (\text{By applying Lemma}~\ref{lemma:eigenvalue} \text{ to } \frac{\sum_{i=1}^k \vx_i\vx_i^\top}{k}.)
        \\&\leq
        4(1+d\nxd^2) \E_{\Sk}\left[\left(\frac{1}{1+k\dw\L}\right)^2\right]
        \\&\leq
        4(1+d\nxd^2) \left(\left(\frac{1}{1+k\dw (\nxd^2(1-\frac{t}{2}-\dk)^2-2\nxd\dk\sqrt{1+t})}\right)^2 + 3\exp\left(-\frac{k t^2}{8}\right)\right).
    \end{align}
    Let $t=k^{\delta-\frac{1}{2}}$, where $\frac{1}{2}>\delta>0$ and $\delta$ is arbitrary small. We have:
    \begin{align}
        \E_{\Skx}[\mathcal{L}_k^*] 
        <
        \frac{4(1+d\nxd^2)}{\nxd^4 \dw^2 k^2}
        +
        O(k^{\delta-\frac{5}{2}})
        .
    \end{align}
\end{proof}
\begin{figure}[th!]
    \centering
    \includegraphics[width = 0.9\textwidth]{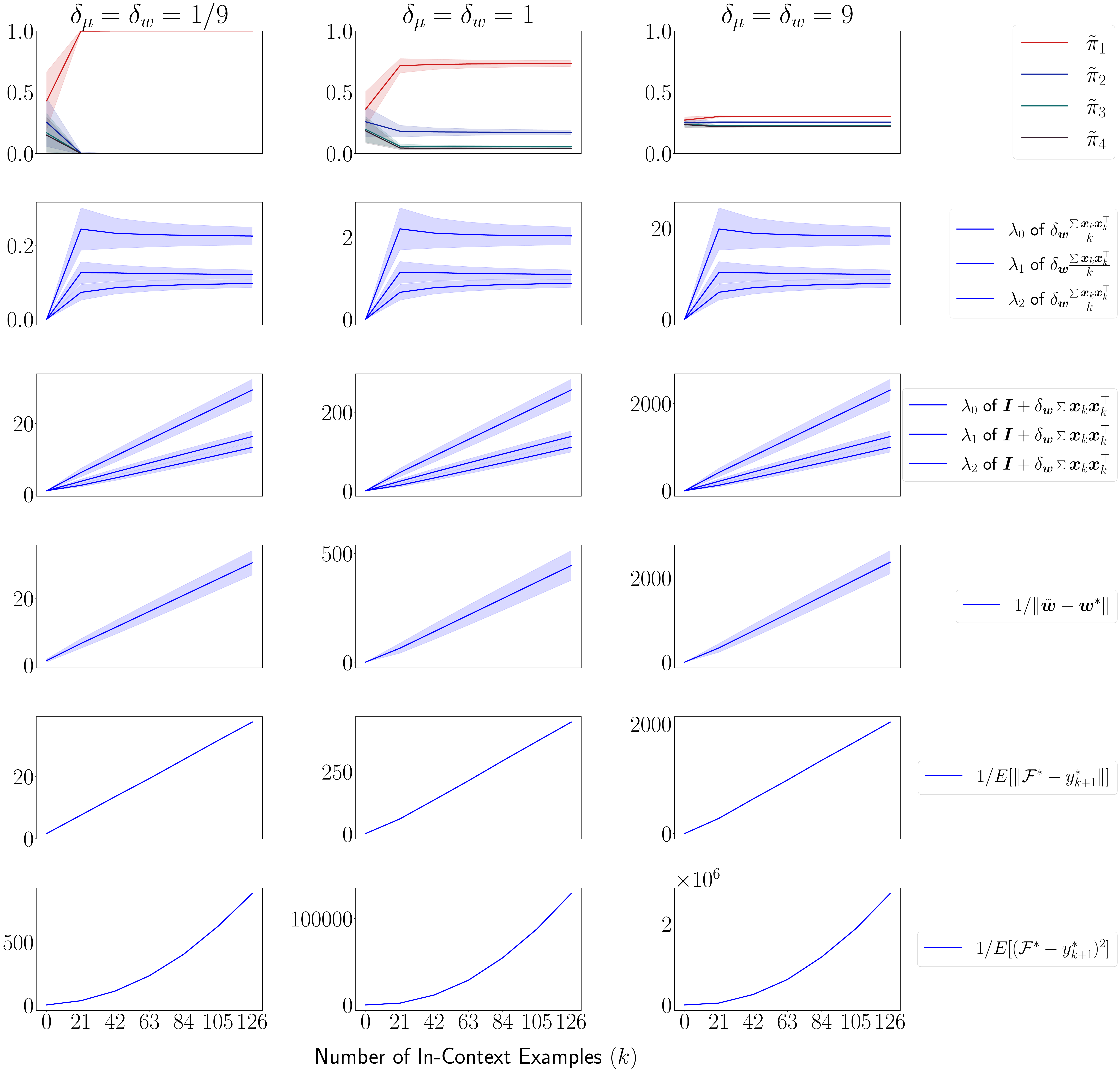}
    \caption{The numerical computation of the task learning.
    The second and third rows show the eigenvalues of the matrices $\dw\frac{\sum_{i=1}^k \vx_i\vx_i^\top}{k}$ and $\mI+\dw\sum_{i=1}^k \vx_i\vx_i^\top$.
    The fourth row shows the distance between the predicted $\tilde{\vw}$ and $\vwt$ has a reciprocal decreasing rate with respect to $k$.
    The fifth and sixth rows indicate the expected squared loss follows a quadratic decreasing rate with respect to $k$.}
    \label{fig:learning}
\end{figure}
We further validate our analysis with numerical computations in Fig.~\ref{fig:learning}, including the trend of $\tpi_m$ for $m\in[M]$, $\lambda_j\left(\dw\frac{\sum_{i=1}^k \vx_i \vx_i^\top}{k}\right)$ for $j\in[d]$, $\lambda_j\left(\mI+\dw \sum_{i=1}^k \vx_i \vx_i^\top\right)$ for $j\in[d]$, $1/\|\tvw-\vwt\|$, $1/\E[\sF(\Skx)-y^*_{k+1}]$, and $1/\E[(\sF(\Skx)-y^*_{k+1})^2]$ as $k$ increases.

\subsection{Case When In-context Input Variable Spans in Subspace}
\label{app:lowrank}
In this section, we refine Eq.~\ref{equation:fine-grained} for the finegrained bound in Theorem~\ref{the:finegrainedlearning}.
Specifically, we refine the following inequality for case when in-context input variable $\vx_i$ only spans in the subspace of $\mathbb{R}^d$, resulting in $\lambda_1(\mA)=1$ constantly as mentioend in Theorem~\ref{the:finegrainedlearning}:
\begin{align}
&
    \sum\nolimits_{m=1}^M \E_{\Skx}\left[\tpi_m ((\vw_m-\vwt)^\top \mA \vx_{k+1})^2\right]
\\&\leq
    \sum\nolimits_{m=1}^M  \E_{\Skx}\left[\tpi_m \|\vw_m-\vwt\|^2 \lambda_1(\mA)^2 \|\vx_{k+1}\|^2\right],
\end{align}
where $\mA = (\mI + \sum_{i=1}^k \vx_i\vx_i^\top)^{-1}$ is derived in Lemma~\ref{lemma:posterior}.
Violating Assumption~\ref{asu:source1}, in this section we consider the case that $\vx_i\sim \mathcal{N}(\vm,\diag(\underbrace{1,\ldots,1}_{d'},0,\ldots,0))$, where $\vm=[p,\underbrace{0,\ldots,0}_{d'-1},q,0,\ldots,0]^\top$.
(If $\vm$ does not follows the format $ [p,\underbrace{0,\ldots,0}_{d'-1},q,0,\ldots,0]^\top$, we can always rotate the coordinates so $\vm$ has this format.)
Therefore, we have matrix $\mA$ (after rotation) with the following format:
\begin{align}
    \mA=
    \begin{cases}
        \begin{bmatrix}
        \mI_{d'\times d'}+\sum_{i=1}^k \vx_{i,1:d'}\vx_{i,1:d'}^\top & \mzero_{d' \times (d-d')} \\
        \mzero_{(d-d')\times d'} & \mI_{(d-d')\times (d-d')}
        \end{bmatrix}^{-1}, \text{ if $q=0$}\\
        \begin{bmatrix}
        \mI_{(d'+1)\times (d'+1)}+\sum_{i=1}^k \vx_{i,1:(d'+1)}\vx_{i,1:(d'+1)}^\top & \mzero_{(d'+1) \times (d-d'-1)} \\
        \mzero_{(d-d'-1)\times (d'+1)} & \mI_{(d-d'-1) \times (d-d'-1)}
        \end{bmatrix}^{-1}, \text{ if $q>0$}\\
    \end{cases}
\end{align}
where $\vx_{i,1:d'}=[\vx_{i,1},\vx_{i,2},\ldots,\vx_{i,d'}]^\top$, $\mI_{a\times a}$ indicates an identity matrix with shape $a$ by $a$, and $\mzero_{a \times b}$ indicates a zero matrix with shape $a$ by $b$.
Finally, we can revise the upper bound for the case when $\vx_i$ only spans in a subspace of $\mathbb{R}^d$ using the new format of $\mA$ as follows:

When $q=0$, we have:
\begin{align}
&
    \sum\nolimits_{m=1}^M \E_{\Skx}\left[\tpi_m ((\vw_m-\vwt)^\top \mA \vx_{k+1})^2\right]
\\&\leq
    \sum\nolimits_{m=1}^M \E_{\Skx}\left[\tpi_m ((\vw_m-\vwt)_{1:d'}^\top \mA_{1:d',1:d'} \vx_{k+1,1:d'} + (\vw_m-\vwt)_{(d'+1):d}^\top \mI_{(d-d')\times (d-d')} \vx_{k+1,(d'+1):d})^2\right]
\\&\leq
    \sum\nolimits_{m=1}^M  \E_{\Skx}\left[\tpi_m 
    (
        \|(\vw_m-\vwt)_{1:d'}\|^2 \lambda_1(\mA_{1:d',1:d'})^2 \|\vx_{k+1,1:d'}\|^2
        +
        \|(\vw_m-\vwt)_{(d'+1):d}\|^2 \|\vx_{k+1,(d'+1):d}\|^2
    )\right], 
\\&
    (\text{Notice } \|\vx_{k+1,(d'+1):d}\|^2=0)
\\&=
    \sum\nolimits_{m=1}^M  \E_{\Skx}\left[\tpi_m 
        \|(\vw_m-\vwt)_{1:d'}\|^2 \lambda_1(\mA_{1:d',1:d'})^2 \|\vx_{k+1,1:d'}\|^2
    \right], 
\end{align}
When $q>0$, we skip the analysis since the analysis for $q>0$ is the same as the analysis for $q=0$. The only difference is that $d'$ for $q>0$ is one bigger than $d'$ for $q=0$.
\begin{figure}[th!]
    \centering
    \includegraphics[width = 1.0\textwidth]{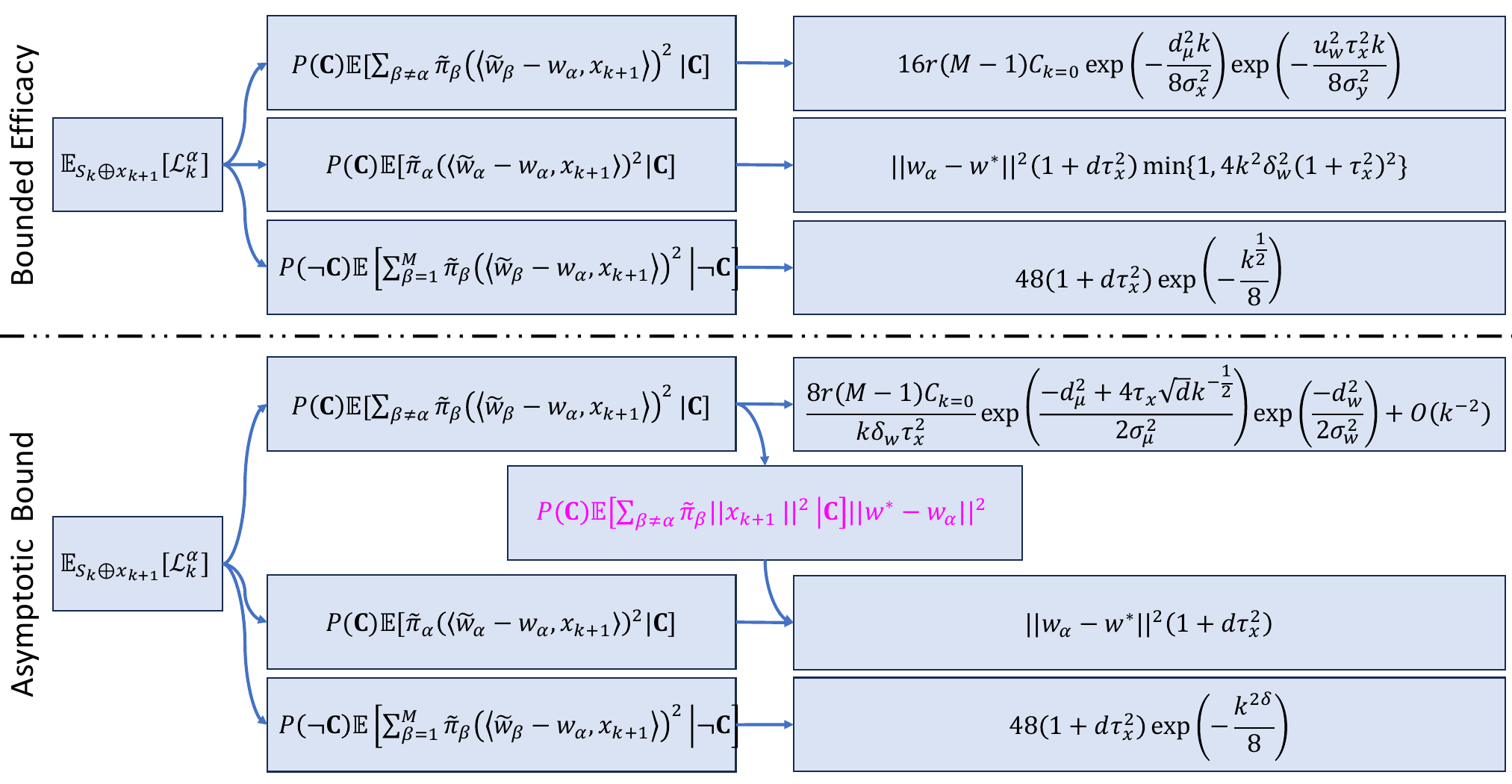}
    \caption{Proof roadmap of ICL with biased labels, Theorem.~\ref{the:retrieval}.}
    \label{fig:proofroadmap}
\end{figure}

\section{ICL with Biased Labels to Retrieve A Task}
\label{app:retrieval}
This section details the proof of Theorem~\ref{the:retrieval}, with Fig.\ref{fig:proofroadmap} serving as a visual guide.
The non-asymptotic bound for the bounded efficacy phenomenon and the asymptotic bound share the same foundational elements in the proof.
However, they are different in handling the components marked in pink. Fig.~\ref{fig:proofroadmap} is thus provided to offer a clearer understanding of its overall framework and assist readers in navigating through the proof.
In the following sections, Sec.~\ref{app:proof:nonasymptotic} introduces the non-asymptotic bound revealing the bounded efficacy phenomenon, and Sec.~\ref{app:proof:asymptotic} introduces the asymptotic bound.

\subsection{Non-Asymptotic Bound for the Bounded Efficacy Phenomenon}
\label{app:proof:nonasymptotic}
This section proves the non-asymptotic bound in Theorem~\ref{the:retrieval}: Consider a \sp attaining the optimal pretraining risk.
When $\dm$ and $\dw$ are sufficiently small, there exists a particular interval (refer to Sec.\ref{app:proof:interval} for the interval) for $k$ such that ICL risk with biased labels is upper bounded by:
\begin{align}
    \E_{\Sk}[\mathcal{L}_k^\alpha]
&<
    C_3
    \exp\left(-\redb{k}\left(\frac{
        d_\vm^2
    }{
        8\nx^2
    }+\frac{
        u_\vw^2\nxd^2
    }{
        8\ny^2
    }\right)\right)
    +
    48(1+d\nxd^2)\exp\left(-\frac{\redb{k}^{\frac{1}{2}}}{8}\right)
\\&~~+
    \|\vw_\alpha-\vwt\|^2 (1+d\nxd^2)\min\{1,4\redb{k}^2\blueb{\dw}^2(1+\nxd^2)^2\}.
\end{align}
where $\mathcal{L}_k^\alpha=(\F(\Skx)-y_{k+1}^\alpha)^2=(\F(\Skx)-\langle \vx_{k+1},\vw_\alpha \rangle)^2$ $C_3$ is a constant depending on the prior setting, $\nxd$, and $(\vmut,\vwt)$.
With small $k$, the first and second terms dominate and exponential decay.
With large $k$, the third term dominates and increases.
Thus, the upper bound reveals a bounded efficacy phenomenon.
\begin{proof}
Assuming we are using in-context examples following Assumptions~\ref{asu:source} and~\ref{asu:misaligned}, \ie, $\vx_i \sim \mathcal{N}(\vmt, \nxd^2\mI), y_i = \langle \vx_i,\vwt \rangle$, $\|\vmt\|=\|\vwt\|=1$, and we aim to retrieve the function $\vw_\alpha$ of the prior center $(\vm_\alpha,\vw_\alpha)$ which is close to the in-context task.
Let $\mathcal{L}_k^\alpha$ indicate the squared risk $(\F^*(\Skx)-\langle \vx_{k+1},\vw_\alpha \rangle)^2$, where $\F^*(\Skx)$ is the prediction of $\Skx$ by the \bosp $\F^*$.
In order to have an upper bound on the risk, we consider $\vx_i\sim \mathcal{N}(\vmt,\nxd^2\mI)$ in two cases: (1) $\textbf{C}$:
$\L < \lambda_d\left(\frac{\sum_{i=1}^k \vx_i\vx_i^\top}{k}\right) \leq \lambda_1\left(\frac{\sum_{i=1}^k \vx_i\vx_i^\top}{k}\right) < \U \text{ and } \left\|\frac{\sum_{i=1}^k\vep_i}{k}\right\| < \nxd\sqrt{\gamma(1+t)}$
(see Lemma~\ref{lemma:eigenvalue} for $t$, $\gamma$, $\L$ and $\U$) and (2) $\neg\textbf{C}$: at least one of the previous inequalities does not hold.
Following Lemma~\ref{lemma:eigenvalue}, the probability of $\neg\textbf{C}$ is bounded by: $P(\neg\textbf{C}) \leq 3\exp(-\frac{k t^2}{8})$).

We start our upper bound analysis on the expected squared risk by splitting the risk into three parts:
\begin{align}
&
    \E_{\Skx}[\mathcal{L}_k^\alpha]
\\&=
    \E_{\Skx}\left[(\F^*(\Skx)-\langle \vw_\alpha,\vx_{k+1} \rangle)^2\right]
\\&
    (\text{By Corollary~\ref{corollary:prediction}}.)
\\&=            
    \E_{\Skx}\left[\left(\sum\nolimits_{\beta=1}^M \tpi_\beta \langle \tvw_\beta,\vx_{k+1} \rangle - \langle \vw_\alpha,\vx_{k+1} \rangle\right)^2\right]
\\&
    (\text{Notice } \sum\nolimits_{\beta=1}^M \tpi_\beta=1.)
\\&=
    \E_{\Skx}\left[\left(\sum\nolimits_{\beta=1}^M \tpi_\beta \left(\langle \tvw_\beta,\vx_{k+1} \rangle - \langle \vw_\alpha,\vx_{k+1} \rangle \right)\right)^2\right]
\\&
    (\text{Notice } \left(\sum\nolimits_{\beta=1}^M \tpi_\beta a_\beta\right)^2 \leq \sum\nolimits_{\beta=1}^M \tpi_\beta a_\beta^2 \text{, since } \E[a]^2 \leq \E[a^2].)
\\&\leq
    \E_{\Skx}\left[\sum\nolimits_{\beta=1}^M \tpi_\beta (\langle \tvw_\beta,\vx_{k+1} \rangle - \langle \vw_\alpha,\vx_{k+1} \rangle)^2\right]
\\&=
    \E_{\Skx}\left[\sum\nolimits_{\beta=1}^M \tpi_\beta \langle \tvw_\beta- \vw_\alpha,\vx_{k+1} \rangle^2\right]
\\&=
    P(\textbf{C}) \E_{\Skx}\left[\sum\nolimits_{\beta=1}^M \tpi_\beta \langle \tvw_\beta- \vw_\alpha,\vx_{k+1} \rangle^2 \middle\vert \textbf{C}\right]
    \\&~~~~~
    +
    P(\neg\textbf{C}) \E_{\Skx}\left[\sum\nolimits_{\beta=1}^M \tpi_\beta \langle \tvw_\beta- \vw_\alpha,\vx_{k+1} \rangle^2 \middle\vert \neg\textbf{C}\right]
\\&=
    \label{equation:part1} P(\textbf{C})\E_{\Skx}\left[\sum\nolimits_{\beta\neq\alpha} \tpi_\beta \langle \tvw_\beta- \vw_\alpha,\vx_{k+1} \rangle^2 \middle\vert \textbf{C}\right] \tag{Part $A$}
    \\&~~~~~
    \label{equation:part2}
    +
    P(\textbf{C})\E_{\Skx}[\tpi_\alpha \langle \tvw_\alpha- \vw_\alpha,\vx_{k+1} \rangle^2 \vert \textbf{C}] \tag{Part $B$}
    \\&~~~~~
    \label{equation:part3}
    +
    P(\neg\textbf{C}) \E_{\Skx}\left[\sum\nolimits_{\beta=1}^M \tpi_\beta \langle \tvw_\beta- \vw_\alpha,\vx_{k+1} \rangle^2 \middle\vert \neg\textbf{C}\right]. \tag{Part $C$}
\end{align}
We will analyze three parts one by one in the following three sections respectively.
\end{proof}

\subsubsection{Bounded Efficacy - \ref{equation:part1}}
\label{sec:retrieval:part1}
\begin{proof}
We firstly analyze the term $P(\textbf{C})\E_{\Skx}[\sum_{\beta\neq\alpha} \tpi_\beta \langle \tvw_\beta- \vw_\alpha,\vx_{k+1} \rangle^2\vert\textbf{C}]$, \ref{equation:part1}:
\begin{align}
&
    P(\textbf{C})\E_{\Skx}\left[\sum\nolimits_{\beta\neq\alpha} \tpi_\beta \langle \tvw_\beta- \vw_\alpha,\vx_{k+1} \rangle^2 \middle\vert \textbf{C}\right]
\\&< 
    P(\textbf{C}) \E_{\Skx}\left[\sum\nolimits_{\beta\neq\alpha} \tpi_\beta \|\tvw_\beta- \vw_\alpha\|^2 \|\vx_{k+1}\|^2 \middle\vert \textbf{C}\right]
\\&
    (\text{See Eq.~\ref{equation:vwshift} for the derivation of } \tvw_\beta.)
\\&=
    P(\textbf{C})\E_{\Skx}\left[\sum\nolimits_{\beta\neq\alpha} \tpi_\beta \|(\mI+k\dw\mCw)^{-1}(\vw_\beta-\vwt) +\vwt - \vw_\alpha\|^2 \|\vx_{k+1}\|^2 \middle\vert \textbf{C}\right]
\\&
    (\text{Let } \mA=(\mI+k\dw\mCw)^{-1}, \text{ and } \lambda_1(\mA) \text{ is the largest eigenvalue of matrix } \mA.)
\\&=
    P(\textbf{C})\E_{\Skx}\left[\sum\nolimits_{\beta\neq\alpha} \tpi_\beta \|\mA(\vw_\beta-\vwt) +\vwt - \vw_\alpha\|^2 \|\vx_{k+1}\|^2 \middle\vert \textbf{C}\right]
\\&\leq
    P(\textbf{C})\E_{\Skx}\left[\sum\nolimits_{\beta\neq\alpha} \tpi_\beta (\|\mA(\vw_\beta-\vwt)\|+\|\vwt - \vw_\alpha\|)^2 \|\vx_{k+1}\|^2 \middle\vert \textbf{C}\right]
\\&
    (\text{Notice } \|\vw_\beta-\vwt\|\leq2.)
\\&\leq
    P(\textbf{C})\E_{\Skx}\left[\sum\nolimits_{\beta\neq\alpha} \tpi_\beta \|\vx_{k+1}\|^2 (2\lambda_1(\mA)+\|\vwt - \vw_\alpha\|)^2 \middle\vert \textbf{C}\right]
\\&
    (\text{Notice } \mA=(\mI+k\dw\mCw)^{-1} \text{ and conditioned on \textbf{C} we have } \L<\lambda_d(\mCw)<\lambda_1(\mCw)<\U.)
\\&\leq
    P(\textbf{C})\E_{\Skx}\left[\sum\nolimits_{\beta\neq\alpha} \tpi_\beta \|\vx_{k+1}\|^2 \middle\vert \textbf{C}\right] \left(\frac{2}{1+k\dw \L}+\|\vwt-\vw_\alpha\|\right)^2
\\&
    (\text{Notice } \|\vwt-\vw_\alpha\|\leq2.)
\\&\leq
    16 P(\textbf{C})\E_{\Skx}\left[\sum\nolimits_{\beta\neq\alpha} \frac{\tpi_\beta}{\tpi_\alpha} \|\vx_{k+1}\|^2 \middle\vert \textbf{C}\right].
\\&
    (\text{By applying Eqs.~\ref{equation:ratio},~\ref{re-weight:vmu},~\ref{re-weight:vw}, and Assumption~\ref{asu:furtherasu3} on } \frac{\tpi_\beta}{\tpi_\alpha}:)
\\&< 
    16 P(\textbf{C})\E_{\Skx}\Bigg[\sum\nolimits_{\beta\neq\alpha} r
    \exp\bigg(\frac{
        -\sum_{i=1}^{k+1} \|\vm_\beta -\vx_i\|^2
        +\sum_{i=1}^{k+1} \|\vm_\alpha-\vx_i\|^2
    }{
        2\nx^2(1+(k+1)\dm)
    }\bigg)
    \\&~~~~~
    \cdot
    \exp\bigg(\frac{
        -\|\vw_\beta -\vwt\|^2_{\mI-(\mI+k\dw\mCw)^{-1}}
        +\|\vw_\alpha-\vwt\|^2_{\mI-(\mI+k\dw\mCw)^{-1}}
    }{
        2\nw^2
    }\bigg)
    \|\vx_{k+1}\|^2 \Bigg\vert \textbf{C}\Bigg]
\\&
    (\text{In the first exponential term, by splitting } \sum\nolimits_{i=1}^{k+1} \text{ to } \sum\nolimits_{i=1}^k \text{ and } i=k+1:)
\\&< 
    16 P(\textbf{C})\E_{\Skx}\Bigg[\sum\nolimits_{\beta\neq\alpha} r
    \underbrace{\exp\bigg(\frac{
        -\sum_{i=1}^k \|\vm_\beta -\vx_i\|^2
        +\sum_{i=1}^k \|\vm_\alpha-\vx_i\|^2
    }{
        2\nx^2(1+(k+1)\dm)
    }\bigg)}_{\text{Part $A$-$1$}}
    \\&~~~~~
    \cdot
    \underbrace{\exp\bigg(\frac{
        -\|\vw_\beta -\vwt\|^2_{\mI-(\mI+k\dw\mCw)^{-1}}
        +\|\vw_\alpha-\vwt\|^2_{\mI-(\mI+k\dw\mCw)^{-1}}
    }{
        2\nw^2
    }\bigg)}_{\text{Part $A$-$2$}}
    \\&~~~~~
    \cdot
    \underbrace{\exp\bigg(\frac{
        -\|\vm_\beta -\vx_{k+1}\|^2
        +\|\vm_\alpha-\vx_{k+1}\|^2
    }{
        2\nx^2(1+(k+1)\dm)
    }\bigg)\|\vx_{k+1}\|^2}_{\text{Part $A$-$3$}}
    \Bigg\vert \textbf{C}\Bigg]
\\&
    (\text{Note that } \vx_1,\ldots,\vx_k \text{ are dependent on } \textbf{C} \text{ but } \vx_{k+1} \text{ is not. Thus, we split them for further analysis.})
\end{align}

In the following, we separately analyze the three terms, Part $A$-$1$, Part $A$-$2$, and Part $A$-$3$.
The high-level idea is that, as $k$ increases, due to the concentration of Part $A$-$1$ and Part $A$-$2$, they can be upper bounded by a function of $k$.
Then, regarding Part $A$-$1$ and Part $A$-$2$ as constant values (their upper bounds), the expectation of Part $A$-$3$ can be upper bounded.

\paragraph{Part $A$-$1$.} We first deal with Part $A$-$1$. 
When conditioned on case $\textbf{C}$, we have:
\begin{align}
&
    \frac{
    \sum_{i=1}^k(
        -\|\vm_\beta -\vx_i\|^2 
        +\|\vm_\alpha-\vx_i\|^2
    )}{
    1+(k+1)\dm
    }
\\&
    (\text{Let } \vx_i = \vmt+\vep_i)
\\&=
    k\frac{
        \|\vm_\alpha-\vmt\|^2
        -\|\vm_\beta-\vmt\|^2
    +
        \frac{
            \sum_{i=1}^k 2\left\langle\vm_\beta-\vm_\alpha,\vep_i\right\rangle
        }{k}
    }{
    1+(k+1)\dm
    }
\\&=
    k\frac{
        \|\vm_\alpha-\vmt\|^2
        -\|\vm_\beta-\vmt\|^2
    +
        \left\langle
            2(\vm_\beta-\vm_\alpha),
            \frac{\sum_{i=1}^k \vep_i}{k}
        \right\rangle
    }{
    1+(k+1)\dm
    }
\\&\leq
    k\frac{
        \|\vm_\alpha-\vmt\|^2
        -\|\vm_\beta-\vmt\|^2
    +
        2
        \|
            \vm_\beta-\vm_\alpha
        \|
        \left\|
            \frac{\sum_{i=1}^k \vep_i}{k}
        \right\|
    }{
    1+(k+1)\dm
    }
\\&
    (\text{Recall we have } \forall \beta\in[M], \|\vm_\beta-\vm_\alpha\|\leq2, \text{ and in case $\textbf{C}$ we have: } \bigg\|\frac{\sum_{i=1}^k \vep_i}{k}\bigg\| < \nxd\gamma\sqrt{1+t}.)
\\&<
    k\frac{
        \|\vm_\alpha-\vmt\|^2
        -\|\vm_\beta-\vmt\|^2
        +4\nxd\gamma\sqrt{1+t}
    }{
    1+(k+1)\dm
    }.
\end{align}
Let $t=k^{-\frac{1}{4}}$. Recall in Assumption~\ref{asu:misaligned}, we have $\forall \beta \neq \alpha, \|\vm_\beta-\vmt\|^2 - \|\vm_\alpha-\vmt\|^2 \geq d^2_\vm$.
If $\dm \ll 1$ s.t. $I_\vm = \{k\vert (k+1)\dm\leq 1 \text{ and } \frac{d_\vm^2}{2} > 4\nxd\gamma\sqrt{1+k^{-\frac{1}{4}}} \}\neq \varnothing$, then when $k\in I_\vm$ we have:
\begin{align}
     k\frac{
        \|\vm_\alpha-\vmt\|^2
        -\|\vm_\beta-\vmt\|^2
        +4\nxd\gamma\sqrt{1+t}
    }{
    1+(k+1)\dm
    }
<
     k\frac{
        \|\vm_\alpha-\vmt\|^2
        -\|\vm_\beta-\vmt\|^2
        +\frac{d_\vm^2}{2}
    }{
    2
    }
=
    -k\frac{
        d_\vm^2
    }{4}.
\end{align}

\paragraph{Part $A$-$2$.} We then deal with Part $A$-$2$.
When conditioned on case $\textbf{C}$, we have:
\begin{align}
&
    -\|\vw_\beta-\vwt\|^2_{\mI-(\mI+k\dw\mCw)^{-1}}
    +\|\vw_\alpha-\vwt\|^2_{\mI-(\mI+k\dw\mCw)^{-1}}
\\&
    (\lambda_1(\mA) \text{ and } \lambda_d(\mA) \text{ indicate the largest and smallest eigenvalues of the matrix } \mA\in\mathbb{R}^{d\times d}.)
\\&< 
    -\|\vw_\beta-\vwt\|^2 \lambda_d(\mI-(\mI+k\dw\mCw)^{-1})
    +\|\vw_\alpha-\vwt\|^2 \lambda_1(\mI-(\mI+k\dw\mCw)^{-1})
\\&
    (\text{Recall in case \textbf{C} we have: }
    \L<\lambda_d(\mCw)<\lambda_1(\mCw)<\U.)
\\&< 
    -\|\vw_\beta-\vwt\|^2 \left(1-\frac{1}{1+k\dw\L}\right)
    +\|\vw_\alpha-\vwt\|^2 \left(1-\frac{1}{1+k\dw\U}\right)
\\&=
    -\|\vw_\beta-\vwt\|^2 \frac{k\dw\L}{1+k\dw\L}
    +\|\vw_\alpha-\vwt\|^2 \frac{k\dw\U}{1+k\dw\U}
\\&<
    -\|\vw_\beta-\vwt\|^2 \frac{k\dw\L}{1+k\dw\nxd^2}
    +\|\vw_\alpha-\vwt\|^2 \frac{k\dw\U}{1+k\dw\nxd^2}
\end{align}
Let $t=k^{-\frac{1}{4}}$.
If $\dw \ll 1$ s.t. $I_\vw = \{k \vert k\dw \nxd^2\leq 1 \text{ and } \L\|\vw_\beta-\vwt\|^2-\U\|\vw_\alpha-\vwt\|^2 > \frac{\nxd^2 u_\vw^2}{2}\} \neq \varnothing$, (note $\lim_{k\rightarrow\infty} \L\|\vw_\beta-\vwt\|^2-\U\|\vw_\alpha-\vwt\|^2 = \nxd^2\|\vw_\beta-\vwt\|^2-(1+\nxd^2)\|\vw_\alpha-\vwt\|^2 \geq \nxd^2 u_\vw^2$) then when $k\in I_\vw$, we have:
\begin{align}
    -\|\vw_\beta-\vwt\|^2 \frac{k\dw\L}{1+k\dw\nxd^2}
    +\|\vw_\alpha-\vwt\|^2 \frac{k\dw\U}{1+k\dw\nxd^2}
<
    -\frac{\nxd^2 u_\vw^2}{2} \frac{k\dw}{1+k\dw\nxd^2}
<
    -k\dw \frac{\nxd^2 u_\vw^2}{4}.
\end{align}

\paragraph{Part $A$-$3$.} We finally deal with Part $A$-$3$.
Part $A$-$3$ is independent to case \textbf{C}, and we have:
\begin{align}
&
    P(\textbf{C})\E_{\Skx}\left[\exp\left(\frac{
        -\|\vm_\beta -\vx_{k+1}\|^2
        +\|\vm_\alpha-\vx_{k+1}\|^2
    }{
        2\nx^2(1+(k+1)\dm)
    }\right)\|\vx_{k+1}\|^2 \bigg\vert\textbf{C}\right]
\\&<
    \E_{\Skx}\left[\exp\left(\frac{
        -\|\vm_\beta -\vx_{k+1}\|^2
        +\|\vm_\alpha-\vx_{k+1}\|^2
    }{
        2\nx^2(1+(k+1)\dm)
    }\right)\|\vx_{k+1}\|^2\right]
\\&
    (\text{Let } \vx_{k+1}=\vmt+\vep.)
\\&=
    \E_{\Skx}\left[\exp\left(\frac{
        -\|\vm_\beta -\vmt-\vep\|^2
        +\|\vm_\alpha-\vmt-\vep\|^2
    }{
        2\nx^2(1+(k+1)\dm)
    }\right)\|\vx_{k+1}\|^2\right]
\\&=
    \E_{\Skx}\left[\exp\left(\frac{
        -\|\vm_\beta -\vmt\|^2
        +\|\vm_\alpha-\vmt\|^2
        +\langle 2(\vm_\beta-\vm_\alpha), \vep \rangle
    }{
        2\nx^2(1+(k+1)\dm)
    }\right)\|\vx_{k+1}\|^2\right]
\\&
    (\text{Let } -\|\vm_\beta -\vmt\|^2
        +\|\vm_\alpha-\vmt\|^2 = -D, 2\nx^2(1+(k+1)\dm)=E, \vb=2(\vm_\beta-\vm_\alpha).)
\\&=
    \E_{\Skx}\left[\exp\left(\frac{
        -D+\vb^\top\vep
    }{
        E
    }\right)\|\vx_{k+1}\|^2\right]
\\&
    (\text{Notice } \|\vx_{k+1}\|^2=\|\vmt+\vep\|^2\leq 2\|\vmt\|^2+2\|\vep\|^2.)
\\&\leq
    \E_{\Skx}\left[\exp\left(\frac{
        -D+\vb^\top\vep
    }{
        E
    }\right)(2\|\vmt\|^2 + 2\|\vep\|^2)\right]
\\&
    (\text{Notice } \|\vmt+\vep\|^2=1.)
\\&=
    2\Bigg(
    \E_{\Skx}\left[\exp\left(\frac{
        -D+\vb^\top\vep
    }{
        E
    }\right)\right]
    +
    \E_{\Skx}\left[\exp\left(\frac{
        -D+\vb^\top\vep
    }{
        E
    }\right)\|\vep\|^2\right]
    \Bigg)
\\&=
    2\Bigg(
    \exp\left(\frac{\nxd^2 \|\vb\|^2}{2E^2}-\frac{D}{E}\right)
    +
    \E_{\Skx}\left[\exp\left(\frac{
        -D+\vb^\top\vep
    }{
        E
    }\right)\|\vep\|^2\right]
    \Bigg)
\\&=
    2\Bigg(
        \exp\left(\frac{\nxd^2 \|\vb\|^2}{2E^2}-\frac{D}{E}\right)
        +
        \nxd^2\left(1+\frac{\nxd^2\|\vb\|^2}{E^2}\right)\exp\left(\frac{\nxd^2\|\vb\|^2}{2E^2}-\frac{D}{E}\right) 
        + 
        (d-1)\nxd^2\exp\left(\frac{\nxd^2 \|\vb\|^2}{2E^2}-\frac{D}{E}\right)
    \Bigg)
\\&=
    2\Bigg(1+\nxd^2\left(d+\frac{\nxd^2\|\vb\|^2}{E^2}\right)\Bigg)
    \exp\left(\frac{\nxd^2 \|\vb\|^2}{2E^2}-\frac{D}{E}\right)
\\&=
    \label{equation:constant}
    C_{k=0}.
\end{align}

\paragraph{Summary of \ref{equation:part1}.} 
Thus, summarizing Part $A$-$1$, Part $A$-$2$, and Part $A$-$3$, we have:
\begin{align}
&
    P(\textbf{C})\E_{\Skx}\left[\sum\nolimits_{\beta\neq\alpha} \tpi_\beta \langle \tvw_\beta- \vw_\alpha,\vx_{k+1} \rangle^2 \middle\vert \textbf{C}\right]
\\&< 
    16 P(\textbf{C})\E_{\Skx}\Bigg[\sum\nolimits_{\beta\neq\alpha} r
    \underbrace{\exp\bigg(\frac{
        -\sum_{i=1}^k \|\vm_\beta -\vx_i\|^2
        +\sum_{i=1}^k \|\vm_\alpha-\vx_i\|^2
    }{
        2\nx^2(1+(k+1)\dm)
    }\bigg)}_{\text{Part $A$-$1$}}
    \\&~~~~~
    \cdot
    \underbrace{\exp\bigg(\frac{
        -\|\vw_\beta -\vwt\|^2_{\mI-(\mI+k\dw\mCw)^{-1}}
        +\|\vw_\alpha-\vwt\|^2_{\mI-(\mI+k\dw\mCw)^{-1}}
    }{
        2\nw^2
    }\bigg)}_{\text{Part $A$-$2$}}
    \\&~~~~~
    \cdot
    \underbrace{\exp\bigg(\frac{
        -\|\vm_\beta -\vx_{k+1}\|^2
        +\|\vm_\alpha-\vx_{k+1}\|^2
    }{
        2\nx^2(1+(k+1)\dm)
    }\bigg)\|\vx_{k+1}\|^2}_{\text{Part $A$-$3$}}
    \Bigg\vert \textbf{C}\Bigg]
\\&<
    16 r (M-1) C_{k=0} 
    \exp\left(-\frac{
        d_\vm^2 k
    }{
        8\nx^2
    }\right)
    \exp\left(-\frac{
        u_\vw^2\nxd^2 k
    }{
        8\ny^2
    }\right)
\\&=
    16 r (M-1) C_{k=0} 
    \exp\left(-k(\frac{
        d_\vm^2
    }{
        8\nx^2
    }
    +
    \frac{
        u_\vw^2\nxd^2
    }{
        8\ny^2
    })\right)
\end{align}
\end{proof}

\subsubsection{Bounded Efficacy - \ref{equation:part2}}
\label{proof:retrieval:part2}
\begin{proof}
We then deal with the second term $P(\textbf{C})\E_{\Skx}[\tpi_\alpha \langle \tvw_\beta- \vw_\alpha,\vx_{k+1} \rangle^2\vert\textbf{C}]$, \ref{equation:part2}:
\begin{align}
&
    P(\textbf{C})\E_{\Skx}[\tpi_\alpha \langle \tvw_\alpha- \vw_\alpha,\vx_{k+1} \rangle^2 \vert\textbf{C}]
\\&\leq
    P(\textbf{C})\E_{\Skx}[\tpi_\alpha \|\tvw_\alpha- \vw_\alpha\|^2 \|\vx_{k+1}\|^2 \vert\textbf{C}]
\\&
    \text{(See Eq.~\ref{equation:vwshift} for the derivation of } \tvw_\alpha.)
\\&=
    P(\textbf{C})\E_{\Skx}[\tpi_\alpha\|(\mI+k\dw\mCw)^{-1}(\vw_\alpha-\vwt) + \vwt - \vw_\alpha\|^2 \|\vx_{k+1}\|^2\vert\textbf{C}]
\\&=
    P(\textbf{C})\E_{\Skx}[\tpi_\alpha\|(\mI-(\mI+k\dw\mCw)^{-1})(\vwt-\vw_\alpha)\|^2 \|\vx_{k+1}\|^2\vert\textbf{C}]
\\&
    (\text{Let } \lambda_1(\mA) \text{ be the maximal eigenvalue of the matrix }\mA.)
\\&\leq
    \|\vw_\alpha-\vwt\|^2P(\textbf{C})\E_{\Skx}[\tpi_\alpha\lambda_1^2(\mI-(\mI+k\dw\mCw)^{-1})\|\vx_{k+1}\|^2\vert\textbf{C}]
\\&
    (\text{Recall that conditioned on \textbf{C} we have } \L<\lambda_d(\mCw)<\lambda_1(\mCw)<\U.)
\\&<
    \|\vw_\alpha-\vwt\|^2 P(\textbf{C})\E_{\Skx}\left[\tpi_\alpha\left(1-\frac{1}{1+k\dw \U}\right)^2\|\vx_{k+1}\|^2 \Bigg\vert\textbf{C}\right]
\\&=
    \|\vw_\alpha-\vwt\|^2 P(\textbf{C})\E_{\Skx}[\tpi_\alpha\|\vx_{k+1}\|^2\vert\textbf{C}]\left(1-\frac{1}{1+k\dw \U}\right)^2
\\&<
    \|\vw_\alpha-\vwt\|^2 \E_{\vx_{k+1}}\left[\|\vx_{k+1}\|^2\right] \left(1-\frac{1}{1+k\dw \U}\right)^2
\\&=
    \|\vw_\alpha-\vwt\|^2 (1+d\nxd^2)\left(1-\frac{1}{1+k\dw \U}\right)^2
\\&=
    \|\vw_\alpha-\vwt\|^2 (1+d\nxd^2) \left(\frac{k\dw \U}{1+k\dw \U}\right)^2.
\end{align}
Let $t=k^{-\frac{1}{4}}$. if $\dw \ll 1$ s.t. $I_\U =\{k \vert \U < 2(1+\nxd^2)\} \neq \varnothing$, then when $k\in I_\U$ we have:
\begin{align}
    \|\vw_\alpha-\vwt\|^2 (1+d\nxd^2) \left(\frac{k\dw \U}{1+k\dw \U}\right)^2
<
    \|\vw_\alpha-\vwt\|^2 (1+d\nxd^2)\min\{1,4k^2\dw^2(1+\nxd^2)^2\}.
\end{align}
\end{proof}

\subsubsection{Bounded Efficacy - \ref{equation:part3}}
\label{proof:retrieval:part3}
\begin{proof}
Finally, for the third term $P(\neg\textbf{C}) \E_{\SK}[\sum\nolimits_{\beta=1}^M \tpi_\beta \langle \tvw_\beta- \vw_\alpha,\vx_{k+1} \rangle^2 \vert \neg\textbf{C}]$, \ref{equation:part3}:
\begin{align}
&
    P(\neg\textbf{C}) \E_{\Skx}\left[\sum\nolimits_{\beta=1}^M \tpi_\beta \langle \tvw_\beta- \vw_\alpha,\vx_{k+1} \rangle^2 \middle\vert \neg\textbf{C}\right]
\\&\leq
    P(\neg\textbf{C}) \E_{\Skx}\left[\sum\nolimits_{\beta=1}^M \tpi_\beta \|\tvw_\beta- \vw_\alpha\|^2 \|\vx_{k+1}\|^2 \middle\vert \neg\textbf{C}\right]
\\&
    \text{(See Eq.~\ref{equation:vwshift} for the derivation of } \tvw_\beta.)
\\&=
    P(\neg\textbf{C}) \E_{\Skx}\left[\sum\nolimits_{\beta=1}^M \tpi_\beta \|(\mI+k\dw\mCw)^{-1}(\vw_\beta-\vwt) +\vwt - \vw_\alpha\|^2 \|\vx_{k+1}\|^2 \middle\vert \neg\textbf{C}\right]
\\&<
    P(\neg\textbf{C}) \E_{\Skx}\left[\sum\nolimits_{\beta=1}^M \tpi_\beta (2\|(\mI+k\dw\mCw)^{-1}(\vw_\beta-\vwt)\|^2 + 2\|\vwt - \vw_\alpha\|^2) \|\vx_{k+1}\|^2 \middle\vert \neg\textbf{C}\right]
\\&<
    P(\neg\textbf{C}) \E_{\Skx}\left[\sum\nolimits_{\beta=1}^M \tpi_\beta \left(2\|\vw_\beta-\vwt\|^2\lambda_1^2\left((\mI+k\dw\mCw)^{-1}\right) + 2\|\vwt - \vw_\alpha\|^2\right) \|\vx_{k+1}\|^2 \middle\vert \neg\textbf{C}\right]
\\&<
    P(\neg\textbf{C}) \E_{\Skx}\left[\sum\nolimits_{\beta=1}^M \tpi_\beta (2\cdot 4\cdot 1 + 2\cdot 4) \|\vx_{k+1}\|^2 \middle\vert \neg\textbf{C}\right]
\\&= 
    16P(\neg \textbf{C})\E_{\Skx}\left[\sum\nolimits_{\beta=1}^M \tpi_\beta \|\vx_{k+1}\|^2 \middle\vert \neg\textbf{C}\right]
\\&<
    16P(\neg \textbf{C})\E_{\vx_{k+1}}[\|\vx_{k+1}\|^2 \vert \neg\textbf{C}]
\\&
(\text{Notice } \textbf{C} \text{ is defined on } \{\vx_1,\ldots,\vx_k\})
\\&<
    16P(\neg \textbf{C})\E_{\vx_{k+1}}[\|\vx_{k+1}\|^2]
\\&<
    16(1+d\nxd^2)P(\neg \textbf{C})
\\&
    (\text{Let } t = k^{-\frac{1}{4}}.)
\\&<
    48(1+d\nxd^2)\exp\left(-\frac{k^{\frac{1}{2}}}{8}\right).
\end{align}
\end{proof}

\subsubsection{Bounded Efficacy - Summary}
\label{proof:retrieval:summary}
\begin{proof}
Summarizing \ref{equation:part1}, \ref{equation:part2}, and \ref{equation:part3}, we have:
\begin{align}
&
    \E_{\Skx}[\mathcal{L}_k^\alpha]
\\&<
    16 r (M-1) C_{k=0} 
    \exp\left(-\frac{
        d_\vm^2 k
    }{
        8\nx^2
    }\right)
    \exp\left(-\frac{
        u_\vw^2\nxd^2 k
    }{
        8\ny^2
    }\right)
\\&~~~~~
    +
    \|\vw_\alpha-\vwt\|^2 (1+d\nxd^2)\min\{1,4k^2\dw^2(1+\nxd^2)^2\}
    +
    48(1+d\nxd^2)\exp\left(-\frac{k^{\frac{1}{2}}}{8}\right)
\\&=
    C_3
    \exp\left(-k\left(\frac{
        d_\vm^2
    }{
        8\nx^2
    }+\frac{
        u_\vw^2\nxd^2
    }{
        8\ny^2
    }\right)\right)
    +
    48(1+d\nxd^2)\exp\left(-\frac{k^{\frac{1}{2}}}{8}\right)
\\&~~~~~
    +
    \|\vw_\alpha-\vwt\|^2 (1+d\nxd^2)\min\{1,4k^2\dw^2(1+\nxd^2)^2\}.
\end{align}
\end{proof}

\subsubsection{The Particular Interval}
\label{app:proof:interval}
The particular interval for the non-asymptotic bound is the union of $I_\vm$, $I_\vw$, and $I_\U$:
\begin{align}
k&\leq\min\{\frac{1}{\dm}-1,\frac{1}{\dw \nxd^2}\}\\
4\nxd\gamma\sqrt{1+k^{-\frac{1}{4}}}) &< \frac{d_\vm^2}{2}\\
\L\|\vw_\beta-\vwt\|^2-\U\|\vw_\alpha-\vwt\|^2 &> \nxd^2 u_\vw^2/2\\
\U &< 2(1+\nxd^2).
\end{align}

\subsection{Asymptotic Bound}
\label{app:proof:asymptotic}
This section proves the non-asymptotic bound in Theorem~\ref{the:retrieval}: Consider a \sp attaining the optimal pretraining risk.
As $k\rightarrow\infty$, ICL risk with biased labels is upper bounded by:
\begin{align}
    \E_{\Sk}[\mathcal{L}_k^\alpha]
&<
    \|\vw_\alpha-\vwt\|^2 (1+d\nxd^2) +
    \frac{C_1}{\redb{k}} 
    \exp\left(C_2 \redb{k}^{-\frac{1}{2}}
    \right)
    +O(\redb{k}^{-2}),
\end{align}
where $\mathcal{L}_k^\alpha=(\F(\Skx)-y_{k+1}^\alpha)^2=(\F(\Skx)-\langle \vx_{k+1},\vw_\alpha \rangle)^2$, and $C_1$ and $C_2$ are constants depending on the prior setting, $\nxd$, and $(\vmt,\vwt)$.

The proof of the asymptotic bound is heavily overlapped with the proof of the non-asymptotic bound.
\textbf{We will hide the overlapped derivations with ``($\ldots$)''.}

\begin{proof}
Assuming we are using in-context examples following Assumptions~\ref{asu:source} and~\ref{asu:misaligned}, \ie, $\vx_i \sim \mathcal{N}(\vmt, \nxd^2\mI), y_i = \langle \vx_i,\vwt \rangle$, $\|\vmt\|=\|\vwt\|=1$, and we aim to retrieve the function $\vw_\alpha$ of the prior center $(\vm_\alpha,\vw_\alpha)$ which is close to the in-context task.
Let $\mathcal{L}_k^\alpha$ indicate the squared risk $(\F^*(\Skx)-\langle \vx_{k+1},\vw_\alpha \rangle)^2$, where $\F^*(\Skx)$ is the prediction of $\Skx$ by the \bosp $\F^*$.
In order to have an upper bound on the risk, we consider $\vx_i\sim \mathcal{N}(\vmt,\nxd^2\mI)$ in two cases: (1) $\textbf{C}$:
$\L < \lambda_d\left(\frac{\sum_{i=1}^k \vx_i\vx_i^\top}{k}\right) \leq \lambda_1\left(\frac{\sum_{i=1}^k \vx_i\vx_i^\top}{k}\right) < \U \text{ and } \left\|\frac{\sum_{i=1}^k\vep_i}{k}\right\| < \nxd\sqrt{\gamma(1+t)}$
(see Lemma~\ref{lemma:eigenvalue} for $t$, $\gamma$, $\L$ and $\U$) and (2) $\neg\textbf{C}$: at least one of the previous inequalities does not hold.
Following Lemma~\ref{lemma:eigenvalue}, the probability of $\neg\textbf{C}$ is bounded by: $P(\neg\textbf{C}) \leq 3\exp(-\frac{k t^2}{8})$).

We start our upper bound analysis on the expected squared risk by splitting the risk into three parts:
\begin{align}
&
    \E_{\Skx}[\mathcal{L}_k^\alpha]
\\&
    (\ldots)
\\&=
    \label{equation:part1'} P(\textbf{C})\E_{\Skx}\left[\sum\nolimits_{\beta\neq\alpha} \tpi_\beta \langle \tvw_\beta- \vw_\alpha,\vx_{k+1} \rangle^2 \middle\vert \textbf{C}\right] \tag{Part $A'$}
    \\&~~~~~
    \label{equation:part2'}
    +
    P(\textbf{C})\E_{\Skx}[\tpi_\alpha \langle \tvw_\alpha- \vw_\alpha,\vx_{k+1} \rangle^2 \vert \textbf{C}] \tag{Part $B'$}
    \\&~~~~~
    \label{equation:part3'}
    +
    P(\neg\textbf{C}) \E_{\Skx}\left[\sum\nolimits_{\beta=1}^M \tpi_\beta \langle \tvw_\beta- \vw_\alpha,\vx_{k+1} \rangle^2 \middle\vert \neg\textbf{C}\right]. \tag{Part $C'$}
\end{align}
We will analyze three parts one by one in the following three sections respectively.
\end{proof}

\subsubsection{Asymptotic Bound - \ref{equation:part1'}}
\label{sec:retrieval:part1'}
\begin{proof}
We firstly analyze the term $P(\textbf{C})\E_{\Skx}[\sum_{\beta\neq\alpha} \tpi_\beta \langle \tvw_\beta- \vw_\alpha,\vx_{k+1} \rangle^2\vert\textbf{C}]$, \ref{equation:part1'}:
\begin{align}
&
    P(\textbf{C})\E_{\Skx}\left[\sum\nolimits_{\beta\neq\alpha} \tpi_\beta \langle \tvw_\beta- \vw_\alpha,\vx_{k+1} \rangle^2 \middle\vert \textbf{C}\right]
\\&
    (\ldots)
\\&<
    P(\textbf{C})\E_{\Skx}\left[\sum\nolimits_{\beta\neq\alpha} \tpi_\beta \|\vx_{k+1}\|^2 \middle\vert \textbf{C}\right] \left(\frac{2}{1+k\dw \L}+\|\vwt-\vw_\alpha\|\right)^2
\\&
    (\text{Notice } \|\vwt-\vw_\alpha\|\leq2.)
\\&\leq
    \label{equation:remain}
    P(\textbf{C})\E_{\Skx}\left[\sum\nolimits_{\beta\neq\alpha} \frac{\tpi_\beta}{\tpi_\alpha} \|\vx_{k+1}\|^2 \middle\vert \textbf{C}\right] \left(\frac{4}{(1+k\dw \L)^2}+\frac{8}{1+k\dw \L}\right)
\\&~~~~~
    \label{equation:magenta}
    +
    P(\textbf{C})\E_{\Skx}\left[\sum\nolimits_{\beta\neq\alpha} \tpi_\beta \|\vx_{k+1}\|^2 \middle\vert \textbf{C}\right] \|\vwt-\vw_\alpha\|^2.
\end{align}
Line~\ref{equation:magenta} will be merged with \ref{equation:part2'} and analyzed in Sec.~\ref{proof:retrieval:part2'}.
The current section will analyze the line~\ref{equation:remain}.
We start by analyzing the term $P(\textbf{C})\E_{\Skx}\left[\sum\nolimits_{\beta\neq\alpha} \frac{\tpi_\beta}{\tpi_\alpha} \|\vx_{k+1}\|^2 \middle\vert \textbf{C}\right]$.
By Eqs.~\ref{equation:ratio},~\ref{re-weight:vmu},~\ref{re-weight:vw}, and Assumption~\ref{asu:furtherasu3} on $\frac{\tpi_\beta}{\tpi_\alpha}$, we have:
\begin{align}
&
    P(\textbf{C})\E_{\Skx}\left[\sum\nolimits_{\beta\neq\alpha} \frac{\tpi_\beta}{\tpi_\alpha} \|\vx_{k+1}\|^2 \middle\vert \textbf{C}\right]
\\&
    (\ldots)
\\&<
    P(\textbf{C})\E_{\Skx}\Bigg[\sum\nolimits_{\beta\neq\alpha} r
    \underbrace{\exp\bigg(\frac{
        -\sum_{i=1}^k \|\vm_\beta -\vx_i\|^2
        +\sum_{i=1}^k \|\vm_\alpha-\vx_i\|^2
    }{
        2\nx^2(1+(k+1)\dm)
    }\bigg)}_{\text{Part $A'$-$1$}}
    \\&~~~~~
    \cdot
    \underbrace{\exp\bigg(\frac{
        -\|\vw_\beta -\vwt\|^2_{\mI-(\mI+k\dw\mCw)^{-1}}
        +\|\vw_\alpha-\vwt\|^2_{\mI-(\mI+k\dw\mCw)^{-1}}
    }{
        2\nw^2
    }\bigg)}_{\text{Part $A'$-$2$}}
    \\&~~~~~
    \cdot
    \underbrace{\exp\bigg(\frac{
        -\|\vm_\beta -\vx_{k+1}\|^2
        +\|\vm_\alpha-\vx_{k+1}\|^2
    }{
        2\nx^2(1+(k+1)\dm)
    }\bigg)\|\vx_{k+1}\|^2}_{\text{Part $A'$-$3$}}
    \Bigg\vert \textbf{C}\Bigg]
\\&
    (\text{Note that } x_1,\ldots,x_k \text{ are dependent on } \textbf{C} \text{ but } x_{k+1} \text{ is not. Thus, we break them for further analysis.})
\end{align}

In the following, we separately analyze the three terms, Part $A'$-$1$, Part $A'$-$2$, and Part $A'$-$3$.
The high-level idea is that, as $k$ increases, due to the concentration of Part $A'$-$1$ and Part $A'$-$2$, they can be upper bounded by a function of $k$.
Then, regarding Part $A'$-$1$ and Part $A'$-$2$ as constant values (their upper bounds), the expectation of Part $A'$-$3$ can be upper bounded.

\paragraph{Part $A'$-$1$.} We first deal with Part $A$-$1$. 
When conditioned on case $\textbf{C}$, we have:
\begin{align}
&
    \frac{
    \sum_{i=1}^k(
        -\|\vm_\beta -\vx_i\|^2 
        +\|\vm_\alpha-\vx_i\|^2
    )}{
    1+(k+1)\dm
    }
\\&
    (\ldots)
\\&<
    k\frac{
        \|\vm_\alpha-\vmt\|^2
        -\|\vm_\beta-\vmt\|^2
        +4\nxd\gamma\sqrt{1+t}
    }{
    1+(k+1)\dm
    }.
\end{align}
With Assumption~\ref{asu:misaligned}, we have $d_\vm^2 \leq \|\vm_\beta-\vmt\|^2-\|\vm_\alpha-\vmt\|^2$.
With Lemma~\ref{lemma:eigenvalue}, we have $\gamma=\sqrt{\frac{d}{k}}$. Let $t=k^{\delta-\frac{1}{2}}$ and $0<\delta<\frac{1}{2}$, we have:
\begin{align}
    k\frac{
        \|\vm_\alpha-\vmt\|^2
        -\|\vm_\beta-\vmt\|^2
        +4\nxd\gamma\sqrt{1+t}
    }{
    1+(k+1)\dm
    }
    =
    -\frac{
        d_\vm^2
    }{\dm}
    +
    \frac{4\nxd\sqrt{d}}{\dm}k^{-\frac{1}{2}}
    +
    O(
    k^{-1}
    ).
\end{align}

\paragraph{Part $A'$-$2$.} We then deal with Part $A'$-$2$.
When conditioned on case $\textbf{C}$, we have:
\begin{align}
&
    -\|\vw_\beta-\vwt\|^2_{\mI-(\mI+k\dw\mCw)^{-1}}
    +\|\vw_\alpha-\vwt\|^2_{\mI-(\mI+k\dw\mCw)^{-1}}
\\&
    (\ldots)
\\&< 
    -\|\vw_\beta-\vwt\|^2 \left(1-\frac{1}{1+k\dw\L}\right)
    +\|\vw_\alpha-\vwt\|^2 \left(1-\frac{1}{1+k\dw\U}\right)
\\&=
    -(\|\vw_\beta-\vwt\|^2 
    -\|\vw_\alpha-\vwt\|^2)
    +
    \left(\frac{\|\vw_\beta-\vwt\|^2}{1+k\dw\L}
    -\frac{\|\vw_\alpha-\vwt\|^2}{1+k\dw\U}\right).
\end{align}
With Assumption~\ref{asu:misaligned}, we have $d_\vw^2 \leq \|\vw_\beta-\vwt\|^2-\|\vw_\alpha-\vwt\|^2$.
Lemma~\ref{lemma:eigenvalue} gives the definitions of \L~and \U.
Let $t=k^{\delta-\frac{1}{2}}$ and $0<\delta<\frac{1}{2}$, we have:
\begin{align}
\\&=
    -d_\vw^2
    +\left(\frac{\|\vw_\beta-\vwt\|^2}{k\dw\nxd^2}
    -\frac{\|\vw_\alpha-\vwt\|^2}{k\dw(1+\nxd^2)}\right)
    +O(k^{-2})
\\&<
    -d_\vw^2
    +\frac{\|\vw_\beta-\vwt\|^2}{k\dw\nxd^2}
    +O(k^{-2})
\\&<
    -d_\vw^2
    +\frac{4}{\dw\nxd^2}k^{-1}
    +O(k^{-2}).
\end{align}

\paragraph{Part $A'$-$3$.} We finally deal with Part $A'$-$3$.
Part $A'$-$3$ is independent to case \textbf{C}, and we have:
\begin{align}
&
    P(\textbf{C})\E_{\Skx}\left[\exp\left(\frac{
        -\|\vm_\beta -\vx_{k+1}\|^2
        +\|\vm_\alpha-\vx_{k+1}\|^2
    }{
        2\nx^2(1+(k+1)\dm)
    }\right)\|\vx_{k+1}\|^2 \bigg\vert\textbf{C}\right]
\\&
    (\dots)
\\&=
    \label{equation:constant}
    C_{k=0}.
\end{align}

\paragraph{Summary of \ref{equation:part1'}.} 
Thus, summarizing Part $A'$-$1$, Part $A'$-$2$, and Part $A'$-$3$, we have:
\begin{align}
&
    P(\textbf{C})\E_{\Skx}\left[\sum\nolimits_{\beta\neq\alpha} \frac{\tpi_\beta}{\tpi_\alpha} \|\vx_{k+1}\|^2 \bigg\vert \textbf{C}\right] \left(\frac{4}{(1+k\dw \L)^2}+\frac{8}{1+k\dw \L}\right)
\\&<
    P(\textbf{C})\E_{\Skx}\Bigg[\sum\nolimits_{\beta\neq\alpha} r
    \underbrace{\exp\bigg(\frac{
        -\sum_{i=1}^k \|\vm_\beta -\vx_i\|^2
        +\sum_{i=1}^k \|\vm_\alpha-\vx_i\|^2
    }{
        2\nx^2(1+(k+1)\dm)
    }\bigg)}_{\text{Part $A'$-$1$}}
    \\&~~~~~
    \cdot
    \underbrace{\exp\bigg(\frac{
        -\|\vw_\beta -\vwt\|^2_{\mI-(\mI+k\dw\mCw)^{-1}}
        +\|\vw_\alpha-\vwt\|^2_{\mI-(\mI+k\dw\mCw)^{-1}}
    }{
        2\nw^2
    }\bigg)}_{\text{Part $A'$-$2$}}
    \\&~~~~~
    \cdot
    \underbrace{\exp\bigg(\frac{
        -\|\vm_\beta -\vx_{k+1}\|^2
        +\|\vm_\alpha-\vx_{k+1}\|^2
    }{
        2\nx^2(1+(k+1)\dm)
    }\bigg)\|\vx_{k+1}\|^2}_{\text{Part $A'$-$3$}}
    \Bigg\vert \textbf{C}\Bigg]
    \\&~~~~~
    \cdot
    \left(\frac{4}{(1+k\dw \L)^2}+\frac{8}{1+k\dw \L}\right)
\\&
    (\text{Notice } \lim_{k\rightarrow\infty} \L=\lim_{k\rightarrow\infty} \nxd^2\left(1-\frac{t}{2}-\dk\right)^2-2\nxd\dk\sqrt{1+t} = \nxd^2.)
\\&<
    r \sum_{\beta\neq\alpha} 
    \exp\left(
    \frac{
        - \frac{d_\vm^2}{\dm} 
        + \frac{4\nxd\sqrt{d}}{\dm}k^{-\frac{1}{2}} 
        + O(k^{-1})
    }{
        2\nx^2
    }\right)
    \exp\left(\frac{
        -d_\vw^2+\frac{4}{\dw\nxd^2}k^{-1}
        +O(k^{-2})
    }{
        2\nw^2
    }\right)
    C_{k=0} \left(\frac{8}{k\dw \nxd^2}+O(k^{-2})\right)
\\&=
    r (M-1) C_{k=0} 
    \exp\left(\frac{
        - d_\vm^2 
        + 4\nxd\sqrt{d}k^{-\frac{1}{2}} 
        + O(k^{-1})
    }{
        2\nm^2
    }\right)
    \exp\left(\frac{
        -d_\vw^2+\frac{4}{\dw\nxd^2}k^{-1}
        +O(k^{-2})
    }{
        2\nw^2
    }\right)
    \left(\frac{8}{k\dw \nxd^2}+O(k^{-2})\right)
\\&=
    \frac{8 r (M-1) C_{k=0}}{k\dw \nxd^2} 
    \exp\left(\frac{
        - d_\vm^2 
        + 4\nxd\sqrt{d}k^{-\frac{1}{2}} 
        + O(k^{-1})
    }{
        2\nm^2
    }\right)
    \exp\left(\frac{
        -d_\vw^2+\frac{4}{\dw\nxd^2}k^{-1}
        +O(k^{-2})
    }{
        2\nw^2
    }\right)
    +O(k^{-2})
\\&=
    \frac{8 r (M-1) C_{k=0}}{k\dw \nxd^2} 
    \exp\left(\frac{
        - d_\vm^2
        + 4\nxd\sqrt{d}k^{-\frac{1}{2}} 
    }{
        2\nm^2
    }\right)
    \exp\left(\frac{
        -d_\vw^2
    }{
        2\nw^2
    }\right)
    +O(k^{-2})
\end{align}
\end{proof}

\subsubsection{Asymptotic Bound - \ref{equation:part2'}}
\label{proof:retrieval:part2'}
\begin{proof}
We then deal with the second term $P(\textbf{C})\E_{\Skx}[\tpi_\alpha \langle \tvw_\beta- \vw_\alpha,\vx_{k+1} \rangle^2\vert\textbf{C}]$, \ref{equation:part2'}:
\begin{align}
&
    P(\textbf{C})\E_{\Skx}[\tpi_\alpha \langle \tvw_\alpha- \vw_\alpha,\vx_{k+1} \rangle^2 \vert\textbf{C}]
\\&
    (\ldots)
\\&<
    \|\vw_\alpha-\vwt\|^2 P(\textbf{C})\E_{\Skx}[\tpi_\alpha\|\vx_{k+1}\|^2\vert\textbf{C}]\left(1-\frac{1}{1+k\dw \U}\right)^2.
\end{align}
We add the line~\ref{equation:magenta} in Sec.~\ref{sec:retrieval:part1'} back:
\begin{align}
&
    P(\textbf{C})\E_{\Skx}[\tpi_\alpha(\langle \tvw_\alpha- \vw_\alpha,\vx_{k+1} \rangle)^2\vert\textbf{C}] + 
    \underbrace{ P(\textbf{C})\E_{\Skx}\left[\sum\nolimits_{\beta\neq\alpha} \tpi_\beta \|\vx_{k+1}\|^2 \middle\vert \textbf{C}\right] \|\vwt-\vw_\alpha\|^2 }_{\text{line~\ref{equation:magenta} in Sec.~\ref{sec:retrieval:part1'}}}
\\&<
    \|\vw_\alpha-\vwt\|^2 P(\textbf{C})\E_{\Skx}[\tpi_\alpha\|\vx_{k+1}\|^2\vert\textbf{C}]\left(1-\frac{1}{1+k\dw \U}\right)^2
    \\&~~~~~
    +
    P(\textbf{C})\E_{\Skx}\left[\sum\nolimits_{\beta\neq\alpha} \tpi_\beta \|\vx_{k+1}\|^2 \middle\vert \textbf{C}\right] \|\vwt-\vw_\alpha\|^2
\\&\leq
    \|\vw_\alpha-\vwt\|^2 P(\textbf{C})\E_{\Skx}[\tpi_\alpha\|\vx_{k+1}\|^2\vert\textbf{C}]
    +
    \|\vw_\alpha-\vwt\|^2 P(\textbf{C})\E_{\Skx}\left[\sum\nolimits_{\beta\neq\alpha} \tpi_\beta \|\vx_{k+1}\|^2 \middle\vert \textbf{C}\right]
\\&
    (\text{Notice } \sum\nolimits_{\beta=1}^M \tpi_\beta = 1)
\\&=
    \|\vw_\alpha-\vwt\|^2 P(\textbf{C})\E_{\Skx}[\|\vx_{k+1}\|^2\vert\textbf{C}]
\\&<
    \|\vw_\alpha-\vwt\|^2 \E_{\vx_{k+1}}\left[\|\vx_{k+1}\|^2\right]
\\&=
    \|\vw_\alpha-\vwt\|^2 (1+d\nxd^2)
\end{align}
\end{proof}

\subsubsection{Asymptotic Bound - \ref{equation:part3'}}
\label{proof:retrieval:part3'}
\begin{proof}
Finally for the third term $P(\neg\textbf{C}) \E_{\SK}[\sum\nolimits_{\beta=1}^M \tpi_\beta \langle \tvw_\beta- \vw_\alpha,\vx_{k+1} \rangle^2 \vert \neg\textbf{C}]$, \ref{equation:part3'}:
\begin{align}
&
    P(\neg\textbf{C}) \E_{\Skx}\left[\sum\nolimits_{\beta=1}^M \tpi_\beta \langle \tvw_\beta- \vw_\alpha,\vx_{k+1} \rangle^2 \middle\vert \neg\textbf{C}\right]
\\&
    (\ldots)
\\&<
    16(1+d\nxd^2)P(\neg \textbf{C})
\\&
    (\text{Let } t = k^{\delta-\frac{1}{2}}.)
\\&<
    48(1+d\nxd^2)\exp\left(-\frac{k^{2\delta}}{8}\right).
\end{align}
\end{proof}

\subsubsection{Asymptotic Bound - Summary}
\label{proof:retrieval:summary'}
\begin{proof}
Summarizing \ref{equation:part1'}, \ref{equation:part2'}, and \ref{equation:part3'}, we have:
\begin{align}
&
    \E_{\Skx}[\mathcal{L}_k^\alpha]
\\&<
    \frac{8 r (M-1) C_{k=0}}{k\dw \nxd^2} 
    \exp\left(\frac{
        - d_\vm^2
        + 4\nxd\sqrt{d}k^{-\frac{1}{2}} 
    }{
        2\nm^2
    }\right)
    \exp\left(\frac{
        -d_\vw^2
    }{
        2\nw^2
    }\right)
    +O(k^{-2})
\\&~~~~~
    +
    \|\vw_\alpha-\vwt\|^2 (1+d\nxd^2)
    +
    48(1+d\nxd^2)\exp\left(-\frac{k^{2\delta}}{8}\right)
\\&=
    \|\vw_\alpha-\vwt\|^2 (1+d\nxd^2)
    +
    \frac{8 r (M-1) C_{k=0}}{k\dw \nxd^2} 
    \exp\left(\frac{
        - d_\vm^2
        + 4\nxd\sqrt{d}k^{-\frac{1}{2}} 
    }{
        2\nm^2
    }\right)
    \exp\left(\frac{
        -d_\vw^2
    }{
        2\nw^2
    }\right)
    +O(k^{-2})
\\&=
    \|\vw_\alpha-\vwt\|^2 (1+d\nxd^2)
    +
    \frac{C_1}{k} 
    \exp(C_2 k^{-\frac{1}{2}})
    +O(k^{-2})
\end{align}
\end{proof}
\section{Proof of Lemma~\ref{lemma:zeroicl}}
\label{app:zeroiclproof}
In this subsection, we introduce the proof of Lemma~\ref{lemma:zeroicl}. We first give the full version of the lemma:

\newtheorem*{lemmainformal}{Lemma~\ref{lemma:zeroicl}}
\begin{lemmainformal}[Upper Bound for Zero-Shot ICL]
Assume a \sp attains the optimal pretraining risk,
and Assumption~\ref{asu:assumption} has only two components $\alpha$ and $\beta$, with centers $(\vm_\alpha,\vw_\alpha)=(-\vm_\beta,-\vw_\beta)$.
When performing ICL with $\vx_i\sim\mathcal{N}(\vmt\vert\nxd^2\mI)$, assume $\|\vmt\|=1$, and $y_i=0$, \ie, $y_i$ has the same preference to prior component $\alpha$ as $\beta$.
When $\dm$ and $\dw$ are sufficiently small, there is a particular interval for $k$ that ICL risk is upper bounded by:
\begin{align}
    \E_{\Sk}[\mathcal{L}_k^\alpha]
    <
    C_4
    \exp\left(-\frac{
        d_\vm^2 \red{k}
    }{
        8\nx^2
    }\right)
    +
    12 (1+d\nxd^2)\exp\left(-\frac{\red{k}^{\frac{1}{2}}}{8}\right)
    +
    (1+d\nxd^2)\min\{1,\red{k}^2\blue{\dw}^2(1+\nxd^2)^2\},
\end{align}
where $\mathcal{L}_k^\alpha=(\F(\Skx)-y_{k+1}^\alpha)^2=(\F(\Skx)-\langle \vx_{k+1},\vw_\alpha \rangle)^2$, $C_4$ is a constant depending on the prior, $\nxd$, and $(\vmt,\vwt)$. When $k$ is small, the first and second terms dominate and exponential decay.
When $k$ is large, the third term dominates and increases.
\end{lemmainformal}

\begin{proof}
The proof techniques are very similar to the proof techniques used in Sec.~\ref{app:proof:nonasymptotic}.
Assuming we are using in-context examples following $\vx_i \sim \mathcal{N}(\vmt, \nxd^2\mI), \|\vmt\|=1, y_i = 0$, \ie, $\vwt=\vzero$, and we aim to retrieve the function $\vw_\alpha$ of the prior center $(\vm_\alpha,\vw_\alpha)$ which is close to the in-context task.
Let $\mathcal{L}_k^\alpha$ indicate the squared loss $(\F^*(\Skx)-\langle \vx_{k+1},\vw_\alpha \rangle)^2$, where $\F^*(\Skx)$ is the prediction of $\Skx$ by the \bosp $\F^*$.
In order to have an upper bound on the loss, we consider $\vx_i\sim \mathcal{N}(\vmt,\nxd^2\mI)$ in two cases: (1) $\textbf{C}$:
$\L < \lambda_d\left(\frac{\sum_{i=1}^k \vx_i\vx_i^\top}{k}\right) \leq \lambda_1\left(\frac{\sum_{i=1}^k \vx_i\vx_i^\top}{k}\right) < \U \text{ and } \left\|\frac{\sum_{i=1}^k\vep_i}{k}\right\| < \nxd\sqrt{\gamma(1+t)}$
(see Lemma~\ref{lemma:eigenvalue} for $t$, $\gamma$, $\L$ and $\U$) and (2) $\neg\textbf{C}$: at least one of the previous inequalities does not hold.
Following Lemma~\ref{lemma:eigenvalue}, the probability of $\neg\textbf{C}$ is bounded by: $P(\neg\textbf{C}) \leq 3\exp(-\frac{k t^2}{8})$).

Similar to Sec.~\ref{app:proof:nonasymptotic}, we split the expected squared loss into three parts:
\begin{align}
&
    \E_{\Skx}[\mathcal{L}_k^\alpha]
\\&<
    \label{equation:zeropart1}
    P(\textbf{C})\E_{\Skx}[\tpi_\beta \langle \tvw_\beta- \vw_\alpha,\vx_{k+1} \rangle^2\vert\textbf{C}]  \tag{Part $A''$}
    \\&~~~~~
    \label{equation:zeropart2}
    +
    P(\textbf{C})\E_{\Skx}[\tpi_\alpha \langle \tvw_\alpha- \vw_\alpha,\vx_{k+1} \rangle^2\vert\textbf{C}] \tag{Part $B''$}
    \\&~~~~~
    \label{equation:zeropart3}
    +
    P(\neg\textbf{C}) \E_{\Skx}\left[\sum\nolimits_{\kappa\in\{\alpha,\beta\}} \tpi_\kappa \langle \tvw_\kappa- \vw_\alpha,\vx_{k+1} \rangle^2 \vert \neg\textbf{C}\right]. \tag{Part $C''$}
\end{align}
\end{proof}

\subsection{Proof of Lemma~\ref{lemma:zeroicl}: \ref{equation:zeropart1}}
\label{app:zeroicl:part1}
\begin{proof}
We first analyze the term $P(\textbf{C})\E_{\Skx}[\tpi_\beta \langle \tvw_\beta- \vw_\alpha,\vx_{k+1} \rangle^2\vert\textbf{C}]$, \ref{equation:zeropart1}.
Similar to Sec.~\ref{app:proof:nonasymptotic}, we have:
\begin{align}
&
    P(\textbf{C})\E_{\Skx}[\tpi_\beta \langle \tvw_\beta- \vw_\alpha,\vx_{k+1} \rangle^2\vert\textbf{C}]
\\&<
    P(\textbf{C})\E_{\Skx}[\frac{\tpi_\beta}{\tpi_\alpha} \langle \tvw_\beta- \vw_\alpha,\vx_{k+1} \rangle^2\vert\textbf{C}]
    \cdot
    \left(\frac{2}{1+k\dw \L}+\|\vwt-\vw_\alpha\|\right)^2
\\&<
    P(\textbf{C})\E_{\Skx}\Bigg[r 
    \exp\left(\frac{
        -\sum_{i=1}^k \|\vm_\beta -\vx_i\|^2
        +\sum_{i=1}^k \|\vm_\alpha-\vx_i\|^2
    }{
        2\nx^2(1+(k+1)\dm)
    }\right)
    \\&~~~~~
    \cdot
    \exp\left(\frac{
        -\|\vw_\beta -\vwt\|^2_{\mI-(\mI+k\dw\mCw)^{-1}}
        +\|\vw_\alpha-\vwt\|^2_{\mI-(\mI+k\dw\mCw)^{-1}}
    }{
        2\nw^2
    }\right)
    \\&~~~~~
    \cdot
    \exp\left(\frac{
        -\|\vm_\beta -\vx_{k+1}\|^2
        +\|\vm_\alpha-\vx_{k+1}\|^2
    }{
        2\nx^2(1+(k+1)\dm)
    }\right)\|\vx_{k+1}\|^2 \Bigg\vert\textbf{C}\Bigg]
    \cdot
    \left(\frac{2}{1+k\dw \L}+\|\vwt-\vw_\alpha\|\right)^2
\\&
    (\text{Notice } \vwt=\vzero, \vw_\beta=-\vw_\alpha.)
\\&=
    r P(\textbf{C}) \E_{\Skx}\Bigg[
    \exp\left(\frac{
        -\sum_{i=1}^k \|\vm_\beta -\vx_i\|^2
        +\sum_{i=1}^k \|\vm_\alpha-\vx_i\|^2
    }{
        2\nx^2(1+(k+1)\dm)
    }\right)
    \\&~~~~~
    \cdot
    \exp\left(\frac{
        -\|\vm_\beta -\vx_{k+1}\|^2
        +\|\vm_\alpha-\vx_{k+1}\|^2
    }{
        2\nx^2(1+(k+1)\dm)
    }\right)\|\vx_{k+1}\|^2 \Bigg\vert\textbf{C}\Bigg]\cdot 3^2
\\&=
    9r P(\textbf{C}) \E_{\Skx}\Bigg[
    \underbrace{\exp\left(\frac{
        -\sum_{i=1}^k \|\vm_\beta -\vx_i\|^2
        +\sum_{i=1}^k \|\vm_\alpha-\vx_i\|^2
    }{
        2\nx^2(1+(k+1)\dm)
    }\right)}_{\text{$A''$-$1$}}
    \\&~~~~~
    \cdot
    \underbrace{\exp\left(\frac{
        -\|\vm_\beta -\vx_{k+1}\|^2
        +\|\vm_\alpha-\vx_{k+1}\|^2
    }{
        2\nx^2(1+(k+1)\dm)
    }\right)\|\vx_{k+1}\|^2}_{\text{$A''$-$3$}} \Bigg\vert\textbf{C}\Bigg].
\end{align}

Same to Sec.~\ref{sec:retrieval:part1}, when conditioned on case $\textbf{C}$, for Part $A''$-$1$ we have:
\begin{align}
    \frac{
    \sum_{i=1}^k(
        -\|\vm_\beta -\vx_i\|^2 
        +\|\vm_\alpha-\vx_i\|^2
    )}{
    1+(k+1)\dm
    }
<
    k\frac{
        \|\vm_\alpha-\vmt\|^2
        -\|\vm_\beta-\vmt\|^2
        +4\nxd\gamma\sqrt{1+t}
    }{
    1+(k+1)\dm
    }.
\end{align}
Let $t=k^{-\frac{1}{4}}$. Recall in Assumption~\ref{asu:misaligned}, we have $\forall \beta \neq \alpha, \|\vm_\beta-\vmt\|^2 - \|\vm_\alpha-\vmt\|^2 \geq d^2_\vm$.
If $\dm \ll 1$ s.t. $I_\vm = \{k\vert (k+1)\dm\leq 1 \text{ and } \frac{d_\vm^2}{2} > 4\nxd\gamma\sqrt{1+k^{-\frac{1}{4}}} \}\neq \varnothing$, then when $k\in I_\vm$ we have:
\begin{align}
    k\frac{
        \|\vm_\alpha-\vmt\|^2
        -\|\vm_\beta-\vmt\|^2
        +4\nxd\gamma\sqrt{1+t}
    }{
    1+(k+1)\dm
    }
<
    -\frac{
        d_\vm^2
    }{4}.
\end{align}


Same to Sec.~\ref{sec:retrieval:part1}, when conditioned on case $\textbf{C}$, for Part $A''$-$3$ we have:
\begin{align}
    P(\textbf{C}) \E_{\Skx}\left[
    \exp\left(\frac{
        -\|\vm_\beta -\vx_{k+1}\|^2
        +\|\vm_\alpha-\vx_{k+1}\|^2
    }{
        2\nx^2(1+(k+1)\dm)
    }\right)\|\vx_{k+1}\|^2 \bigg\vert\textbf{C}\right] = C_{k=0}.
\end{align}

As a summary of the above analysis, we have:
\begin{align}
     P(\textbf{C})\E_{\Skx}[\tpi_\beta \langle \tvw_\beta- \vw_\alpha,\vx_{k+1} \rangle^2\vert\textbf{C}]
<
    9 r C_{k=0} 
    \exp\left(-\frac{
        d_\vm^2 k
    }{
        8\nx^2
    }\right).
\end{align}
\end{proof}

\subsection{Proof of Lemma~\ref{lemma:zeroicl}: \ref{equation:zeropart2}}
\label{app:zeroicl:part2}
\begin{proof}
We then deal with the second term $P(\textbf{C})\E_{\Skx}[\tpi_\alpha(\langle \tvw_\alpha- \vw_\alpha,\vx_{k+1} \rangle)^2\vert\textbf{C}]$, \ref{equation:zeropart2}.
The analysis is exactly the same as Sec.~\ref{proof:retrieval:part2}, and we have:
\begin{align}
    P(\textbf{C})\E_{\Skx}[\tpi_\alpha \langle \tvw_\alpha- \vw_\alpha,\vx_{k+1} \rangle^2 \vert\textbf{C}]
<
    \|\vw_\alpha-\vwt\|^2 (1+d\nxd^2) \left(\frac{k\dw \U}{1+k\dw \U}\right)^2.
\end{align}
Let $t=k^{-\frac{1}{4}}$. if $\dw \ll 1$ s.t. $I_\U =\{k \vert \U < 2(1+\nxd^2)\} \neq \varnothing$, then when $k\in I_\U$ we have:
\begin{align}
    \|\vw_\alpha-\vwt\|^2 (1+d\nxd^2) \left(\frac{k\dw \U}{1+k\dw \U}\right)^2
<
    \|\vw_\alpha-\vwt\|^2 (1+d\nxd^2)\min\{1,4k^2\dw^2(1+\nxd^2)^2\}.
\end{align}
\end{proof}

\subsection{Proof of Lemma~\ref{lemma:zeroicl}:  \ref{equation:zeropart3}}
\label{app:zeroicl:part3}
\begin{proof}
Finally, for the third term $P(\neg\textbf{C}) \E_{\Skx}[\sum\nolimits_{\kappa\in\{\alpha,\beta\}} \tpi_\kappa \langle \tvw_\kappa- \vw_\alpha,\vx_{k+1} \rangle^2 \vert \neg\textbf{C}]$, \ref{equation:zeropart3}. Similar to Sec.~\ref{proof:retrieval:part3}, we have:
\begin{align}
&
    P(\neg\textbf{C}) \E_{\Skx}\left[\sum\nolimits_{\kappa\in\{\alpha,\beta\}} \tpi_\kappa (\langle \tvw_\kappa- \vw_\alpha,\vx_{k+1} \rangle)^2 \middle\vert \neg\textbf{C}\right]
\\&<
    P(\neg\textbf{C}) \E_{\Skx}\left[\sum\nolimits_{\kappa\in\{\alpha,\beta\}} \tpi_\kappa \left(2\|(\mI+k\dw\mCw)^{-1}(\vw_\kappa-\vwt)\|^2 + 2\|\vwt - \vw_\alpha\|^2\right) \|\vx_{k+1}\|^2 \middle\vert \neg\textbf{C}\right]
\\&
    (\text{Recall } \vwt = \vzero.)
\\&<
    P(\neg\textbf{C}) \E_{\Skx}\left[\sum\nolimits_{\kappa\in\{\alpha,\beta\}} \tpi_\kappa (2\cdot 1\cdot 1 + 2\cdot 1) \|\vx_{k+1}\|^2 \middle\vert \neg\textbf{C}\right]
\\&= 
    4 P(\neg \textbf{C})\E_{\Skx}\left[\sum\nolimits_{\kappa\in\{\alpha,\beta\}} \tpi_\kappa \|\vx_{k+1}\|^2 \middle\vert \neg\textbf{C}\right]
\\&<
    4 P(\neg \textbf{C})\E_{\vx_{k+1}}[\|\vx_{k+1}\|^2 \vert \neg\textbf{C}]
\\&
(\text{Notice } \textbf{C} \text{ is defined on } \{\vx_1,\ldots,\vx_k\}.)
\\&<
    4 P(\neg \textbf{C})\E_{\vx_{k+1}}[\|\vx_{k+1}\|^2]
\\&<
    4(1+d\nxd^2)P(\neg \textbf{C})
\\&
    (\text{Let } t = k^{-\frac{1}{4}}.)
\\&<
    12(1+d\nxd^2)\exp\left(-\frac{k^{\frac{1}{2}}}{8}\right).
\end{align}
\end{proof}

\subsection{Proof of Lemma~\ref{lemma:zeroicl}: Summary}
\begin{proof}
Summarizing \ref{equation:zeropart1}, \ref{equation:zeropart2}, and \ref{equation:zeropart3}, we have:
\begin{align}
&
    \E_{\Skx}[\mathcal{L}_k^\alpha]
\\&<
    9 r C_{k=0} 
    \exp\left(-\frac{
        d_\vm^2 k
    }{
        8\nx^2
    }\right)
    +
    \|\vw_\alpha-\vwt\|^2 (1+d\nxd^2)\min\{1,4k^2\dw^2(1+\nxd^2)^2\}
    +
    12(1+d\nxd^2)\exp\left(-\frac{k^{\frac{1}{2}}}{8}\right)
\\&=
    9 r C_{k=0} 
    \exp\left(-\frac{
        d_\vm^2 k
    }{
        8\nx^2
    }\right)
    +
    (1+d\nxd^2)\min\{1,4k^2\dw^2(1+\nxd^2)^2\}
    +
    12(1+d\nxd^2)\exp\left(-\frac{k^{\frac{1}{2}}}{8}\right)
\\&=
    C_4
    \exp\left(-\frac{
        d_\vm^2 k
    }{
        8\nx^2
    }\right)
    +
    12(1+d\nxd^2)\exp\left(-\frac{k^{\frac{1}{2}}}{8}\right)
    +
    (1+d\nxd^2)\min\{1,4k^2\dw^2(1+\nxd^2)^2\}.
\end{align}
\end{proof}

\subsection{The Particular Interval}
\label{app:proof:zeroicl:interval}
The particular interval for the risk bound revealing bounded efficacy is the union of $I_\vm$ and $I_\U$:
\begin{align}
k&\leq \frac{1}{\dm}-1\\
4\nxd\gamma\sqrt{1+k^{-\frac{1}{4}}}) &< \frac{d_\vm^2}{2}\\
\U &< 2(1+\nxd^2).
\end{align}

\section{Toy Example for Component Shifting and Component Re-weighting}
\label{sec:demo}
We study how in-context examples affect the prediction of ICL by a pretrained \bosp and how the pretraining distribution affects this phenomenon.
Assume the \sp $f$ is initially pretrained on a dataset distribution to produce the minimum risk minimizer $\sf$, and then the pretrained $\sf$ is used to predict the next token $y$ of the token $x$.
Instead of direct inference via $\sf(x)$, we consider inference with additional $k$ in-context examples $\{x_i\}_{i=1}^k$ via the format $\sf([x_1,\ldots,x_k,x])$.
We aim to theoretically examine the effect of in-context examples $\{x_i\}_{i=1}^k$ on the prediction $\sf([x_1,\ldots,x_k,x])$.
While the formal problem setting may involve verbose math, this demo section illustrates the basic phenomenon for better delivering our work.

The following demo subsections are organized as follows. 
We first introduce the problem setting in Sec.~\ref{demo:setting}. 
We then connect ICL with Bayesian inference in Sec.~\ref{demo:connection}.
Further, we introduce the assumptions for the pretraining dataset in Sec.~\ref{demo:assumptions}. 
Finally, we derive a closed-form posterior and introduce two phenomena, ``Component Shifting'' and ``Component Re-weighting'' in Sec.~\ref{demo:posterior}.

\subsection{Toy Example: Pretraing Data Generative Modela}
\label{demo:setting}
ICL involves two important components: the pretraining dataset, and the \sp supporting varied input lengths.
We assume the \sp $f:\cup_{k\in \{0,\ldots,K-1\}} \mathcal{R}^{k\times 1} \rightarrow \mathcal{R}^{1\times 1}$ can fit the pretraining distribution exactly with enough data and expressivity.
To generate a training sample, we first sample a task $\mu$ from underlying task distribution $\mathcal{D}_\mu$, and then we generate tokens of the sequence from a distribution $\mathcal{D}_x(\mu)$ based on the task $\vm$.
The sample generation process is described as follows:
\begin{asu}[Demo: Pretraining Data Generative Model]
\label{asu:demo:setting} Given a task prior distribution $\mathcal{D}_{\mu}$, and a conditioned $x$ sampler $\mathcal{D}_{x}(\mu)$ conditioned on task $\mu$, the process of generating a sequence $S_K = [x_1, x_2, \ldots, x_K]$ with length $K$ follows:\\
\subasu\label{asu:demo:setting1} Sample a task $\mu$ from the task prior: $\mu\sim\mathcal{D}_\mu$, and the probability of $\mu$ is indicated by $P(\mu)$;\\
\subasu\label{asu:demo:setting2} Sample $K$ samples, each denoted by $x_i$, from the chosen task: For $i\in \{1,2,\ldots,K\}$, $x_i\sim\mathcal{D}_{x}(\mu)$, and the probability of $x_i=x$ is indicated by $P(x|\mu)$;\\
\subasu\label{asu:demo:setting3}
Define a Sequence $S_k$: For capital $K$, $S_K=[x_1,\ldots,x_K]$; and for lowercase $k$, the sequence of the first $k$ demonstrations of $S_K$ is indicated by $S_k=[x_1,\ldots,x_k]$, \eg, $S_2=[x_1,x_2]$.
\end{asu}
The generation process is related to real-world scenarios via two points: (i) For sampling step~\ref{asu:demo:setting1}, the LM is trained on varied tasks; (ii) For sampling step~\ref{asu:demo:setting2}, when one person/agent produces texts for one task, the generated text could be noisy.
For instance, given a task such as describing a football game, one person has multiple ways to describe it.

\subsection{Toy Example: \BOSP}
\label{demo:connection}
Now we consider training $f(\cdot)$ using sample $S_K$ generated via the above generation process~\ref{asu:demo:setting}:
\begin{align}
    \label{equation:demo:pretraining}
    \mathcal{L} (f) 
    = \DE_{S_K}\left[\frac{1}{K}\sum_{k=0}^{K-1} (f(S_k)-x_{k+1})^2\right] 
    = \DE_{\mu\sim \mathcal{D}_\mu} \left[
    \DE_{
    \substack{x_i \sim \mathcal{D}(\mu),\\ i\in\{1,\ldots,K\}}} 
    \left[\frac{1}{K}\sum_{k=0}^{K-1} (f(S_k)-x_{k+1})^2\middle\vert\mu\right]\right].
\end{align}
$f$ can be viewed as $K$ separate models $f_0,\ldots,f_{K-1}$, where $f_k$ takes a sequence of $k$ tokens as input.
Therefore, when the model $f$ has enough expressivity, the optimization problem $f^*=\argminA_f \mathcal{L} (f)$
could be regarded as $K$ different optimization problems: \begin{align}
    f_k^* = \argminA_{f_k} \DE_{S_K} [(f(S_k)-x_{k+1})^2], \forall k\in\{0,\ldots,K-1\}.
\end{align}
Thus, the solution $f_k^*$ for each $k$ is a minimum mean square error (MMSE) estimator~\citep[page~63]{van2004detection}, and the prediction of $\sf(S_k)$ satisfies:
\begin{align}
    \label{equation:demo:connection}
    \sf(S_k)
    = \DE_{S_K} [x_{k+1}|S_k]
    = \DE_{\mu\sim \mathcal{D}_\mu} [
    \DE_{
    \substack{x_i \sim \mathcal{D}(\mu),\\ i\in\{1,\ldots,K\}}}  [x_{k+1}|\mu,S_k]|S_k]
    = \DE_{\mu\sim \mathcal{D}_\mu} [
    \DE_{x_{k+1} \sim \mathcal{D}(\mu)}  [x_{k+1}|\mu]|S_k].
\end{align}
The prediction $\sf(S_k)$ is the expectation of $\DE_{x_{k+1} \sim \mathcal{D}(\mu)} [x_{k+1}|\mu]$ on the task posterior observing $S_k$.

\subsection{Toy Example: Gaussian Assumptions on Pretraining Data Generative Model}
\label{demo:assumptions}
In Sec.~\ref{demo:connection}, we connect ICL with Bayesian inference, and in Eq.~\ref{equation:demo:connection}, we observe that the prediction $\sf(S_k)$ depends on the posterior.
We are interested in how the in-context examples affect the prediction and the posterior.
We make assumptions on the pretraining dataset to have a closed-form expression of the posterior facilitating further analyses:
\begin{asu}[Demo: Gaussian Assumptions for Generative Model for Pretraining Data]
\label{asu:demo:assumption}
\,~\\
\subasu \label{asu:demo:assumption1} 
Task distribution: $\mu \sim \mathcal{D}_{\mu}, P(\mu) = \sum_{m=1}^M \pi_m P(\mu \vert T_m)$, where $T_m$ is the $m^\text{th}$ mixture component of the Gaussian mixture, \ie, $P(\mu \vert T_m) = \mathcal{N}(\mu \vert \mu_m,\sigma^2)$, and $\pi_m$ is the corresponding mixture weight.
$\sum_{m=1}^M\pi_m=1$, $0<\pi_m<1$, $\mu_m$ is the center of the mixture component $T_m$,
and all components share the same covariance matrix controlled by $\sigma$;\\
\subasu \label{asu:demo:assumption2} 
Token distribution: $x\sim\mathcal{D}_x(\mu)$, $P(x\vert \mu)=\mathcal{N}(x\vert \mu_m, \tau^2)$.
\end{asu}

\begin{figure}
    \centering
    \includegraphics[width = 0.95\textwidth]{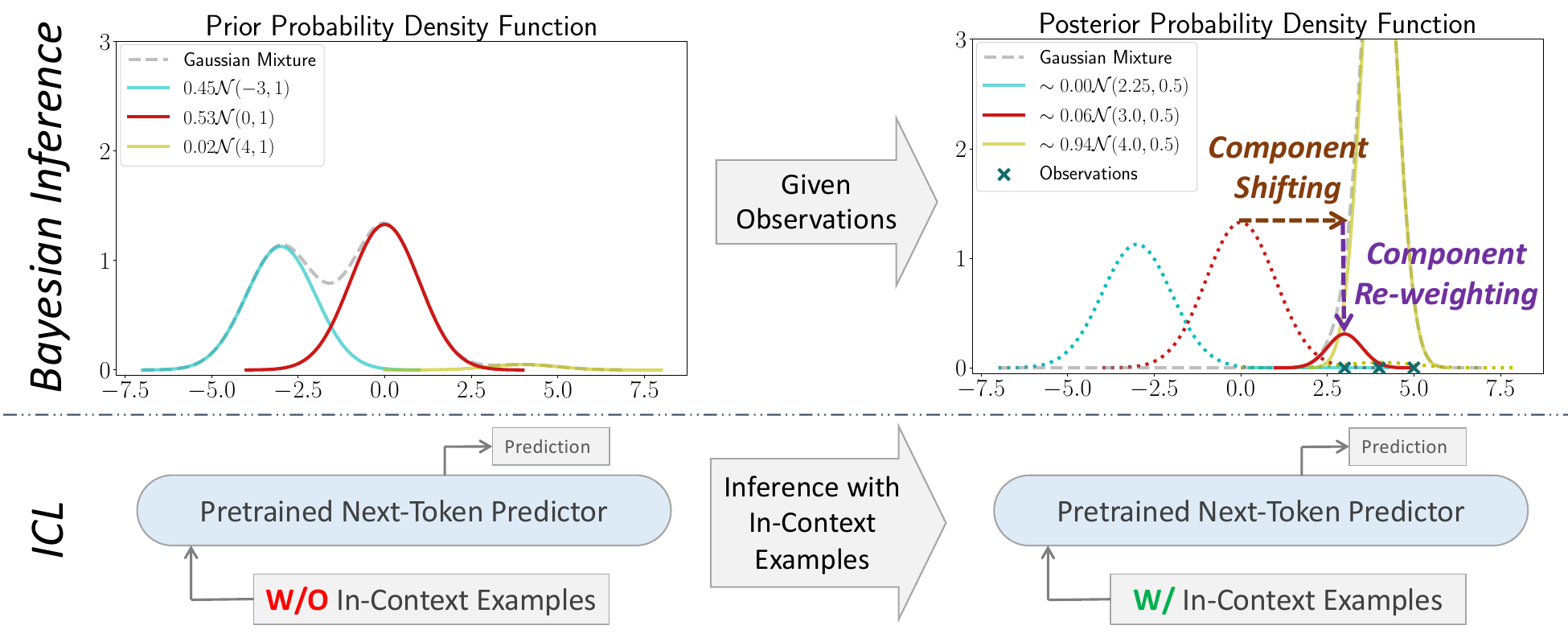}
    \caption{The left part of the figure indicates the pretrained \sp is pretrained on the task prior distribution according to Assumption~\ref{asu:demo:assumption}, and the prediction is based on the prior without in-context examples.
    The right part of the figure indicates that with in-context samples, the prediction is based on posterior, regarding the in-context examples as observed samples.}
    \label{fig: demo priorposterior}
\end{figure}
\subsection{Toy Example: Posterior Analysis}
\label{demo:posterior}
With Assumption~\ref{asu:demo:assumption}, we derive the closed-form expression of the posterior as follows:
\begin{align}
    \label{equation:demo:posterior}
    P(\mu|S_k)
    &\propto \sum_{m=1}^M \tpi_m \mathcal{N}(\mu|\tmu_m,\tsigma^2).
    \\&
    (\tpi_m = \pi_m \exp\left(\frac{k\left(\mu_m-\frac{\sum_{i=1}^k x_i}{k}\right)^2}{2(\tau^2+k\sigma^2)}\right), \tmu_m = \frac{\tau^2\mu_m+\sigma^2\sum_{i=1}^k x_i}{\tau^2+k\sigma^2}, \tsigma^2 = \frac{\tau^2\sigma^2}{\tau^2+k\sigma^2})
\end{align}
See Sec.~\ref{app:toy:proof} for proof details. From Eq.~\ref{equation:demo:posterior}, we observe two factors when comparing the posterior with the prior in Assumption~\ref{asu:demo:assumption}: (i) Component Shifting: after observing $S_k=[x_1,x_2,\ldots,x_k]$, the center of each mixture component is shifted to $\frac{\tau^2\mu_m+\sigma^2\sum_{i=1}^k x_i}{\tau^2+k\sigma^2}$; (ii) Component Re-weighting: the mixture weight $\pi_m$ of each mixture component is re-weighted by multiplying $\exp\left(\frac{k\left(\mu_m-\frac{\sum_{i=1}^k x_i}{k}\right)^2}{2(\tau^2+k\sigma^2)}\right)$ (which needs to be further normalized so that re-weighted mixture weights sum to $1$).
Fig.~\ref{fig: demo priorposterior} illustrates the phenomena of Component Shifting and Component Re-weighting by observing in-context examples.

\subsection{Proof of Posterior Derivation in Toy Example}
\label{app:toy:proof}
In this section, we give a detailed derivation of the posterior in Eq.~\ref{equation:demo:posterior} of Sec.~\ref{demo:posterior}:
\begin{align}
    P(\mu\vert S_k)
    &\propto P(\mu,S_k)\\
    &= P(S_k\vert \mu) P(\mu)\\
    &= (\Pi_{i=1}^{k} P(x_i \vert \mu))P(\mu)\\
    &= \sum_{m=1}^M \pi_m \mathcal{N}(\mu \vert \mu_m,\sigma^2) (\Pi_{i=1}^{k} \mathcal{N}(x_i \vert \mu,\tau^2)).
\end{align}
We then show $\mathcal{N}(\mu \vert \mu_m,\sigma^2) (\Pi_{i=1}^{k} \mathcal{N}(x_i \vert \mu,\tau^2))$ is proportional to a Gaussian distribution:
\begin{align}
&
    \log\left(\mathcal{N}(\mu \vert \mu_m,\sigma^2) \cdot \Pi_{i=1}^{k} \mathcal{N}(x_i \vert \mu,\tau^2)\right)
\\&=
    \left(\log\left(\frac{1}{\sqrt{2\pi}\sigma}\right)-\frac{(\mu-\mu_m)^2}{2\sigma^2}\right) + 
    \sum_{i=1}^k \left(\log\left(\frac{1}{\sqrt{2\pi}\tau}\right)-\frac{(x_i-\mu)^2}{2\tau^2}\right)
\\&
    (\text{Let } C_{10} = \log\left(\frac{1}{\sqrt{2\pi}\sigma}\right)+k\log\left(\frac{1}{\sqrt{2\pi}\tau}\right))
\\&=
    C_{10}-\frac{(\mu-\mu_m)^2}{2\sigma^2} -\sum_{i=1}^k \frac{(x_i-\mu)^2}{2\tau^2}
\\&= 
    C_{10}-\frac{1}{2\tau^2\sigma^2}\left(\tau^2(\mu-\mu_m)^2+\sigma^2\sum_{i=1}^k (x_i-\mu)^2\right)
\\&
    (\text{Abbreviate } \sum_{i=1}^k \text{ as } \sum \text{ for simplicity}.)
\\&=
    C_{10}-\frac{1}{2\tau^2\sigma^2}\bigg(\mu^2(\tau^2+k\sigma^2)-2\mu\left(\tau^2\mu_m+\sigma^2\sum x_i\right)+\left(\tau^2\mu_m^2+\sigma^2\sum x_i^2\right)\bigg)
\\&=
    C_{10}-\frac{\tau^2+k\sigma^2}{2\tau^2\sigma^2}\Bigg(\left(\mu-\frac{\tau^2\mu_m+\sigma^2\sum x_i}{\tau^2+k\sigma^2}\right)^2+\frac{\tau^2\mu_m^2+\sigma^2\sum x_i^2}{\tau^2+k\sigma^2}-\left(\frac{\tau^2\mu_m+\sigma^2\sum x_i}{\tau^2+k\sigma^2}\right)^2\Bigg)
\\&=
    C_{10}-\frac{\tau^2+k\sigma^2}{2\tau^2\sigma^2}\Bigg(\left(\mu-\frac{\tau^2\mu_m+\sigma^2\sum x_i}{\tau^2+k\sigma^2}\right)^2
    +\frac{(\tau^2\mu_m^2+\sigma^2\sum x_i^2)(\tau^2+k\sigma^2)-(\tau^2\mu_m+\sigma^2\sum x_i)^2}{(\tau^2+k\sigma^2)^2}
    \Bigg)
\\&=
    C_{10}-\frac{\tau^2+k\sigma^2}{2\tau^2\sigma^2}\Bigg(\left(\mu-\frac{\tau^2\mu_m+\sigma^2\sum x_i}{\tau^2+k\sigma^2}\right)^2
    +\frac{
        k\sigma^2\tau^2\mu_m^2+\sigma^2\sum x_i^2(\tau^2+k\sigma^2)
        -2\mu_m \tau^2\sigma^2\sum x_i - (\sigma^2\sum x_i)^2
    }{(\tau^2+k\sigma^2)^2}
    \Bigg)
\\&
    (\text{Let } C_{11} = C_{10}-\frac{\tau^2+k\sigma^2}{2\tau^2\sigma^2} \cdot \frac{
        \sigma^2\sum x_i^2(\tau^2+k\sigma^2)
        - (\sigma^2\sum x_i)^2
        - \tau^2\sigma^2(\sum x_i)^2/k
    }{(\tau^2+k\sigma^2)^2}.)
\\&=
    C_{11}-\frac{\tau^2+k\sigma^2}{2\tau^2\sigma^2}\Bigg(\left(\mu-\frac{\tau^2\mu_m+\sigma^2\sum x_i}{\tau^2+k\sigma^2}\right)^2
    +\frac{
        k\sigma^2\tau^2\mu_m^2
        -2\mu_m \tau^2\sigma^2\sum x_i + \tau^2\sigma^2(\sum x_i)^2/k
    }{(\tau^2+k\sigma^2)^2}
    \Bigg)
\\&=
    C_{11}-\frac{\tau^2+k\sigma^2}{2\tau^2\sigma^2}\Bigg(\left(\mu-\frac{\tau^2\mu_m+\sigma^2\sum x_i}{\tau^2+k\sigma^2}\right)^2
    +\frac{
        k\tau^2\sigma^2
    }{(\tau^2+k\sigma^2)^2}\cdot
    \left(\mu_m-\frac{\sum x_i}{k}\right)^2
    \Bigg)
\\&=
    C_{11} - \frac{k\left(\mu_m-\frac{\sum_{i=1}^k x_i}{k}\right)^2}{2(\tau^2+k\sigma^2)} -\frac{\left(\mu-\frac{\tau^2\mu_m+\sigma^2\sum_{i=1}^k x_i}{\tau^2+k\sigma^2}\right)^2}{2\cdot \frac{\tau^2\sigma^2}{\tau^2+k\sigma^2}}.
\end{align}
Notice $C_{11}$ is independent to $m, \forall m\in[M]$ and $\mu$. Therefore, we have:
\begin{align}
&
    \pi_m\cdot \mathcal{N}(\mu \vert \mu_m,\sigma^2) \cdot \Pi_{i=1}^{k} \mathcal{N}(x_i \vert \mu,\tau^2)
\propto
    \tpi_m\cdot \mathcal{N}(\mu\vert \tmu_m,\tsigma^2),
\end{align}
where $\tpi_m = \pi_m \exp\left(-\frac{k\left(\mu_m-\frac{\sum_{i=1}^k x_i}{k}\right)^2}{2(\tau^2+k\sigma^2)}\right), \tmu_m = \frac{\tau^2\mu_m+\sigma^2\sum_{i=1}^k x_i}{\tau^2+k\sigma^2},$ and $\tsigma^2 = \frac{\tau^2\sigma^2}{\tau^2+k\sigma^2}$. Thus:
\begin{align}
    P(\mu\vert S_k)
    &\propto 
    \sum_{m=1}^M \pi_m \mathcal{N}(\mu \vert \mu_m,\sigma^2) (\Pi_{i=1}^{k} \mathcal{N}(x_i \vert \mu,\tau^2))
    \\&\propto
    \tpi_m \mathcal{N}(\mu \vert \tmu_m,\tsigma^2).
\end{align}

\end{document}